\documentclass{article}
\PassOptionsToPackage{numbers, compress}{natbib}

 \usepackage[main, final]{neurips_2026}

\usepackage[utf8]{inputenc} 
\usepackage[T1]{fontenc}    
\usepackage{hyperref}       
\usepackage{url}            
\usepackage{booktabs}       
\usepackage{array}          
\usepackage{amsmath}        
\usepackage{amsfonts}       
\usepackage{amsthm}         
\usepackage{nicefrac}       
\usepackage{microtype}      
\usepackage[table]{xcolor}  
\usepackage{graphicx}       
\usepackage{wrapfig}        

\usepackage{tabularx}
\usepackage{array}
\usepackage{caption}
\usepackage{pifont}
\usepackage{makecell}
\usepackage{multirow}
\usepackage{placeins}        
\usepackage{float}           

\newtheorem{lemma}{Lemma}
\newtheorem{proposition}{Proposition}

\graphicspath{{images/}{figures/}}    

\definecolor{tabfirst}{rgb}{1, 0.7, 0.7}
\definecolor{tabsecond}{rgb}{1, 0.85, 0.7}
\definecolor{tabthird}{rgb}{1, 1, 0.7}

\definecolor{mygray}{gray}{0.95}

\title{Spectral Progressive Diffusion \\ for Efficient Image and Video Generation}

\newcommand{\authoremail}[1]{\makebox[0.4\textwidth]{\texttt{#1}}}

\begin{document}
\author{%
  Howard Xiao \\
  Stanford University \\
  \authoremail{howardx@stanford.edu} \\
  \And
  Brian Chao \\
  Stanford University \\
  \authoremail{brianchc@stanford.edu} \\
  \AND
  Lior Yariv \\
  Stanford University \\
  \authoremail{liory@stanford.edu} \\
  \And
  Gordon Wetzstein \\
  Stanford University \\
  \authoremail{gordon.wetzstein@stanford.edu} \\
}

\maketitle

\makeatletter
\let\arxiv@savedaddcontentsline\addcontentsline
\def\arxiv@tocfile{toc}
\renewcommand{\addcontentsline}[3]{%
  \def\arxiv@targetfile{#1}%
  \ifx\arxiv@targetfile\arxiv@tocfile
  \else
    \arxiv@savedaddcontentsline{#1}{#2}{#3}%
  \fi
}
\makeatother

\definecolor{mg}{rgb}{0.639,0.984,0.722}
\definecolor{my}{rgb}{0.996,0.875,0.643}
\definecolor{mr}{rgb}{0.941,0.561,0.620}
\newcommand{\ccg}{\cellcolor{mg}}  
\newcommand{\ccy}{\cellcolor{my}}  
\newcommand{\ccr}{\cellcolor{mr}}  

\newcommand{\cmark}{\ding{51}}%
\newcommand{\xmark}{\ding{55}}%
\newcommand{\bluecheck}{{\color{blue}\checkmark}}
\newcommand{\blackcheck}{{\color{black}\cmark}}
\newcommand{\redx}{{\Large\color{\red}\xmark} }

\newcommand{\headfontsize}{\scriptsize}
\newcommand{\tablerefsize}{\tiny}
\newcommand{\vertice}{\mathbf{v}}
\newcommand{\normal}{\mathbf{n}}

\newcommand*\colourcheck[1]{%
  \expandafter\newcommand\csname #1check\endcsname{\textcolor{#1}{\ding{51}}}%
}
\colourcheck{red}
\colourcheck{green}

\newcommand*\colourxx[1]{%
  \expandafter\newcommand\csname #1xx\endcsname{\textcolor{#1}{\ding{55}}}%
}
\colourxx{red}
\colourxx{green}

\newcommand*\colourtriangle[1]{%
  \expandafter\newcommand\csname #1triangle\endcsname{\textcolor{#1}{$\varDelta$}}%
}
\colourtriangle{yellow}

\newcommand{\greenvcell}{\ccg \cmark}
\newcommand{\redxcell}{\ccr \xmark}
\newcommand{\yellowtrcell}{\ccy $\varDelta$} 
\newcommand{\greenxcell}{\ccg \xmark}
\newcommand{\redvcell}{\ccr \cmark}


\newcommand{\hlineB}[1]{\noalign{\hrule height #1\arrayrulewidth}}

\newcommand{\bc}[1]{\textcolor{blue}{bc: #1}}
\newcommand{\hx}[1]{\textcolor{purple}{hx: #1}}
\newcommand{\ly}[1]{\textcolor{teal}{ly: #1}}

\newcommand{\methodname}{\textbf{Spectral Progressive Diffusion}}
\begin{abstract}
Diffusion models have been shown to implicitly generate visual content autoregressively in the frequency domain, where low-frequency components are generated earlier in the denoising process while high-frequency details emerge only in later timesteps. 
This structure offers a natural opportunity for efficient generation, as high-resolution computation on noise-dominated frequencies is largely redundant.
We propose \emph{Spectral Progressive Diffusion}, a general framework that progressively grows resolution along the denoising trajectory of pretrained diffusion models. 
To this end, we develop a spectral noise expansion mechanism and derive an optimal resolution schedule from the model's power spectrum. Our framework supports training-free acceleration and a novel fine-tuning recipe that further improves efficiency and quality.
We demonstrate significant speedups on state-of-the-art pretrained image and video generation models while preserving visual quality.
Our project website is at \url{https://howardxiao.ca/speed/}.


\end{abstract}

\section{Introduction}
\label{sec:intro}

Computational demands for visual generative models are increasing rapidly as image resolutions and video sequence lengths continue to grow. This trend reveals a fundamental scaling crisis: while scaling model and data sizes consistently yields better generation quality, the underlying cost of self-attention in Diffusion Transformers (DiTs)~\citep{dit} scales quadratically with the number of generated tokens. This creates a growing conflict between the desire for high-fidelity, long-duration content and the compute cost of existing architectures. A shift toward token-efficient representations offers a flexible and broadly compatible path for further scaling image and video generation.

Rissanen et al.~\citep{rissanen2023generative} and Dieleman~\citep{dieleman2024spectral} observed that diffusion models implicitly learn to generate visual content autoregressively in the frequency domain, where low-frequency components are generated earlier in the denoising process while high-frequency details emerge only in later timesteps.  While this observation has motivated a variety of frequency-domain designs for diffusion models \cite{wavediff, spectral-reg, fourierspace, frecas}, they either do not support existing pretrained models or provide limited efficiency gains. A natural way to exploit spectral autoregression is through progressive resolution growth, since the frequency content representable by a signal is intrinsically tied to its spatial resolution. However, prior progressive-resolution approaches \cite{pyrflow, scalespace, ralu, lssgen} either require significant modifications to the model architecture or rely heavily on heuristics for when and how to upsample, limiting their compatibility with state-of-the-art pretrained models \cite{zimage, flux-2, wan, cogvideox}.

In this paper, we propose \emph{Spectral Progressive Diffusion}: a general framework for progressive-resolution generation, guided by the spectral structure of the denoising process. Motivated by the spectral autoregression property of diffusion models \cite{dieleman2024spectral}, we expand the resolution only at the timesteps where high-frequency content begins to emerge from noise (see Fig.~\ref{fig:method}). This keeps early denoising steps in a reduced token space and avoids redundant computation on noise-dominated frequencies. 
To achieve this, we introduce a spectral noise expansion mechanism that uses a spectral transformation to inject high-frequency noise at the correct level while preserving the partially denoised low-frequency content. We further derive optimal resolution transition times directly from the model's power spectrum, determined by a single error-tolerance hyperparameter. 

Our framework applies directly to state-of-the-art pretrained diffusion and flow-matching models without architectural modifications. We demonstrate the flexibility of our framework across three visual generation modalities: latent-space image generation~\cite{black2024flux, zimage}, pixel-space image generation~\cite{pixelgen}, and latent-space video generation~\cite{wan}. In a training-free setting, our method delivers immediate speedups on existing pretrained models. We additionally introduce a fine-tuning recipe that further improves efficiency and quality, a direction previously unexplored for progressive-resolution generation. Extensive experiments show speedups of up to $7.8\times$ on image generation and $2.5\times$ on video generation, outperforming prior spatial acceleration methods in both runtime and visual fidelity.

To summarize: \vspace{-3pt}
\begin{itemize}
\item We propose a progressive-resolution generation framework guided by the spectral autoregression of diffusion models that applies directly to pretrained models and improves efficiency.
\item We introduce a principled spectral noise expansion mechanism and derive optimal resolution transition times from the model's power spectrum.
\item We demonstrate our framework on both image and video generation models, supporting both training-free inference and fine-tuning while maintaining high generation quality.
\end{itemize}

\begin{figure}[t!]
  \centering
  \includegraphics[width=\textwidth]{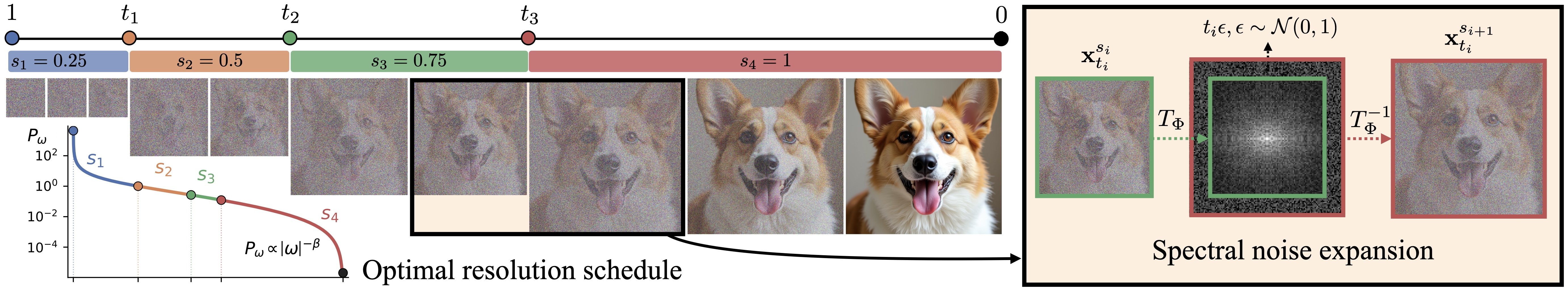}\vspace{-2pt}
  \caption{\textbf{Spectral Progressive Diffusion.} We progressively grow the resolution along the denoising trajectory using an \textit{optimal resolution schedule} derived from the spectral power of pretrained models \textbf{(left)}. At each scheduled transition, our \textit{spectral noise expansion} mechanism \textbf{(right)} injects high-frequency noise at the correct level while preserving the partially denoised low-frequency content.}
  \label{fig:method}
\end{figure}

\section{Related work}

\label{sec:related}



\vspace{-4pt}
\paragraph{Spectral-domain designs in diffusion models.} The seminal discussion by \citet{dieleman2024spectral} on the spectral autoregression property of diffusion has inspired many spectral-domain designs in diffusion models. Most of these models \cite{wavediff, wavescore, dimsum, wfvae, wavedm, du2025flowk, dctdiff} require training from scratch, making them incompatible with existing pretrained image and video generation models \cite{flux-2,zimage,wan,cogvideox}. Other methods use frequency-aware losses~\cite{lwd, spectral-reg}, autoencoder regularization~\cite{diffusability}, or noise schedules~\cite{spectralschedule, spectrallyguided, fourierspace} to improve the generation quality of existing models but do not improve efficiency. Recent training-free methods target high-resolution generation beyond a model's native resolution \cite{frecas, diffusehigh, hiwave, demofusion, freescale}. However, most of these methods first generate a fully denoised lower-resolution image, limiting their overall efficiency gains. Our method leverages similar insights of spectral autoregression but progressively grows resolution within a single denoising trajectory, achieving greater efficiency gains while maintaining compatibility with pretrained models. 




\vspace{-3pt}
\paragraph{Progressive resolution for efficient generation.} 
Progressive-resolution approaches have been proposed to improve generation efficiency by reducing token counts early in the denoising process.
However, most existing methods~\cite{cdm, jing2022subspace, edifyimage, relaydiff, pyrflow, matryoshka, pyrdiff, cogview3, scalespace} require training from scratch due to model-specific architectures, limiting their compatibility with pretrained models. Current training-free~\cite{bottleneck, ralu, fresco} and fine-tuning progressive-resolution approaches~\cite{selfcascade, lssgen, starodubcev2026scalewise} either require significant heuristics-based hyperparameter tuning or specialized modules that limit their effectiveness and their compatibility with adaptation methods such as LoRA~\cite{hu2021lora}. Progressive-resolution generation has also been explored in visual autoregressive models \cite{var, infinity, flowar, huang2025spectralar}, where next-token prediction is reformulated as next-scale prediction. Our method uniquely leverages the spectral autoregression property of diffusion models to design a principled progressive resolution growing mechanism for efficient image and video generation using minimal hyperparameters. Furthermore, our method is fully compatible with current pretrained models for both training-free inference and fine-tuning.

\paragraph{Other methods for efficient generation.} Aside from progressive-resolution pipelines, a wide variety of acceleration methods have been proposed for efficient generation. Model distillation approaches reduce the number of denoising steps by training a student model to approximate the outputs of a pretrained teacher \cite{dmd, dmd2, causvid, lcm, pcm, prog-dist}. Feature caching approaches reuse features across denoising timesteps to reduce redundant computation \cite{teacache, ma2025magcache, toca, freqca, ma2023deepcache}. Token merging techniques identify redundant tokens and merge similar tokens based on predefined heuristics \cite{lu2025toma, haurum2024agglomerative, bolya2023token, kim2024tokenfusion, wu2025importance, lee2025local, fang2025attend}. Sparse attention mechanisms identify sparse patterns in the attention maps and use block-wise calculations to improve efficiency \cite{li2025radial, xia2025trainingfree, zhan2025bidirectional, katharopoulos202linear, wang2020linformer, choromanski2021performer, beltagy2020longformer}. In this paper, we present an orthogonal and complementary progressive-resolution generation approach that accelerates both image and video generation through spectral-domain transformations.

\section{Preliminaries}
\label{sec:method:prelim}

\subsection{Diffusion models and flow matching}
\label{sec:method:prelim:basics}

Diffusion models~\citep{ho2020denoising,song2021score} define a generative process that gradually transforms samples from a simple prior (e.g., Gaussian) to data samples via a learned reverse-time process. Flow Matching \cite{liu2022rectified,lipman2023flow} reformulates the diffusion process as a continuous-time optimal transport problem and learns a velocity field that deterministically transforms noisy samples to data samples along straight paths.

Specifically, given a clean data point $\mathbf{x}_0$ sampled from a data distribution $p_{\text{data}}$, and a noise sample $\mathbf{x}_1 = \boldsymbol{\epsilon} \sim \mathcal{N}(0, I)$, the noise-to-data path is a straight line defined by
\begin{equation}
  \mathbf{x}_t \;=\; (1-t)\, \mathbf{x}_0 \,+\, t\,\boldsymbol{\epsilon},
  \qquad t \in [0, 1].
  \label{eq:flow-forward}
\end{equation}
A neural network
$\mathbf{v}_\theta(\mathbf{x}_t, t)$ is trained to predict the target velocity $\dot{\mathbf{x}}_t = \boldsymbol{\epsilon} - \mathbf{x}_0$ with the following objective:
\begin{equation}
  \mathcal{L}_\text{flow}(\theta) \;=\; \mathbb{E}_{t, \mathbf{x}_0, \boldsymbol{\epsilon}}\!\left[\,\|\mathbf{v}_\theta(\mathbf{x}_t, t) - \dot{\mathbf{x}}_t \|^2\,\right].
  \label{eq:flow-loss}
\end{equation}

A network with sufficient capacity that is trained on enough data converges to the \textit{Bayes-optimal velocity predictor} $\mathbf{v}_\theta(\mathbf{x}_t, t) = \mathbf{v}^*(\mathbf{x}_t, t)$~\citep{liu2022rectified, lipman2023flow}.

During inference, data samples are generated through sampling from the prior $\mathbf{x}_1 \sim \mathcal{N}(0, I)$ and solving the probability-flow Ordinary Differential Equation (ODE) $\dot{\mathbf{x}}_t = \mathbf{v}_\theta(\mathbf{x}_t, t)$. Due to its faster training and stable convergence, recent image and video generative models \cite{flux-2,zimage,sd3,wan} have increasingly adopted the flow-matching paradigm.

Most state-of-the-art diffusion and flow-matching models adopt the Diffusion Transformer (DiT) architecture \citep{dit}, which processes tokens through a sequence of self-attention and Multilayer Perceptron (MLP)-based operations. This motivates our approach to reduce the number of tokens processed along the denoising trajectory using the spectral autoregression property of diffusion models \citep{dieleman2024spectral}.

\begin{figure}[t]
  \centering
  \includegraphics[width=\linewidth]{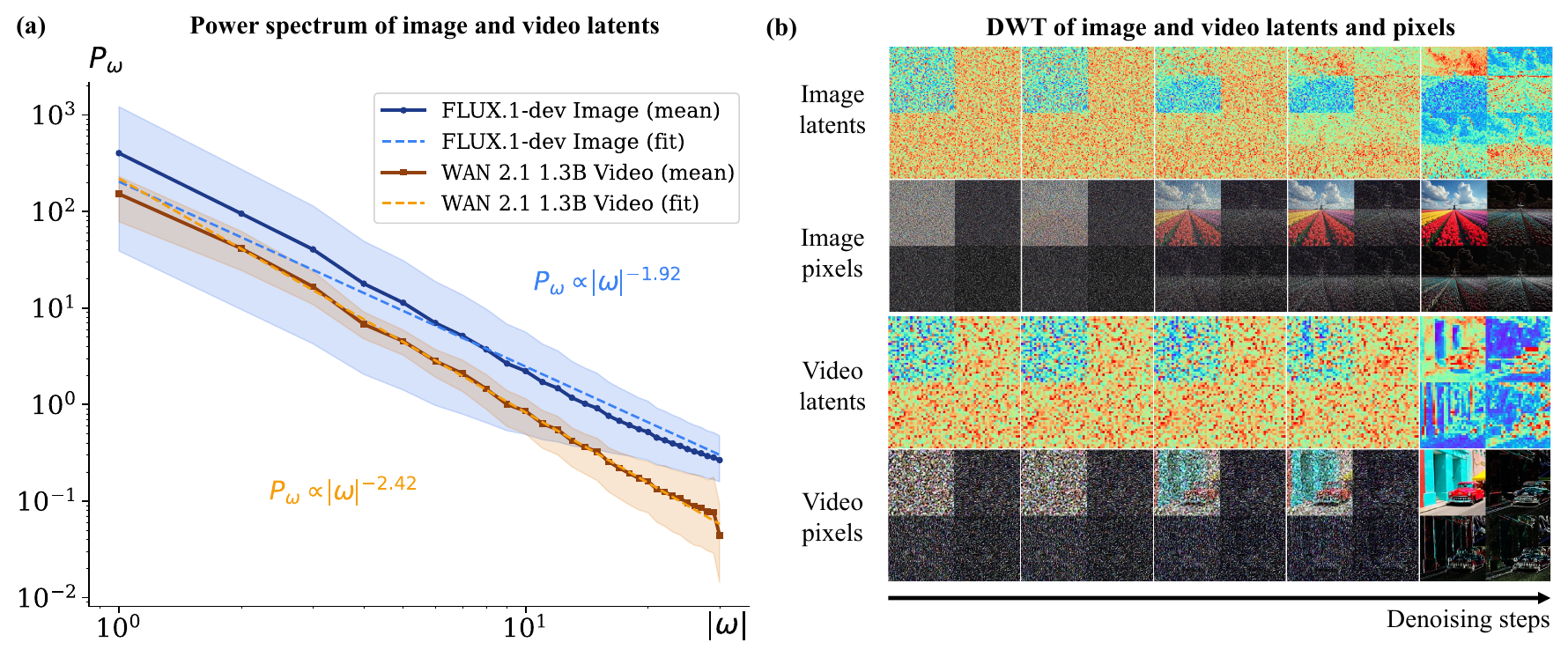}
  \caption{\textbf{Diffusion process in the spectral domain.} Latent power spectra in both image and video models decay rapidly with frequency (Fig.~\textbf{(a)}), consistent with natural images. Diffusion exhibits a frequency-domain autoregressive structure (Fig.~\textbf{(b)}) due to the aforementioned property: low frequencies emerge early in the denoising process, while high frequencies remain noise-dominated.}
  \label{fig:wavelet-spectral}
\end{figure}

\subsection{Spectral autoregression in diffusion models}
\label{sec:method:prelim:spectral}

Natural images exhibit a characteristic power-law decay \citep{ruderman1997origins}, which is the foundation of the spectral autoregression property of diffusion models \citep{dieleman2024spectral}. Specifically, let $\Phi = \{\phi_\omega\}_{\omega \in \Omega}$ be a spectral basis indexed by frequency $\omega \in \Omega$ (e.g., Fourier~\citep{fourier1822theorie}, discrete cosine~\citep{ahmed1974dct}, or Haar wavelet~\citep{haar1909theorie}), and denote the associated spectral transformation by $T_\Phi$. We define the per-frequency signal power, or \emph{power spectrum}, of a data distribution $p_{\text{data}}$ as
\begin{equation}
  P_\omega \;:=\; \mathbb{E}_{\mathbf{x}_0}\!\left[\,\bigl|x_0^{(\omega)}\bigr|^2\,\right],
  \quad \text{where  } x_0^{(\omega)} := \langle \mathbf{x}_0, \phi_\omega \rangle \text{ under $T_\Phi$}.
  \label{eq:Pb-def}
\end{equation}

When $p_\text{data}$ represents natural images, $P_\omega$ is known to exhibit a power-law decay, that is
\begin{equation}
  P_\omega \;\propto\; |\omega|^{-\beta},
  \label{eq:power-law}
\end{equation}
with the exponent term typically in the range $\beta \in [2, 3]$~\citep{dieleman2024spectral,fourierspace}. We validated similar trends in both image and video latent representations of modern latent-space diffusion models (see Fig.~\ref{fig:wavelet-spectral}).



The power-law decay of $P_\omega$ in Eq.~\eqref{eq:power-law} implies that high frequencies carry substantially weaker signal than low frequencies. Since the Gaussian noise $\boldsymbol{\epsilon}$ added during the forward process has a flat, i.e., frequency-independent, power spectrum, higher frequencies are dominated by noise early in the denoising trajectory. In this work, we leverage this spectral autoregression property to progressively grow resolution within a single denoising trajectory via spectral-domain transformations for efficient image and video generation.


\section{Spectral Progressive Diffusion}
In this section, we present a spectral-transformation-based progressive generation framework that enables arbitrary resolution transitions within a single denoising trajectory for pretrained image and video generation models (Sec.~\ref{sec:method:noise-injection}). We then derive an optimal resolution transition schedule based on power spectrum analysis and resolution-dependent frequency limits (Sec.~\ref{sec:method:schedule:signal}). Finally, we introduce a principled fine-tuning strategy to further improve both generation quality and efficiency (Sec.~\ref{sec:method:lora}).

\label{sec:method:schedule}


\subsection{Training-free inference with spectral noise expansion}
\label{sec:method:noise-injection}


Motivated by the spectral autoregression property of diffusion (Sec.~\ref{sec:method:prelim:spectral}), we progressively increase image resolution by injecting higher-frequency components along the denoising trajectory using a spectral transformation. Our framework applies natively to existing pretrained image and video generation models, without additional custom modules or architectural modifications.

Given an orthonormal spectral basis $\Phi$, let $T_\Phi$ and $T_\Phi^{-1}$ denote the forward and inverse spectral transformations. We define $S$ progressive resolution scales $s_{1:S}$, where $0 < s_1 < s_2 < \cdots < s_S = 1$, and their corresponding resolution transition times $t_{1:S-1}$, with $1 > t_1 > \cdots > t_{S-1} > 0$, that are matched to the noise schedules of pretrained models. We denote the final full-resolution state as $\mathbf{x}_0^{s_S} \in \mathbb{R}^{C \times T \times H \times W}$, and assume stage $i$ runs at spatial resolution $(s_iH, s_iW)$ for $t \in (t_i,\, t_{i-1}]$. 

At each transition time $t_i$, we expand the current low-resolution state $\mathbf{x}^{s_i}_{t_i} \in \mathbb{R}^{C \times T \times (s_iH) \times (s_iW)}$ to the next resolution state $\mathbf{x}^{s_{i+1}}_{t_i} \in \mathbb{R}^{C \times T \times (s_{i+1}H) \times (s_{i+1} W)}$ using \emph{spectral noise expansion} and timestep alignment (see Fig.~\ref{fig:method}).

\paragraph{Spectral noise expansion.}
Given a low-resolution $\mathbf{x}^{s_i}_{t_i}$ at transition time $t_i$, we expand its resolution to the next scale $s_{i + 1}$ through the following steps:
\begin{enumerate}
  \item[(i)] Compute the spectrum $\boldsymbol{\xi}^{s_i}_{t_i} = T_\Phi(\mathbf{x}^{s_i}_{t_i})$ supported on the frequency set $\Omega_{s_i}$ representable at scale $s_i$.
  \item[(ii)] Embed $\boldsymbol{\xi}^{s_i}_{t_i}$ in the lower-frequency part of the spectrum for $\Omega_{s_{i + 1}}$ and fill the slots $\Omega_{s_{i+1}} \setminus \Omega_{s_i}$ with $t_i \epsilon$, $\epsilon \sim \mathcal{N}(0, 1)$ to get $\boldsymbol{\xi}^{s_{i + 1}}_{t_i}$. 
  \item[(iii)] Convert the spectrum back to the spatial domain $\mathbf{x}^{s_{i+1}}_{t_i}$ at resolution $(s_{i+1}H, s_{i+1}W)$ via $\mathbf{x}^{s_{i+1}}_{t_i} = T_\Phi^{-1}(\boldsymbol{\xi}^{s_{i + 1}}_{t_i})$.
\end{enumerate}
Steps~(i)--(iii) expand the resolution scale from $s_i$ to $s_{i + 1}$ for $\mathbf{x}_{t_i}^{s_{i + 1}}$ by injecting high frequencies at the correct noise level while preserving the partially denoised low-frequency content of $\mathbf{x}^{s_i}_{t_i}$.

\paragraph{Timestep alignment.}
As $\Phi$ is orthonormal, the high-frequency noise padding in $\Omega_{s_{i + 1}} \setminus \Omega_{s_i}$ raises the overall noise level of the enlarged $\mathbf{x}_{t_i}^{{s_{i+1}}}$, which makes it no longer correspond to the original timestep $t_i$. We therefore rescale the inverse-transformed output and re-index the timestep by:
\begin{equation}
  \tilde{\mathbf{x}}^{s_{i+1}}_{\tilde{t}_i} := \frac{s_{i+1}/s_i}{1 + ((s_{i+1}/s_i) - 1)t_i} \cdot \mathbf{x}^{s_{i+1}}_{t_i},
  \label{eq:scale-align}
\end{equation}
where
\begin{equation}
  \tilde{t}_i := \frac{(s_{i+1}/s_i)\,t_i}{1 + ((s_{i+1}/s_i) - 1)t_i}.
  \label{eq:time-align}
\end{equation}
Eq.~\eqref{eq:scale-align} and Eq.~\eqref{eq:time-align} give us the aligned state $\tilde{\mathbf{x}}^{s_{i+1}}_{\tilde{t}_i}$ at the aligned timestep $\tilde{t}_i$. The integration of the probability-flow ODE then resumes at resolution $(s_{i+1}H, s_{i+1}W)$ until the next transition time $t_{i+1}$. 
The full derivation of Eqs.~\eqref{eq:scale-align}--\eqref{eq:time-align} is provided in Appendix Sec.~\ref{app:time-align}. In practice, $S$ and $s_{1:S}$ are chosen to align with the multi-resolution training distribution of pretrained generative models. By default, we choose $T_\Phi$ as the Discrete Cosine Transform (DCT)~\citep{ahmed1974dct} with discrete cosine basis $\Phi$. Alternative spectral transforms (such as the Discrete Wavelet Transform (DWT) and the Fourier Transform (FFT)) are compared and evaluated in Sec.~\ref{sec:exp:ablations}.

\renewcommand{\arraystretch}{1.0}
\setlength{\aboverulesep}{0.2ex}
\begin{table*}[!t]
  \centering
  \caption{\textbf{Training-free quantitative comparisons on FLUX.1-dev~\citep{black2024flux}.} Image resolution is $1024^2$ and baseline rows are copied from Table 1 in RALU~\citep{ralu}. We group results based on $1\times, 5\times, 7\times$ speedup settings where progressive-resolution approaches are applied in combination with reducing denoising steps. Our method improves the speed--quality tradeoff across all speedup settings compared to baseline approaches. Further evaluation details and comparisons are shown in Appendix Sec.~\ref{subsec:image_latent}.}
  \label{tab:flux-main}
  \scriptsize
  \begin{tabular*}{\textwidth}{@{\extracolsep{\fill}} l ccccccc @{}}
    \toprule
    \multirow{2}{*}{Method} & \multirow{2}{*}{\makecell{Speedup ($\times$) $\uparrow$}} & \multirow{2}{*}{\makecell{TFLOPs $\downarrow$}} & Overall & \multicolumn{2}{c}{Image quality} & \multicolumn{2}{c}{Prompt alignment} \\
    \cmidrule(lr){4-4} \cmidrule(lr){5-6} \cmidrule(lr){7-8}
     & & & ImageReward $\uparrow$ & CLIP-IQA $\uparrow$ & NIQE $\downarrow$ & T2I-Comp. $\uparrow$ & GenEval $\uparrow$ \\
    \midrule
    FLUX (50 steps)             & 1.00$\times$              & 2991.01              & \textbf{1.095}    & 0.707             & 6.75             & \textbf{0.634}    & \textbf{0.698} \\
    RALU~\citep{ralu}     & \underline{1.58}$\times$              & \textbf{1749.94}  & 1.028             & \underline{0.712} & \textbf{6.07}    & 0.613                 & 0.648     \\
    \textbf{Ours ($S=2$)} & \textbf{1.66}$\times$  & \underline{1755.22}              & \underline{1.049} & \textbf{0.719}    & \underline{6.43}             & \underline{0.617} & \underline{0.654}     \\
    \midrule[0.3pt]
    FLUX (10 steps) & 4.84$\times$ & 610.02 & 0.981 & 0.679 & 6.93 & 0.618 & 0.647 \\
    Bottleneck~\citep{bottleneck} & 4.67$\times$              & 571.23              & 0.889          & 0.661             & 9.16          & 0.620             & \textbf{0.687}    \\
    RALU~\citep{ralu}             & 4.98$\times$              & 540.47              & 1.022          & \underline{0.700} & \textbf{6.43} & \textbf{0.626}    & 0.652             \\
    \textbf{Ours ($S=2$)}         & \underline{5.77}$\times$  & \underline{500.34}  & \textbf{1.059} & 0.696             & 6.69          & \underline{0.624} & \underline{0.655} \\
    \textbf{Ours ($S=3$)}         & \textbf{6.09}$\times$     & \textbf{469.15}     & \underline{1.042} & \textbf{0.701} & \underline{6.53} & 0.623          & 0.637             \\
    \midrule[0.3pt]
    FLUX (7 steps) & 6.62$\times$ & 431.45 & 0.920 & 0.660 & 8.25 & 0.594 & 0.583 \\
    Bottleneck~\citep{bottleneck} & 6.64$\times$              & 431.52              & 0.792          & 0.631             & 8.71          & 0.605             & \underline{0.672} \\
    RALU~\citep{ralu}             & 6.69$\times$              & \underline{426.01}  & 0.999          & 0.681             & 6.87          & \textbf{0.633}    & \textbf{0.682}    \\
    \textbf{Ours ($S=2$)}         & \underline{6.78}$\times$  & 427.03              & \textbf{1.039} & \underline{0.689} & \underline{6.78} & 0.620          & 0.667             \\
    \textbf{Ours ($S=3$)}         & \textbf{7.09}$\times$     & \textbf{406.24}     & \underline{1.015} & \textbf{0.694} & \textbf{5.99} & \underline{0.627} & 0.637             \\
    \bottomrule
  \end{tabular*}
\end{table*}

\subsection{Optimal resolution schedule}

\label{sec:method:schedule:signal}

Section~\ref{sec:method:noise-injection} described our progressive-resolution transition mechanism during the denoising process given a predetermined resolution transition schedule $t_{1:S-1}$. In practice, while increasing the number of resolution stages $S$ improves token efficiency, it introduces additional transition times as hyperparameters to tune. As a result, existing progressive-resolution approaches typically limit $S = 2$ or leave $t_{1:S-1}$ as a brittle, model-specific design choice~\citep{bottleneck, ralu}. 
In contrast, our method supports an arbitrary number of resolution stages $S$ and arbitrary resolution scales $s_{1:S}$. In the following paragraphs, we derive a $\delta$-\textit{optimal resolution schedule} for pretrained flow-matching models that determines the optimal transition times $t^*_{1:S-1}$ from a single error-tolerance parameter $\delta$ independent of $S$.


\paragraph{Spectral-domain flow matching.}
The spatial-domain flow-matching forward process defined in Eq.~\eqref{eq:flow-forward} can be written in the spectral domain under transformation $T_\Phi$ with orthonormal basis $\Phi = \{\phi_\omega\}_{\omega \in \Omega}$:
\begin{equation}
  x_t^{(\omega)} \;=\; (1-t)\,x_0^{(\omega)} \,+\, t\,\epsilon^{(\omega)},
  \qquad \omega \in \Omega,
  \label{eq:forward-mode}
\end{equation}
where $\epsilon^{(\omega)} := \langle \boldsymbol{\epsilon}, \phi_\omega \rangle$. By orthonormality of $\Phi$, $\epsilon^{(\omega)} \sim \mathcal{N}(0, 1)$ is i.i.d.\ Gaussian noise for each $\omega \in \Omega$. We denote the per-frequency component of the Bayes-optimal velocity predictor described in Sec.~\ref{sec:method:prelim:basics} by $v^{*(\omega)}(\mathbf{x}_t, t)$. We then determine when a spectral component $\omega$ is noise-dominated in Proposition~\ref{prop:activation} and link this to spatial resolution to derive the $\delta$-optimal resolution schedule in Proposition~\ref{prop:transition}. 


\newcounter{mainprop}

\paragraph{Modeling assumptions.}
Our derivations below rely on two simplifying assumptions.
(i) \emph{Independent Gaussian coefficients.} We model the clean spectral coefficients as independent across frequencies, with $x_0^{(\omega)} \sim \mathcal{N}(0, P_\omega)$ and $P_\omega > 0$. This second-moment approximation, based on the power spectrum measured on centered data, decouples the frequencies and makes the Bayes-optimal velocity $v^{*(\omega)}$ analytically tractable.
(ii) \emph{Power-law decay.} We assume $P_\omega$ follows the power-law decay in Eq.~\eqref{eq:power-law}.
We empirically validate the activation times of Proposition~\ref{prop:activation} in Appendix Sec.~\ref{sec:motivation}.

\refstepcounter{mainprop}\label{prop:activation}\textbf{Proposition~\themainprop{} (Per-frequency $\delta$-optimal activation time):} Under assumption (i), for all $t \ge t_\omega$ and $\delta\in(0, 1)$, we have:
\begin{equation}
  \mathbb{E}\!\left[\bigl|v^{*(\omega)}(\mathbf{x}_t, t) - \epsilon^{(\omega)}\bigr|^2\right] \;\le\; \delta,
  \label{eq:noise-dominated}
\end{equation}
where $t_\omega$ is the $\delta$-optimal activation time for $\omega$:
\begin{equation}
  t_\omega \;:=\; \frac{1}{1 + \sqrt{\dfrac{\delta}{P_\omega\,(1 + P_\omega - \delta)}}}.
  \label{eq:activation}
\end{equation}
Intuitively, for $t \geq t_\omega$, spectral component $\omega$ is noise-dominated as its optimal velocity predictions are approximately noise $\epsilon^{(\omega)}$. After $t_\omega$, the signal power overcomes the error tolerance $\delta$ and structural information is recovered. The proof of Proposition~\ref{prop:activation} is given in Appendix Sec.~\ref{app:proof-prop1}.

\refstepcounter{mainprop}\label{prop:transition}\textbf{Proposition~\themainprop{} (Per-resolution $\delta$-optimal transition time):} Under assumptions (i)--(ii), for any resolution scales $s_i, s_{i + 1}$, the optimal transition time from scale $s_i$ up to $s_{i+1}$ is
\begin{equation}
  t_i^* \;:=\; \min_{\omega \in \Omega_{s_i}}\, t_\omega \;=\; t_{\omega=s_i \cdot \omega_{\max}(H,W)},
  \label{eq:stage-transition}
\end{equation}
where $\Omega_{s_i}$ is the set of frequencies representable on the $(s_iH, s_iW)$ grid. The maximum representable frequency of the full-resolution spatial grid, $\omega_{\max}(H, W) = \min(H, W)/2$, is given by the Nyquist--Shannon sampling theorem~\citep{nyquist1928certain,shannon1949communication}. 

Proposition~\ref{prop:transition} connects frequency activation times with spatial resolution transitions through the Nyquist limit. The Nyquist limit determines \emph{which} frequencies a resolution can represent, while the power spectrum determines \emph{when} each frequency emerges from noise along the denoising trajectory. This implies that high-resolution computation in early denoising steps is redundant as most representable frequencies $\omega$ are still noise-dominated, motivating our approach of Spectral Progressive Diffusion. We apply spectral noise expansion at $t_i^*$ to progressively increase resolution precisely when finer details emerge from noise, thereby maximizing token efficiency. The hyperparameter $\delta$ is independent of $S$, making our method particularly robust to tuning compared to prior methods. The proof of Proposition~\ref{prop:transition} is deferred to Appendix Sec.~\ref{app:proof-prop2} and we analyze the effect of the number of resolution stages $S$ and the error tolerance $\delta$ in Sec.~\ref{sec:exp:ablations}.

\renewcommand{\arraystretch}{1.0}
\setlength{\doublerulesep}{0.5pt}
\setlength{\aboverulesep}{0.2ex}
\begin{table*}[!t]
  \centering
  \caption{\textbf{Fine-tuning quantitative comparisons on latent- (Z-Image~\citep{zimage}) and pixel-space image generation (PixelGen~\citep{pixelgen}).} Image resolution is $1024^2$ for Z-Image and $512^2$ for PixelGen. Across both latent- and pixel-space image generation, our fine-tuning method bridges the gap between model pretraining and progressive-resolution inference, further improving generation quality while preserving efficiency gains. RALU~\citep{ralu} does not support end-to-end fine-tuning and is evaluated training-free only. Bottleneck Sampling~\citep{bottleneck} is not included here or in Table~\ref{tab:wan_vbench_720p} because its inference code is unavailable. TF denotes our training-free variant. \textbf{Bold} and \underline{underline} mark the best and second-best results within each group. Further evaluation details and comparisons are shown in Appendix Sec.~\ref{subsec:image_pixel}.}
  \label{tab:zimage-main}
  \scriptsize
  \begin{tabular*}{\textwidth}{@{\extracolsep{\fill}} l ccccccc @{}}
    \toprule
    Method & \makecell{Speedup ($\times$) $\uparrow$} & \makecell{TFLOPs $\downarrow$} & ImageReward $\uparrow$ & CLIP-IQA $\uparrow$ & NIQE $\downarrow$ & T2I-Comp. $\uparrow$ & GenEval $\uparrow$ \\
    \midrule
    Z-Image (50 steps)           & 1.00$\times$             & 4941.23             & 0.965    & 0.700    & 5.41    & 0.731    & 0.745 \\
    \midrule
    RALU~\citep{ralu}          & 1.59$\times$ & \underline{3102.62} & 0.892             & 0.680             & 6.09             & 0.648 & 0.720 \\
    \textbf{Ours} (TF, $S=2$)  & \underline{1.65}$\times$ & 3132.03 & 0.904             & 0.688             & 5.87             & 0.658 & \underline{0.730} \\
    \textbf{Ours} (LoRA, $S = 2$)   & \underline{1.65}$\times$             & 3132.03             & \textbf{0.982} & \textbf{0.699} & \underline{5.72}             & \textbf{0.725}                 & \textbf{0.731} \\
    \textbf{Ours} (TF, $S=3$)  & \textbf{1.74}$\times$ & \textbf{2871.09} & 0.875             & 0.690             & \textbf{5.59}             & 0.650 & 0.682 \\
    \textbf{Ours} (LoRA, $S = 3$)   & \textbf{1.74}$\times$             & \textbf{2871.09}             & \underline{0.954} & \underline{0.697} & 5.75             & \underline{0.717}                 & 0.728 \\

    \midrule
    Z-Image (7 steps)          & 5.97$\times$ & 701.92 & 0.763 & 0.631 & \textbf{6.14} & 0.648 & 0.667 \\
    RALU~\citep{ralu}          & \underline{7.73}$\times$ & \textbf{673.99} & 0.572 & 0.595 & 6.74 & 0.636 & 0.596 \\
    \textbf{Ours} (TF, $S=2$)  & \underline{7.73}$\times$ & \underline{678.77} & \underline{0.804} & \underline{0.641} & \underline{6.46} & \underline{0.665} & \underline{0.680} \\
    \textbf{Ours} (LoRA, $S = 2$) & \textbf{7.81}$\times$ & \underline{678.77} & \textbf{0.918} & \textbf{0.658} & 6.57 & \textbf{0.683} & \textbf{0.744} \\
    \midrule
    \midrule
    PixelGen (25 steps)         & 1.00$\times$             & 65.36             & 0.921    & 0.734    & 5.95    & 0.574    & 0.794 \\
    \midrule
    RALU~\citep{ralu}         & \textbf{1.61}$\times$             & 34.46             & \underline{0.801}    & 0.713    & 6.83    & 0.551    & 0.770 \\
    \textbf{Ours} (TF, $S=2$) & \underline{1.60}$\times$             & \textbf{33.72}             & {0.799}    & \underline{0.718}    & \underline{6.10}    &  \underline{0.568}  & \textbf{0.782} \\
    \textbf{Ours} (LoRA, $S=2$) & {1.55}$\times$             & \textbf{33.72}             & \textbf{0.913}    & \textbf{0.728}    & \textbf{5.87}    & \textbf{0.580}    & \underline{0.776} \\
    \bottomrule
  \end{tabular*}
\end{table*}

\renewcommand{\arraystretch}{1.0}
\setlength{\aboverulesep}{0.2ex}
\begin{table*}[!t]
  \centering
  \caption{\textbf{Quantitative comparison on latent-space video generation on WAN~2.1~\citep{wan}.} Video resolution is maintained at 720p. Our training-free approach demonstrates more than $2\times$ speedup while maintaining high generation quality and outperforming the 25-step full-resolution baseline and TeaCache~\citep{teacache}. TeaCache uses a caching threshold of 0.07. \textbf{Bold} and \underline{underline} mark the best and second-best results among methods at comparable speedups.}
  \label{tab:wan_vbench_720p}
  \scriptsize
  \setlength{\tabcolsep}{3pt}
  \begin{tabular*}{\textwidth}{@{\extracolsep{\fill}} l cccccccc @{}}
    \toprule
    \makecell{Method} &
    \makecell{Speedup ($\times$) $\uparrow$} &
    \makecell{TFLOPs $\downarrow$} &
    \makecell{Subject\\Consistency $\uparrow$} &
    \makecell{Background\\Consistency $\uparrow$} &
    \makecell{Motion\\Smoothness $\uparrow$} &
    \makecell{Dynamic\\Degree $\uparrow$} &
    \makecell{Aesthetic\\Quality $\uparrow$} &
    \makecell{Image\\Quality $\uparrow$} \\
    \midrule
    WAN~2.1 (50 steps)     & 1.00$\times$ & 119292         & 0.9492          & 0.9621          & 0.9874          & 0.4800          & 0.5993          & 0.6133          \\
    \midrule[0.3pt]
    WAN~2.1 (25 steps)     & \underline{2.00}$\times$ & \underline{59646} & 0.9434          & \underline{0.9597} & \textbf{0.9879} & \underline{0.4250} & \underline{0.5893} & 0.5706          \\
    TeaCache (0.07)        & 1.91$\times$ & 62032 & \underline{0.9449} & 0.9592 & \underline{0.9877} & 0.4200 & 0.5886 & \underline{0.5763} \\
    \textbf{Ours ($S=2$)} & \textbf{2.03}$\times$ & \textbf{57417} & \textbf{0.9462} & \textbf{0.9598} & 0.9859          & \textbf{0.4950} & \textbf{0.5975} & \textbf{0.6114} \\
    \midrule[0.3pt]
    \textbf{Ours ($S=3$)} & 2.54$\times$ & 45953 & 0.9299          & 0.9499          & 0.9815          & 0.5700 & 0.5696          & 0.5990 \\
    \bottomrule
  \end{tabular*}
\end{table*}


\subsection{Spectral-transformation-based fine-tuning}
\label{sec:method:lora}
While Spectral Progressive Diffusion is compatible with pretrained image and video generation models in a training-free manner (Sec.~\ref{sec:method:noise-injection}), it implicitly assumes multi-resolution generation capability of pretrained models across scales $s_{1:S}$. Moreover, our optimal resolution schedule (Sec.~\ref{sec:method:schedule:signal}) may not accurately reflect pretrained models' training dynamics, potentially introducing training--inference gaps. 
In this subsection, we outline a spectral-transformation-based fine-tuning approach that directly follows our Spectral Progressive Diffusion framework and our optimal resolution schedule. We include additional implementation details of our fine-tuning framework in Sec.~\ref{sec:exp:image} and Appendix Sec.~\ref{subsec:image_pixel}.

\paragraph{Resolution-specific velocity targets.} 
We assume $S$ progressive resolution scales $s_{1:S}$ and their corresponding $\delta$-optimal resolution transition times $t_{1:S-1}$ from Sec.~\ref{sec:method:schedule:signal} (denoted here as $t_i$ instead of $t^*_i$ for simplicity). For resolution stage $i$, the model operates at scale $s_i$ for $t \in (t_i, \tilde{t}_{i-1}]$, where $\tilde{t}_{i-1}$ is the aligned timestep obtained from Eq.~\eqref{eq:time-align} after expanding from $s_{i-1}$ to $s_i$ at transition time $t_{i-1}$. The state at the beginning of stage $i$ is hence the enlarged and timestep-aligned $\tilde{\mathbf{x}}_{\tilde{t}_{i - 1}}^{s_{i}}$. The model is fine-tuned to follow a straight-line path from $\tilde{\mathbf{x}}_{\tilde{t}_{i - 1}}^{s_{i}}$ to the standard flow-matching state at the next transition time:
\begin{equation}
  \mathbf{x}_{t_i}^{s_i}
  =
  (1-t_i)\mathbf{x}_0^{s_i}
  +
  t_i\boldsymbol{\epsilon}^{s_i},
  \label{eq:lora-stagewise}
\end{equation}
where $\mathbf{x}_0^{s_i}$ denotes the clean data sample at scale $s_i$.
The corresponding stage-specific velocity target is therefore:
\begin{equation}
  \mathbf{v}^{s_i}
  =
  \frac{
  \tilde{\mathbf{x}}_{\tilde{t}_{i - 1}}^{s_{i}}
  -
  \mathbf{x}_{t_i}^{s_i}
  }{
  \tilde{t}_{i - 1} - t_i
  }.
  \label{eq:lora-stagewise-velocity}
\end{equation}

\paragraph{Training and inference.}
At each training step we sample $\mathbf{x}_0 \sim p_\text{data}$ and $t \sim \mathcal{U}(0,1)$. We assign $t$ to the resolution stage $i$ satisfying $t \in (t_i,t_{i-1}]$, and then sample $\boldsymbol{\epsilon}^{s_i} \sim \mathcal{N}(0, I)$ at scale $s_i$. We construct the stage input $\tilde{\mathbf{x}}_{\tilde{t}_{i - 1}}^{s_{i}}$ by
applying the same spectral noise expansion and timestep alignment used at
inference, where newly introduced spectral coefficients are filled using
$t_{i-1} \cdot T_\Phi(\boldsymbol{\epsilon}^{s_i})$. We use the same
$\boldsymbol{\epsilon}^{s_i}$ to construct the endpoint
$\mathbf{x}^{s_i}_{t_i}$ from Eq.~\eqref{eq:lora-stagewise}, so that $\tilde{\mathbf{x}}_{\tilde{t}_{i - 1}}^{s_{i}}$ and $\mathbf{x}^{s_i}_{t_i}$ are correlated through a shared noise realization. The training sample is the point on the straight path between them:
\begin{equation}
  \mathbf{x}_t^{s_i}
  =
  \mathbf{x}_{t_i}^{s_i}
  +
  (t-t_i)\,\mathbf{v}^{s_i},
\end{equation}
with target velocity $\mathbf{v}^{s_i}$ from
Eq.~\eqref{eq:lora-stagewise-velocity}.

We fine-tune a pretrained model $\mathbf{v}_\theta$ by minimizing
\begin{equation}
  \mathcal{L}(\theta) \;=\; \mathbb{E}_{t,\, \mathbf{x}_0,\, \boldsymbol{\epsilon}}\!\left[\,\|\,\mathbf{v}_\theta(\mathbf{x}_t^{s_i}, t) - \mathbf{v}^{s_i}\,\|^2\,\right].
  \label{eq:lora-loss}
\end{equation}
At inference, we use the same procedure as in Sec.~\ref{sec:method:noise-injection} with the fine-tuned model, using the same resolution schedule $t_{1:S-1}$ and transform $T_\Phi$.

\section{Experiments}
\label{sec:exp}

\begin{figure}[!t]
  \centering
  \includegraphics[width=\linewidth]{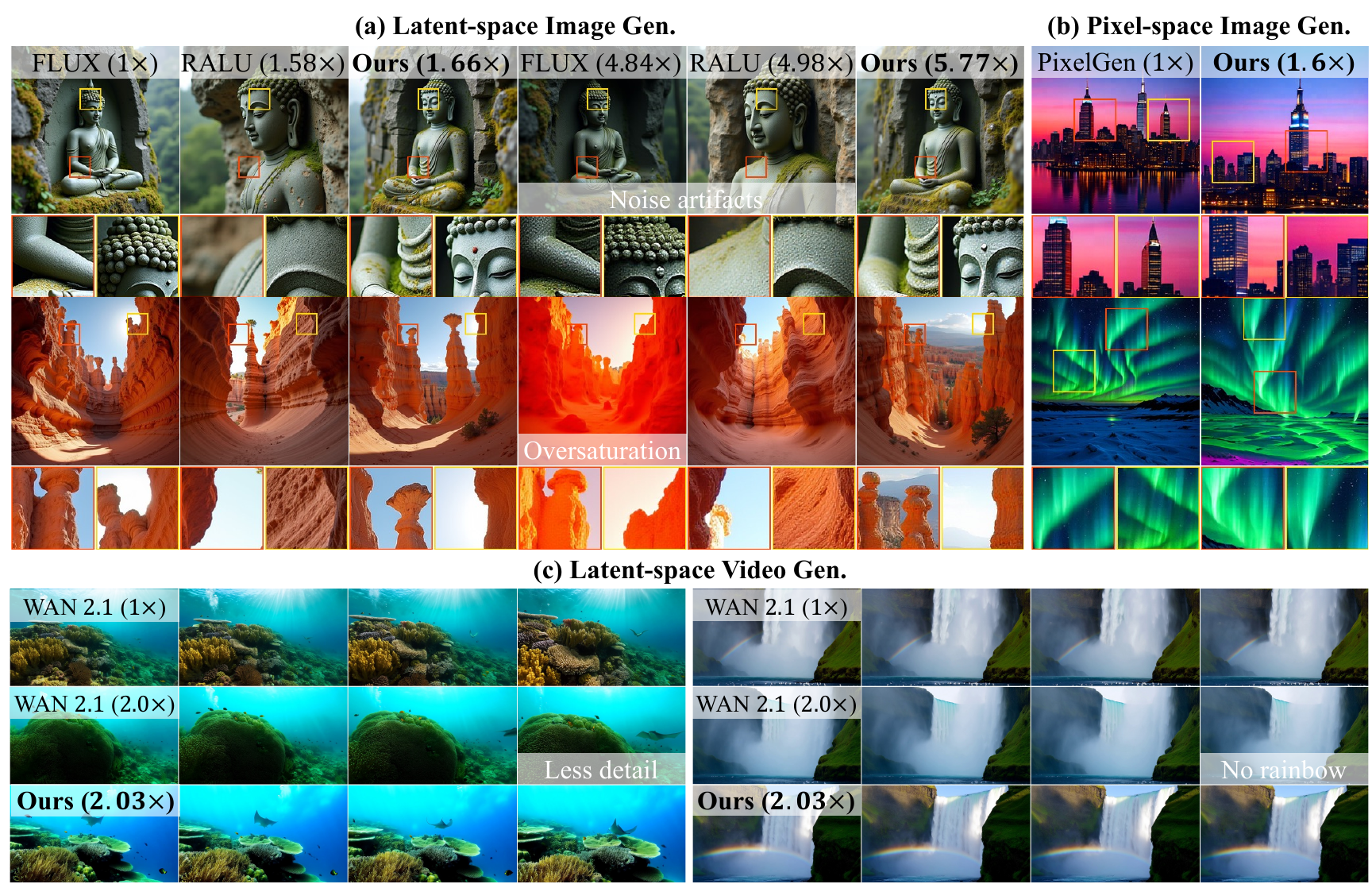}
  \caption{\textbf{Visual generation qualitative comparisons.} For the main comparison of latent-space image generation, our method outperforms the state-of-the-art spatial acceleration method RALU \cite{ralu} in both visual fidelity and inference speed. Across all evaluated modalities (latent/pixel-space image generation and latent-space video generation), we achieve substantial acceleration over standard high-resolution baselines while preserving generation quality.}
  \label{fig:qual}
\end{figure}

\subsection{Setup}
\label{sec:exp:setup}

We extensively evaluate our Spectral Progressive Diffusion framework on latent-space image generation with \textbf{FLUX.1-dev}~\citep{black2024flux} and \textbf{Z-Image}~\citep{zimage}, covering both training-free inference acceleration and fine-tuning with LoRA~\citep{hu2021lora}. We further demonstrate that our framework applies to pixel-space image generation with \textbf{PixelGen-XXL/16 T2I}~\citep{pixelgen} and latent-space video generation with \textbf{WAN~2.1-T2V-1.3B}~\citep{wan}. Each pretrained backbone uses its native resolution and default model-specific inference schedule.

For latent- and pixel-space image generation, we follow RALU's~\citep{ralu} evaluation protocol, reporting ImageReward~\citep{imagereward} for overall quality, CLIP-IQA~\citep{wang2023clipiqa} and NIQE~\citep{mittal2013niqe} for image quality together with T2I-CompBench~\citep{t2icompbench} and GenEval~\citep{geneval} for prompt alignment. For video generation, we report VBench~\citep{vbench} scores across its standard evaluation dimensions. We report total FLOPs integrated over the full denoising trajectory using the same convention as RALU and end-to-end normalized wall-clock speedup. We use $\delta=0.01$ and the Discrete Cosine Transform (DCT) as our spectral transformation $T_\Phi$ across the experiments and include ablation studies on $\delta$, $S$, and $T_\Phi$ in Sec.~\ref{sec:exp:ablations}. We include additional implementation and experimental details in Appendix Secs.~\ref{sec:image_generation} and~\ref{sec:video_generation}.

\vspace{-5pt}
\subsection{Latent- and pixel-space image generation}
\vspace{-5pt}
\label{sec:exp:image}



\paragraph{Training-free acceleration.} In Table~\ref{tab:flux-main} and Fig.~\ref{fig:qual}, we compare our training-free acceleration approach (Sec.~\ref{sec:method:noise-injection}) on FLUX.1-dev with training-free acceleration baselines that implement progressive resolution growing, including Bottleneck Sampling~\citep{bottleneck} and RALU~\citep{ralu}. Additional FLUX.1-dev comparisons are included in Appendix Sec.~\ref{subsec:image_latent}. Our approach consistently improves the speed--quality tradeoff over other methods across all three evaluated speedup regimes. Increasing $S=2$ to $S=3$ further improves efficiency, offering up to $7.36 \times$ FLOPs speedup and $7.09 \times$ wall-clock speedup while maintaining high generation quality.




\vspace{-5pt}
\paragraph{Spectral-transformation-based fine-tuning.} In Table~\ref{tab:zimage-main}, we demonstrate that our spectral-transformation-based fine-tuning method (Sec.~\ref{sec:method:lora}) further improves image quality and efficiency compared to training-free acceleration. Since FLUX.1-dev is guidance-distilled and its base model is not open-sourced, we select Z-Image \citep{zimage} as the base model for our latent-space image fine-tuning experiments. All models are fine-tuned with LoRA~\citep{hu2021lora} using rank 32 for 2,000 iterations in both latent-space and pixel-space image generation experiments. Across both latent- and pixel-space image generation, our fine-tuning framework narrows the gap between training and inference, further improving generation quality while maintaining efficiency gains. RALU relies on model-specific hyperparameters, and its quality degrades on other backbones. On Z-Image at $7.7\times$ speedup, RALU reaches an ImageReward of 0.572, compared with 0.804 for our training-free variant and 0.918 with fine-tuning. In contrast, our power-spectrum-derived schedule uses a single error tolerance $\delta$ and robustly transfers across pretrained models. 
We include additional implementation details and results in Appendix Sec.~\ref{subsec:image_pixel}.
\vspace{-5pt}
\paragraph{Compatibility with other acceleration methods.} Our approach reduces spatial computation within each step, making it complementary to methods that reduce or reuse denoising steps. Combining it with MagCache~\citep{ma2025magcache} increases the speedup on FLUX.1-dev from $3.49\times$ to $4.20\times$, and applying it to step-distilled FLUX.1-schnell~\citep{flux1schnell} yields a further $1.20\times$ speedup, while maintaining or improving quality (Appendix Sec.~\ref{app:compatibility}).

\vspace{-5pt}
\subsection{Latent-space video generation}
\label{sec:exp:video}

We further demonstrate the flexibility and robustness of our training-free acceleration approach on latent-space video generation (Table~\ref{tab:wan_vbench_720p}, Fig.~\ref{fig:qual}). To our knowledge, there are no diffusion acceleration approaches that were validated on all three generation modalities presented in Secs.~\ref{sec:exp:image}--\ref{sec:exp:video}. We observe trends consistent with image generation, where our method achieves significant speedups while preserving generation quality and outperforming TeaCache~\citep{teacache} at comparable speedups. We include additional implementation details and results in Appendix Sec.~\ref{sec:video_generation}.

\vspace{-5pt}
\subsection{Ablation studies}
\label{sec:exp:ablations}

We ablate our core design choices, including error tolerance $\delta$, resolution stages $S$, and spectral transformation $T_\Phi$, on latent-space image generation using FLUX.1-dev. We show the speed--quality Pareto frontier in Fig.~\ref{fig:ablation}(a) for $\delta \in [0.001, 0.1]$, $S \in \{2, 3, 4\}$, and different spectral transformations $T_\Phi$. Consistent with our theory, increasing $\delta$ or $S$ further accelerates generation at the cost of gradual image quality loss. The Discrete Cosine Transform (DCT), Discrete Wavelet Transform (DWT), and Fourier Transform (FFT) achieve similar high-fidelity performance. Qualitative comparisons are given in Appendix Sec.~\ref{sec:ablation}.

\begin{figure}[t]
  \centering
  \includegraphics[width=\linewidth]{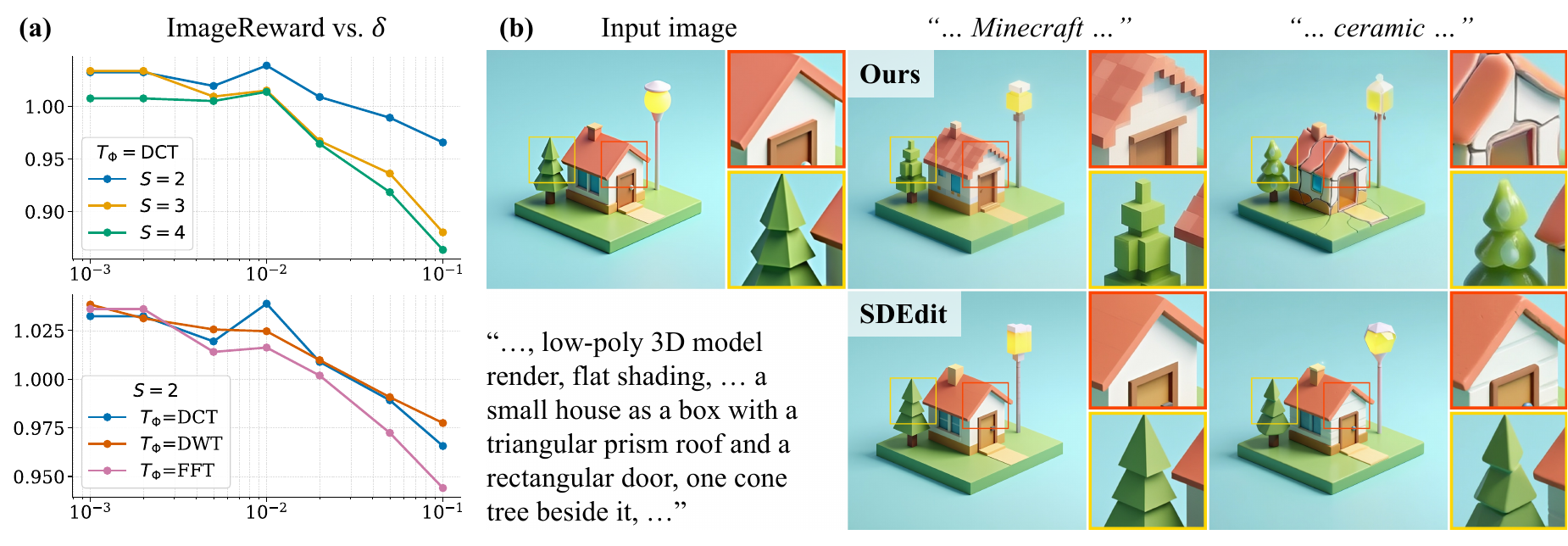}
  \caption{\textbf{(a): Ablation studies on $\delta, S$ and $T_\Phi$.} We observe a clear tradeoff between image quality and efficiency when varying $\delta$ and $S$ as shown in the top plot. As shown in the bottom plot, all three transforms (DCT, DWT, FFT) achieve comparable quality and follow the same quality--efficiency tradeoff as $\delta$ increases. \textbf{(b): Frequency-based image editing.} Our method demonstrates superior prompt alignment and geometric consistency compared to standard SDEdit-style spatial-domain editing under identical step counts. Our method successfully edits the texture and style of the input image, whereas SDEdit~\citep{meng2022sdedit} fails to do so.}
  \label{fig:ablation}
\end{figure}

\subsection{Frequency-based image editing}
\label{sec:exp:applications}
We demonstrate a unique \textit{frequency-based} image editing capability enabled by our method. Specifically, given an input $\mathbf{x}_\text{in}$ to edit, we add noise to the low-frequency part of $T_\Phi(\mathbf{x}_\text{in})$, perform spectral noise expansion and timestep alignment (Sec.~\ref{sec:method:noise-injection}), and then resume denoising from the timestep $t$ corresponding to the resolution schedule with the editing prompt. Fig.~\ref{fig:ablation}(b) shows our method achieves substantially improved prompt alignment and geometric consistency compared to a spatial-domain SDEdit-style baseline \cite{meng2022sdedit}. Additional details and editing results are in Appendix Sec.~\ref{sec:image_editing}.

\section{Discussion}\vspace{-5pt}
\label{sec:conclusion}
In this work, we introduce Spectral Progressive Diffusion, a generation framework that leverages spectral autoregression in diffusion models to match computation to frequencies where signal emerges from noise. 
Through a principled spectral noise expansion mechanism and an optimal resolution schedule derived from the model’s power spectrum, we enable progressive resolution growth along the denoising trajectory of pretrained image and video models, achieving notable speedups. While our approach applies directly in a training-free setting, it implicitly assumes that the underlying model can generalize across varying resolutions. Though modern DiT models are trained on multiple resolutions, the continuous schedules we employ may include intermediate resolutions not seen during training. In practice, using a small number of transitions near the training distribution suffices to achieve substantial efficiency gains in both image and video generation. Our fine-tuning strategy aligns the model with the progressive-resolution trajectory, reducing the training--inference gap and improving quality, suggesting the use of our method during pretraining as a promising direction. While we focus on spatial spectra, extending the framework to temporal frequencies is a natural direction for accelerating video generation. Overall, our approach provides a principled and complementary path for improving generation efficiency via the spectral autoregressive structure of diffusion models while maintaining high generation quality.

\section*{Acknowledgments}
Howard Xiao and Brian Chao are supported by Stanford Graduate Fellowships (SGF). Brian Chao is also supported by the NSF Graduate Research Fellowship Program (GRFP). We thank Google for their partial support on this project. We thank Shengqu Cai, Hansheng Chen, Kiyohiro Nakayama, and Zichun Xu for fruitful discussions. Compute resources were provided by the Marlowe cluster at Stanford University~\cite{marlowe2025}. The authors declare no competing interests.

\clearpage   
{
\small
\bibliographystyle{plainnat}
\bibliography{references_cam_ready}
}

\makeatletter
\let\addcontentsline\arxiv@savedaddcontentsline
\makeatother

\newpage
\setcounter{tocdepth}{2}
\renewcommand{\contentsname}{Appendix Contents}
\tableofcontents
\newpage
\appendix

\section{Supporting definitions and proofs for Section~\ref{sec:method:schedule}}
\label{app:proofs-schedule}

This appendix provides the justification of timestep alignment in Spectral Progressive Diffusion, the deferred proofs of Propositions~\ref{prop:activation} and~\ref{prop:transition} from Sec.~\ref{sec:method:schedule:signal}, along with supporting definitions.

\subsection{Definitions}
\label{app:prelim-schedule}

\paragraph{Per-frequency SNR.}
Following~\citet{kingma2021vdm}, we define the \emph{per-frequency signal-to-noise ratio (SNR)} as the ratio of the clean- and noise-component second moments of the per-frequency forward process in Eq.~\eqref{eq:forward-mode}:
\begin{equation}
  \mathrm{SNR}_\omega(t)
  \;:=\;
  \frac{\mathbb{E}[\,|(1-t)\,x_0^{(\omega)}|^2\,]}{\mathbb{E}[\,|t\,\epsilon^{(\omega)}|^2\,]}
  \;=\;
  \underbrace{\frac{(1-t)^2}{t^2}}_{\mathrm{SNR}(t)}\;\cdot\;P_\omega,
  \label{eq:snr}
\end{equation}
which factorizes into a frequency-independent timestep term $\mathrm{SNR}(t) = (1-t)^2/t^2$ and the frequency-specific signal power $P_\omega$.

\paragraph{Resolution grids and representable frequencies.}
By the Nyquist--Shannon sampling theorem~\citep{nyquist1928certain,shannon1949communication}, a spatial grid of size $H \times W$ can represent 2D frequencies in the rectangle $|\omega_x| \le W/2$, $|\omega_y| \le H/2$. We define the maximum (fully) representable radial frequency as
\begin{equation}
  \omega_{\max}(H, W) \;=\; \frac{\min(H, W)}{2}.
  \label{eq:nyquist-cap}
\end{equation}
Note that our definition of the maximum (fully) representable frequency above is conservative, as it captures only the maximum radial frequency fully supported on the spatial grid $H \times W$.
A stage operating at reduced resolution $(sH, sW)$ for $s \in (0,1]$ can represent only a subset of the representable frequencies at full resolution:
\begin{equation}
  \Omega_s \;:=\; \{\,\omega \in \Omega : |\omega| \le s \cdot \omega_{\max}(H, W)\,\} \;\subseteq\; \Omega.
  \label{eq:nyquist-set}
\end{equation}

\subsection{Justification of timestep alignment}
\label{app:time-align}

We verify that applying the scalar rescaling in Eq.~\eqref{eq:scale-align} and querying the pretrained model at the aligned timestep $\tilde{t}_i$ in Eq.~\eqref{eq:time-align} maps the spectral noise expansion output $\mathbf{x}^{s_{i+1}}_{t_i}$ to a valid flow-matching state $\tilde{\mathbf{x}}^{s_{i+1}}_{\tilde{t}_i}$ at the new resolution scale $s_{i+1} > s_i$. Throughout this subsection we write $r := s_{i+1}/s_i > 1$ for the resolution ratio; the final formulas are stated in both the $r$ and $s_{i+1}/s_i$ forms to tie directly back to Eqs.~\eqref{eq:scale-align}--\eqref{eq:time-align} in Sec.~\ref{sec:method:noise-injection}. The derivations below assume the ideal setting in which the network outputs match the ground-truth flow targets.

\paragraph{1. Setup.}
The spectral noise expansion produces an expanded state $\mathbf{x}^{s_{i+1}}_{t_i}$ in the frequency domain of the expanded resolution grid. For a grid with $N_{s_{i+1}}$ points, the coefficients are:
$$
\left(x^{s_{i+1}}_{t_i}\right)^{(\omega)} =
\begin{cases}
(1-t_i)\left(x_{0}^{s_i}\right)^{(\omega)} + t_i\epsilon^{(\omega)}, & \omega \in \Omega_{s_i}, \\[2pt]
t_i\epsilon'^{(\omega)}, & \omega \in \Omega_{s_{i+1}} \setminus \Omega_{s_i},
\end{cases}
$$
where $\epsilon^{(\omega)}, \epsilon'^{(\omega)} \sim \mathcal{N}(0, 1)$ are independent standard Gaussian variables. Under the orthonormal identification of spectra across resolutions, the clean-signal coefficient at a shared frequency $\omega \in \Omega_{s_i}$ relates to its scale-$s_{i+1}$ counterpart by the resolution ratio $r$:
$$
\left(x_{0}^{s_{i+1}}\right)^{(\omega)} = r \cdot \left(x_{0}^{s_i}\right)^{(\omega)} = (s_{i+1}/s_i) \cdot \left(x_{0}^{s_i}\right)^{(\omega)}, \quad \omega \in \Omega_{s_i}.
$$
For the newly exposed high-frequency slots $\omega \in \Omega_{s_{i+1}} \setminus \Omega_{s_i}$, the clean coefficients are treated as noise-dominated at timestep $t_i$ per Proposition~\ref{prop:activation}, so spectral noise expansion initializes them with the correct noise level.

\paragraph{2. Timestep alignment.}
We seek a scaling factor $\kappa_i$ and an aligned timestep $\tilde{t}_i \in (0, 1)$ such that, for every $\omega \in \Omega_{s_{i+1}}$, the transformed state satisfies:
$$
\left(\tilde{x}^{s_{i+1}}_{\tilde{t}_i}\right)^{(\omega)} = \kappa_i\left(x^{s_{i+1}}_{t_i}\right)^{(\omega)} = (1-\tilde{t}_i)\left(x_{0}^{s_{i+1}}\right)^{(\omega)} + \tilde{t}_i\tilde\epsilon^{(\omega)}, \quad \tilde\epsilon^{(\omega)} \sim \mathcal{N}(0,1) \text{ i.i.d.}
$$
By matching the noise coefficients across all bands and the signal coefficients on the shared band $\Omega_{s_i}$, we derive two scalar conditions:
$$
\kappa_i \cdot t_i = \tilde{t}_i, \quad \kappa_i(1-t_i) = r(1-\tilde{t}_i).
$$
Substituting $\kappa_i = \tilde{t}_i/t_i$ from the noise match into the signal match equation gives:
$$
\frac{\tilde{t}_i}{t_i}(1-t_i) = r(1-\tilde{t}_i)
$$
Solving for $\tilde{t}_i$ and re-expressing the result in $s_{i+1}/s_i$ notation gives:
$$
\tilde{t}_i = \frac{r \cdot t_i}{1 + (r - 1)t_i} = \frac{(s_{i+1}/s_i) \cdot t_i}{1 + ((s_{i+1}/s_i) - 1)t_i}.
$$
Consequently, the scaling factor $\kappa_i$ is:
$$
\kappa_i = \frac{r}{1 + (r - 1)t_i} = \frac{s_{i+1}/s_i}{1 + ((s_{i+1}/s_i) - 1)t_i}.
$$
Substituting these back into the setup confirms that $\left(\tilde{x}^{s_{i+1}}_{\tilde{t}_i}\right)^{(\omega)} = (1-\tilde{t}_i)\left(x_{0}^{s_{i+1}}\right)^{(\omega)} + \tilde{t}_i\tilde\epsilon^{(\omega)}$ holds for all $\omega \in \Omega_{s_{i+1}}$. This proves that the resulting scaled and timestep-aligned state $\tilde{\mathbf{x}}^{s_{i+1}}_{\tilde{t}_i}$ is a valid flow-matching state at resolution $s_{i+1}$ and timestep $\tilde{t}_i$. $\square$

\subsection{Proof of Proposition~\ref{prop:activation}: Per-frequency $\delta$-optimal activation time}
\label{app:proof-prop1}

We first establish a lemma giving the Bayes-optimal velocity error in closed form under a simplified modeling assumption $x_0^{(\omega)} \sim \mathcal{N}(0, P_\omega)$. In practice, we measure the power spectrum on centered (i.e. zero-mean) inputs $\mathbf{x}_0$ to get the power spectrum $P_\omega$, and this modeling assumption serves as a second moment approximation that allows the analytical derivation of our $\delta$-optimal resolution schedule. The proof of Proposition~\ref{prop:activation} then follows by solving the $\delta$-bound on the Bayes-optimal velocity error for $t$ to obtain the activation time $t_\omega$ in Eq.~\eqref{eq:activation}.

\begin{lemma}[Bayes-optimal velocity error] Let $v^{*(\omega)}$ be the per-frequency component of the Bayes-optimal velocity predictor $\mathbf{v}^{*}$ introduced in Sec.~\ref{sec:method:prelim:basics}:

\begin{equation}
  \mathbf{v}^*(\mathbf{x}_t, t) \;=\; \mathbb{E}\!\left[\,\boldsymbol{\epsilon} - \mathbf{x}_0 \,\big|\, \mathbf{x}_t\,\right].
  \label{eq:optimal-v}
\end{equation}
Under the modeling assumption $x_0^{(\omega)} \sim \mathcal{N}(0, P_\omega)$, $v^{*(\omega)}$ deviates from the noise-only prediction $\epsilon^{(\omega)}$ in expectation by
\begin{equation}
  \mathbb{E}\!\left[\bigl|v^{*(\omega)}(\mathbf{x}_t, t) - \epsilon^{(\omega)}\bigr|^2\right]
  \;=\; \frac{\mathrm{SNR}_\omega(t)\,(1 + P_\omega)}{1 + \mathrm{SNR}_\omega(t)}.
  \label{eq:velocity-error-optimal}
\end{equation}
\end{lemma}


\begin{proof}
\leavevmode

\paragraph{1. Assumptions.} The proof assumes a linear-Gaussian model, where:
\begin{enumerate}
  \item[(A1)] we model the clean data spectrum as zero-mean Gaussian $x_0^{(\omega)} \sim \mathcal{N}(0,\, P_\omega)$ with second moment $P_\omega > 0$;
  \item[(A2)] we assume $\epsilon^{(\omega)} \sim \mathcal{N}(0, 1)$, by orthonormality of $\Phi$ applied to $\boldsymbol{\epsilon} \sim \mathcal{N}(0, I)$;
  \item[(A3)] we assume $x_0^{(\omega)} \perp \epsilon^{(\omega)}$, i.e. the signal and noise coefficients are independent.
\end{enumerate}
(A1) is the simplified modeling assumption; (A2) and (A3) follow directly from the flow-matching forward process. Under (A1), the triple $(x_0^{(\omega)},\, \epsilon^{(\omega)},\, x_t^{(\omega)})$ is jointly Gaussian, so conditional expectations are linear and closed-form.

\paragraph{2. Second-moment quantities.} From the per-frequency forward process in Eq.~\eqref{eq:forward-mode} and (A1)--(A3),
\begin{align}
  \mathbb{E}[x_t^{(\omega)}]                     &= 0, \qquad
  \mathrm{Var}(x_t^{(\omega)}) \;=\; (1-t)^2 P_\omega + t^2,  \\
  \mathrm{Cov}(x_t^{(\omega)},\, x_0^{(\omega)})      &= (1-t)\,P_\omega, \qquad
  \mathrm{Cov}(x_t^{(\omega)},\, \epsilon^{(\omega)}) \;=\; t.
\end{align}

\paragraph{3. Bayes-optimal velocity as a linear predictor.}
Under (A1)--(A3), the triplet $(x_0^{(\omega)},\, \epsilon^{(\omega)},\, x_t^{(\omega)})$ is jointly Gaussian with zero mean, since $x_t^{(\omega)} = (1-t) x_0^{(\omega)} + t\,\epsilon^{(\omega)}$ is a linear combination of two independent zero-mean Gaussians. For any pair of zero-mean jointly Gaussian random variables $(X, Y)$ with $\mathrm{Var}(Y) > 0$, the conditional expectation is a linear function of $Y$:
\begin{equation}
  \mathbb{E}[X \,|\, Y] \;=\; \frac{\mathrm{Cov}(X, Y)}{\mathrm{Var}(Y)}\,Y.
  \label{eq:gauss-cond}
\end{equation}

We apply \eqref{eq:gauss-cond} twice, using the second moments computed above. First, with $X = x_0^{(\omega)}$ and $Y = x_t^{(\omega)}$,
\begin{equation}
  \mathbb{E}\bigl[x_0^{(\omega)} \,\big|\, x_t^{(\omega)}\bigr]
  \;=\; \frac{(1-t)\,P_\omega}{(1-t)^2 P_\omega + t^2}\, x_t^{(\omega)}.
  \label{eq:cond-x0}
\end{equation}
Second, with $X = \epsilon^{(\omega)}$ and $Y = x_t^{(\omega)}$,
\begin{equation}
  \mathbb{E}\bigl[\epsilon^{(\omega)} \,\big|\, x_t^{(\omega)}\bigr]
  \;=\; \frac{t}{(1-t)^2 P_\omega + t^2}\, x_t^{(\omega)}.
  \label{eq:cond-eps}
\end{equation}

The Bayes-optimal velocity at frequency $\omega$ is $v^{*(\omega)} = \mathbb{E}\bigl[\epsilon^{(\omega)} - x_0^{(\omega)} \,\big|\, \mathbf{x}_t\bigr]$, the $\phi_\omega$-projection of the full-input optimal velocity in Eq.~\eqref{eq:optimal-v}. Under the modeling assumption that $\mathbf{x}_0$ has diagonal covariance in basis $\Phi$ (i.e., (A1) holds independently for every $\omega$), the joint Gaussian $(\mathbf{x}_0, \boldsymbol{\epsilon}, \mathbf{x}_t)$ factorizes across $\Omega$, so the posterior collapses to its single-coefficient marginal:
\begin{equation}
  \mathbb{E}\bigl[\epsilon^{(\omega)} - x_0^{(\omega)} \,\big|\, \mathbf{x}_t\bigr]
  \;=\; \mathbb{E}\bigl[\epsilon^{(\omega)} - x_0^{(\omega)} \,\big|\, x_t^{(\omega)}\bigr].
  \label{eq:full-to-scalar}
\end{equation}
Subtracting \eqref{eq:cond-x0} from \eqref{eq:cond-eps} and combining with \eqref{eq:full-to-scalar},
\begin{equation}
  v^{*(\omega)} \;=\; A \cdot x_t^{(\omega)},
  \qquad
  A \;:=\; \frac{t - (1-t)\,P_\omega}{(1-t)^2 P_\omega + t^2}.
  \label{eq:vstar-linear}
\end{equation}

\paragraph{4. Substitute and simplify.} Plugging $x_t^{(\omega)} = (1-t)\, x_0^{(\omega)} + t\,\epsilon^{(\omega)}$ into \eqref{eq:vstar-linear} and collecting terms in $x_0^{(\omega)}$ and $\epsilon^{(\omega)}$,
\begin{equation}
  v^{*(\omega)} - \epsilon^{(\omega)}
  \;=\; A\bigl[(1-t)\,x_0^{(\omega)} + t\,\epsilon^{(\omega)}\bigr] - \epsilon^{(\omega)}
  \;=\; A(1-t)\,x_0^{(\omega)} \,+\, (At - 1)\,\epsilon^{(\omega)}.
\end{equation}
Squaring and taking expectation, using (A1)--(A3) (so $\mathbb{E}[(x_0^{(\omega)})^2] = P_\omega$, $\mathbb{E}[(\epsilon^{(\omega)})^2] = 1$, and $\mathbb{E}[x_0^{(\omega)}\,\epsilon^{(\omega)}] = 0$), the cross term vanishes:
\begin{equation}
  \mathbb{E}\!\left[\bigl|v^{*(\omega)} - \epsilon^{(\omega)}\bigr|^2\right]
  \;=\; A^2 (1-t)^2\,P_\omega \,+\, (At - 1)^2.
\end{equation}
Expanding $(At-1)^2 = A^2 t^2 - 2At + 1$ and grouping the $A^2$ terms,
\begin{equation}
  \mathbb{E}\!\left[\bigl|v^{*(\omega)} - \epsilon^{(\omega)}\bigr|^2\right]
  \;=\; A^2\!\left[(1-t)^2 P_\omega + t^2\right] - 2At + 1.
  \label{eq:verror-expanded}
\end{equation}
Denote $D := (1-t)^2 P_\omega + t^2 = \mathrm{Var}(x_t^{(\omega)})$; from Eq.~\eqref{eq:vstar-linear},
\begin{equation}
  A \;=\; \frac{t - (1-t)\, P_\omega}{D}.
  \label{eq:A-in-D}
\end{equation}
Substituting into Eq.~\eqref{eq:verror-expanded} term by term,
\begin{align}
  A^2\!\left[(1-t)^2 P_\omega + t^2\right] \;&=\; A^2 D \;=\; \frac{(t - (1-t)\, P_\omega)^2}{D}, \\
  2At \;&=\; \frac{2t\,(t - (1-t)\, P_\omega)}{D},
\end{align}
so the right-hand side of Eq.~\eqref{eq:verror-expanded} combines into a single fraction over $D$:
\begin{equation}
  \mathbb{E}\!\left[\bigl|v^{*(\omega)} - \epsilon^{(\omega)}\bigr|^2\right]
  \;=\; \frac{(t - (1-t)\, P_\omega)^2 \,-\, 2t\,(t - (1-t)\, P_\omega) \,+\, D}{D}.
  \label{eq:verror-singlefrac}
\end{equation}
Expanding $(t - (1-t)\, P_\omega)^2 = t^2 - 2t(1-t) P_\omega + (1-t)^2 P_\omega^2$, the numerator of Eq.~\eqref{eq:verror-singlefrac} simplifies as
\begin{align}
  & \bigl[t^2 - 2t(1-t) P_\omega + (1-t)^2 P_\omega^2\bigr]
    - \bigl[2t^2 - 2t(1-t) P_\omega\bigr]
    + \bigl[(1-t)^2 P_\omega + t^2\bigr] \nonumber\\
  ={}& (1-t)^2 P_\omega^2 + (1-t)^2 P_\omega
    \;=\; (1-t)^2\,P_\omega\,(1 + P_\omega).
\end{align}
Substituting back into Eq.~\eqref{eq:verror-singlefrac},
\begin{equation}
  \mathbb{E}\!\left[\bigl|v^{*(\omega)} - \epsilon^{(\omega)}\bigr|^2\right]
  \;=\; \frac{(1-t)^2\,P_\omega\,(1 + P_\omega)}{(1-t)^2 P_\omega + t^2}.
\end{equation}

\paragraph{5. Rewrite via $\mathrm{SNR}_\omega(t)$.} Using $\mathrm{SNR}_\omega(t) = (1-t)^2 P_\omega / t^2$ from Eq.~\eqref{eq:snr}, the denominator is $(1-t)^2 P_\omega + t^2 = t^2(1 + \mathrm{SNR}_\omega(t))$ and the numerator equals $t^2\,\mathrm{SNR}_\omega(t)\,(1 + P_\omega)$. The $t^2$ cancels, recovering the right-hand side of Eq.~\eqref{eq:velocity-error-optimal} and completing the proof of the lemma.

\end{proof}

\begin{proposition}[Per-frequency $\delta$-optimal activation time] Denote the per-frequency Bayes-optimal velocity predictor by $v^{*(\omega)}(\mathbf{x}_t, t)$, whose spatial-domain counterpart is defined in Eq.~\eqref{eq:optimal-v}. Under the simplified modeling assumption $x_0^{(\omega)} \sim \mathcal{N}(0, P_\omega)$ with $P_\omega > 0$, we have
\begin{equation}
  \mathbb{E}\!\left[\bigl|v^{*(\omega)}(\mathbf{x}_t, t) - \epsilon^{(\omega)}\bigr|^2\right] \;\le\; \delta
  \label{eq:noise-dominated-appendix}
\end{equation}
for all $t \ge t_\omega$ and $\delta\in(0, 1)$, where
\begin{equation}
  t_\omega \;:=\; \frac{1}{1 + \sqrt{\dfrac{\delta}{P_\omega\,(1 + P_\omega - \delta)}}}.
  \label{eq:activation-appendix}
\end{equation}
\end{proposition}

\begin{proof}
\leavevmode

\paragraph{Solving the $\delta$-bound for $t$.} Substituting the closed-form error into the error-tolerance bound $\mathbb{E}[|v^{*(\omega)} - \epsilon^{(\omega)}|^2] \le \delta$ of Proposition~\ref{prop:activation} yields
\begin{equation}
  \frac{\mathrm{SNR}_\omega(t)\,(1 + P_\omega)}{1 + \mathrm{SNR}_\omega(t)} \;\le\; \delta.
\end{equation}
Clearing the (positive) denominator and collecting terms in $\mathrm{SNR}_\omega(t)$,
\begin{equation}
  \mathrm{SNR}_\omega(t)\,(1 + P_\omega - \delta) \;\le\; \delta,
  \qquad \text{i.e.,} \qquad
  \mathrm{SNR}_\omega(t) \;\le\; \frac{\delta}{1 + P_\omega - \delta},
\end{equation}
where the rearrangement is valid because $\delta < 1 \le 1 + P_\omega$ keeps $1 + P_\omega - \delta > 0$. Substituting $\mathrm{SNR}_\omega(t) = (1-t)^2 P_\omega / t^2$ and taking the positive square root (all quantities positive for $t \in (0, 1)$),
\begin{equation}
  \frac{1-t}{t} \;\le\; \sqrt{\frac{\delta}{P_\omega\,(1 + P_\omega - \delta)}}.
\end{equation}
Adding $1$ to both sides gives $1/t \le 1 + \sqrt{\delta / (P_\omega(1 + P_\omega - \delta))}$, which inverts to
\begin{equation}
  t \;\ge\; \frac{1}{1 + \sqrt{\dfrac{\delta}{P_\omega\,(1 + P_\omega - \delta)}}} \;=\; t_\omega,
\end{equation}
recovering Eq.~\eqref{eq:activation-appendix}.
\end{proof}

\subsection{Proof of Proposition~\ref{prop:transition}: Per-resolution $\delta$-optimal transition time}
\label{app:proof-prop2}

\begin{proposition}[Per-resolution $\delta$-optimal transition time] Under the setting of Proposition~\ref{prop:activation}, and assuming $P_\omega$ is monotonically decreasing in $|\omega|$ based on the power-law decay in Eq.~\eqref{eq:power-law}, for any resolution scale $s_i, s_{i+1} \in (0, 1]$, the optimal transition time from scale $s_i$ up to $s_{i+1}$ is
\begin{equation}
  t_i^* \;:=\; \min_{\omega \in \Omega_{s_i}}\, t_\omega \;=\; t_{\omega=s_i \cdot \omega_{\max}(H,W)},
  \label{eq:stage-transition-appendix}
\end{equation}
where $\Omega_{s_i}$ is the set of frequencies representable on the $(s_iH, s_iW)$ grid.
\end{proposition}

\begin{proof}
\leavevmode

We show that the activation time $t_\omega$ of Eq.~\eqref{eq:activation} is strictly decreasing in $|\omega|$ under the assumptions of the proposition; the claim then follows immediately.

\paragraph{1. $t_\omega$ is strictly increasing in $P_\omega$.}
For $\delta \in (0, 1)$ and $P_\omega > 0$, the quadratic $P_\omega\,(1 + P_\omega - \delta) = P_\omega^2 + (1-\delta)\, P_\omega$ has derivative $2 P_\omega + (1-\delta) > 0$, hence is strictly increasing in $P_\omega$. Therefore the radicand $\delta / \bigl(P_\omega(1 + P_\omega - \delta)\bigr)$ in Eq.~\eqref{eq:activation} is strictly decreasing in $P_\omega$, and so is its square root. Hence
\begin{equation}
  t_\omega \;=\; \frac{1}{1 + \sqrt{\dfrac{\delta}{P_\omega\,(1 + P_\omega - \delta)}}}
\end{equation}
is strictly increasing in $P_\omega$.

\paragraph{2. $P_\omega$ is strictly decreasing in $|\omega|$.}
The power-law assumption in Eq.~\eqref{eq:power-law} gives $P_\omega \propto |\omega|^{-\beta}$ with $\beta > 0$. Thus $P_\omega$ is strictly decreasing in $|\omega|$.

\paragraph{3. Minimum over $\Omega_s$.}
Combining the above observations, $t_\omega$ is strictly decreasing in $|\omega|$. The set $\Omega_s = \{\omega \in \Omega : |\omega| \le s \cdot \omega_{\max}(H, W)\}$ consists exactly of the frequencies with radius at most $s \cdot \omega_{\max}(H, W)$, so
\begin{equation}
  \min_{\omega \in \Omega_s}\, t_\omega \;=\; t_{\,s \cdot \omega_{\max}(H, W)},
\end{equation}
which is Eq.~\eqref{eq:stage-transition}.
\end{proof}

\clearpage
\section{Empirical validation of the per-frequency activation criterion}
\label{sec:motivation}

We empirically validate Proposition~\ref{prop:activation} with a simple \textit{spectral noise passthrough} experiment.
While maintaining full-resolution processing at all denoising steps, for each frequency $\omega \in \Omega$, if the current timestep $t$ is before its activation time $t_\omega$ (i.e., $t > t_\omega$), we replace that frequency component with $t \cdot T_\Phi(\boldsymbol{\epsilon})^{(\omega)}$, where $\boldsymbol{\epsilon} \sim \mathcal{N}(0, I)$ is the initial noise at $t = 1$. We use the Discrete Cosine Transform (DCT) as $T_\Phi$ in this experiment and sweep over multiple $\delta$ values. As described in Sec.~\ref{sec:method:schedule:signal}, a larger $\delta$ corresponds to transitioning to the high-resolution stage later in denoising with more speedup, while a smaller $\delta$ indicates transitioning to high-resolution earlier in denoising with less speedup.

From Fig. \ref{fig:spectral-passthrough}, we observe that at smaller $\delta \in [0.0001, 0.001]$, there is almost no observable difference compared to native full-resolution generation. As $\delta$ increases, we see a clear transition to blurry and lower-quality images around $\delta=0.05$, and larger values further degrade generation quality. This finding provides empirical support for the spectral autoregression property of diffusion models and Proposition~\ref{prop:activation}, which serves as the foundation of our Spectral Progressive Diffusion framework and the $\delta$-optimal resolution schedule. It also implies that the Gaussian modeling assumption in Proposition~\ref{prop:activation} serves as a reasonable approximation for subsequent analytical resolution schedule derivations.

\begin{figure}[h!]
    \centering
    \includegraphics[width=\linewidth]{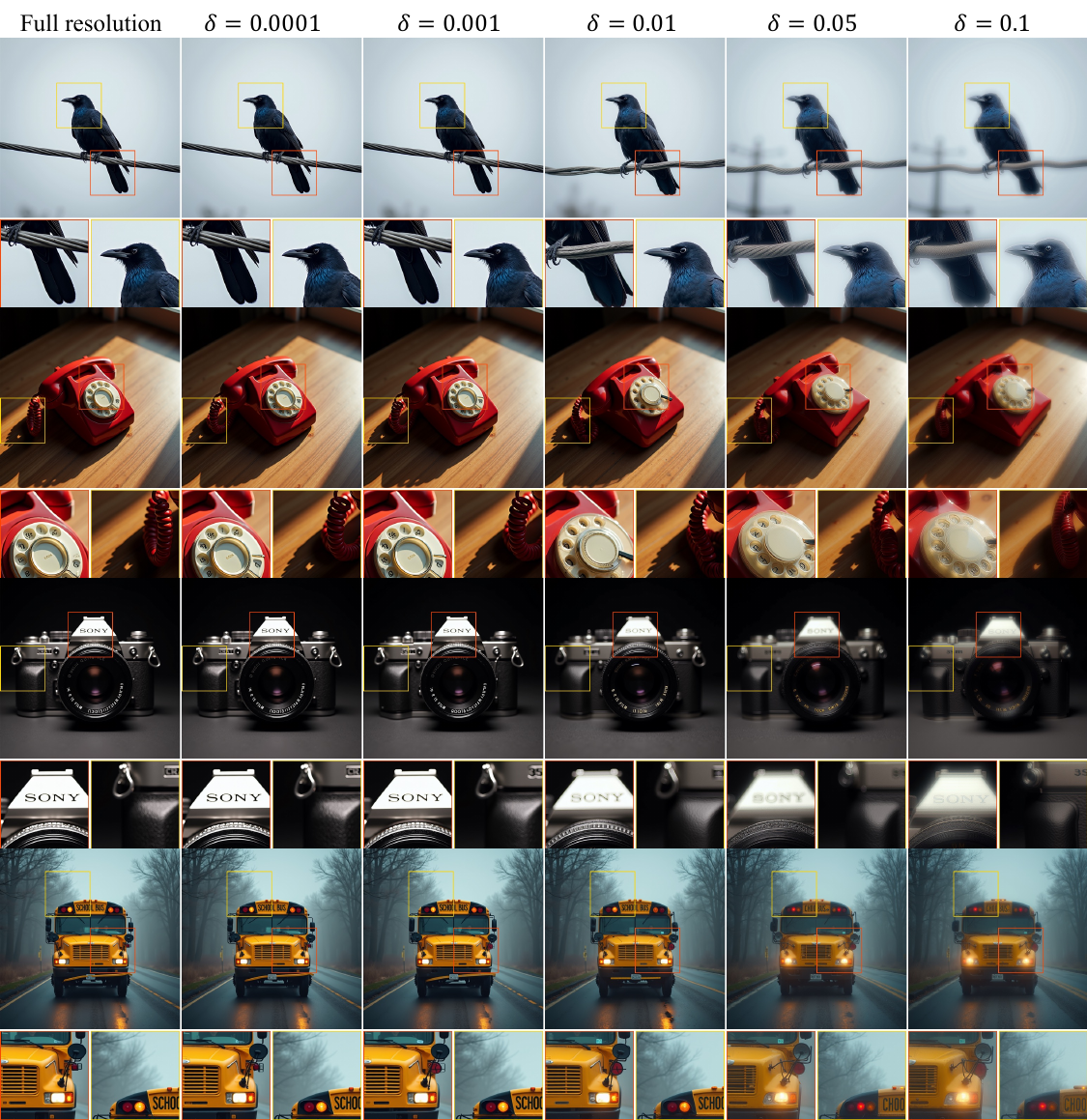}
    \caption{\textbf{Spectral noise passthrough experiment.} At smaller $\delta \in [0.0001, 0.001]$, there is almost no observable difference compared to native full-resolution generation. As larger $\delta$ values cause high-frequency replacement to persist later in the denoising trajectory, we observe increasingly blurry and distorted results (i.e., ghosting artifacts and ``CHOOLBUS'' instead of ``SCHOOLBUS''). }
    \label{fig:spectral-passthrough}
\end{figure}

\clearpage
\section{Additional details on power spectrum measurement}
\label{sec:power_spec}
To remain consistent with the zero-mean Gaussian modeling assumption in Sec.~\ref{sec:method:schedule:signal}, we center each data sample before measuring the radially averaged power spectrum of the VAE latent space of FLUX.1-dev~\citep{black2024flux} and Z-Image~\citep{zimage} for image generation, and the VAE latent space of WAN 2.1~\citep{wan} for video generation.
We fit the power-law model $P_\omega \approx A \cdot |\omega|^{-\beta}$ and report the fitted $A$ and $\beta$ values for image latents, video latents, and image pixels in Table~\ref{tab:power-spec-fit}.
These fitted values instantiate the per-frequency signal power $P_\omega$ in the activation-time and transition-time formulas in Sec.~\ref{sec:method:schedule:signal}, which we use to compute the $\delta$-optimal resolution schedule for each model family.

For image latents, we run the FLUX.1-dev VAE on 100,000 images from the Aesthetic-Train-V2~\citep{zhang2025diffusion4k} dataset, center-cropped and resized to $1024^2$ resolution. We then average power over channels and radial frequency bins before fitting the spectrum. For video latents, we use the same procedure with the WAN 2.1 VAE on 10,000 videos from the VChitect-T2V-Dataverse~\citep{fan2025vchitect} dataset, averaging over channels, latent frames, and radial frequency bins to estimate the spatial power spectrum. We further measure the pixel-space power spectrum on images from the Aesthetic-Train-V2 dataset, center-cropped and resized to $512^2$, the native resolution supported by the PixelGen~\citep{pixelgen} model.

\begin{table}[h!]
  \centering
  \small
  \caption{\textbf{Power-law fits to measured power spectra.} We fit $P_\omega \approx A \cdot |\omega|^{-\beta}$ on centered samples for image latents, video latents, and image pixels.}
  \label{tab:power-spec-fit}
  \resizebox{\linewidth}{!}{
  \begin{tabular}{lcccc}
    \toprule
    Spectrum & $\beta$ & $A$ & $R^2$ & Samples \\
    \midrule
    FLUX.1-dev (Z-Image) latent & 1.9155 & 203.62 & 0.978 & 100,000 images \\
    WAN 2.1 latent & 2.4227 & 219.48 & 0.999 & 10,000 videos \\
    Pixel $512^2$ & 2.4493 & 3745.46 & 1.000 & 100,000 images \\
    \bottomrule
  \end{tabular}
  }
\end{table}

\clearpage
\section{Additional details and results on image generation}
\label{sec:image_generation}

\subsection{Training-free acceleration of latent-space image generation}
\label{subsec:image_latent}

For training-free inference acceleration of image generation, we use NVIDIA A100 GPUs with 80~GB of VRAM to match RALU~\citep{ralu}'s evaluation protocol. For training-free inference acceleration of video generation, as well as image editing and ablation studies, we use NVIDIA H100 GPUs with 80~GB of VRAM. For training-free acceleration of latent-space image generation, we use FLUX.1-dev~\citep{black2024flux} as the base model and follow the evaluation protocol of RALU~\citep{ralu}. Table~\ref{tab:flux-extended} extends Table~\ref{tab:flux-main} with additional training-free acceleration baselines from RALU. We organize the comparison into three wall-clock speedup tiers:

\begin{itemize}
    \item \textbf{Small speedup}: the 50-step native full-resolution baseline, compared against RALU (50 steps) and our method (40 steps), yielding approximately $1.7\times$ speedup;
    \item \textbf{Medium speedup}: the 10-step native full-resolution baseline, used as an iso-compute reference for RALU (15 steps) and our method (12 steps), yielding approximately $5\times$ speedup; and
    \item \textbf{Large speedup}: the 7-step native full-resolution baseline, used as an iso-compute reference for RALU (10 steps) and our method (10 steps), yielding approximately $7\times$ speedup.
\end{itemize}
All speedups are calculated relative to the wall-clock latency of the 50-step FLUX.1-dev baseline on a single A100 GPU, following RALU~\citep{ralu}. As in Table~\ref{tab:flux-main}, we evaluate image quality using CLIP-IQA~\citep{wang2023clipiqa} and NIQE~\citep{mittal2013niqe}, prompt alignment using T2I-CompBench~\citep{t2icompbench} and GenEval~\citep{geneval}, and overall quality using ImageReward~\citep{imagereward}.
Following RALU's evaluation protocol, ImageReward, CLIP-IQA, and NIQE are averaged over 5,000 images generated from MS-COCO~\citep{lin2014microsoft} validation prompts. For compositional alignment, GenEval is evaluated with four random seeds per prompt on its object-focused prompt suite, and T2I-CompBench is evaluated on its spatial, non-spatial, and complex prompt subsets with four random seeds per prompt. Unless otherwise specified, all metrics are computed using their official or default evaluation settings.

We keep the error tolerance fixed at $\delta=0.01$, use $T_\Phi = \text{DCT}$, and compare $S=2$ and $S=3$ resolution-stage settings. To align with the multi-resolution training distribution of FLUX.1-dev, we use latent dimensions $64^2$ ($s_1 = 0.5$) and $128^2$ $(s_2 = 1.0)$ (pixel dimensions $512^2$ and $1024^2$) for $S=2$, and latent dimensions $32^2$ $(s_1 = 0.25)$, $64^2$ $(s_2 = 0.5)$, and $128^2$ ($s_3 = 1.0$) (pixel dimensions $256^2$, $512^2$, and $1024^2$) for $S=3$.

By combining spectral transformations for progressive-resolution generation with our derived resolution schedule, our method achieves a better speed-quality tradeoff than the reduced-step FLUX.1-dev baselines and existing temporal and spatial diffusion acceleration methods, including progressive-resolution methods such as RALU~\citep{ralu}. In the largest speedup tier, our $S=3$ setting reaches $7.09\times$ wall-clock speedup and $7.36\times$ FLOPs speedup while preserving competitive image quality. We further provide extended qualitative baseline comparisons in Figures~\ref{fig:flux-1}--\ref{fig:flux-3}. FLUX.1-dev with reduced steps degrades image quality and exhibits over-saturation artifacts, while RALU introduces noticeable noise artifacts in reduced-step settings.

\begin{table*}[t!]
    \centering
    \caption{\textbf{Extended quantitative comparisons on training-free latent-space image generation.} We evaluate FLUX.1-dev~\citep{black2024flux} at $1024^2$ resolution and group methods into the same $1\times$, $5\times$, and $7\times$ wall-clock speedup tiers as Table~\ref{tab:flux-main}. Rows shared with Table~\ref{tab:flux-main} use the same values. Speedup is calculated relative to the 50-step FLUX.1-dev baseline wall-clock runtime on a single A100 GPU. \textcolor{blue}{\textbf{T}} and \textcolor{red}{\textbf{S}} denote temporal and spatial acceleration methods, respectively. Baseline metrics are copied from RALU~\citep{ralu} under the same evaluation protocol.}
    \label{tab:flux-extended}
    \resizebox{\textwidth}{!}{
    \begin{scriptsize}
    \begin{tabular}{cccccccccc}
         \toprule
        \multirow{2}{*}{Method} & \multirow{2}{*}{Accel.} & \multirow{2}{*}{Latency (s) $\downarrow$} & \multirow{2}{*}{Speedup ($\times$) $\uparrow$} & \multirow{2}{*}{TFLOPs $\downarrow$} & Overall & \multicolumn{2}{c}{Image quality} & \multicolumn{2}{c}{Prompt alignment} \\
        \cmidrule(lr){6-10}
        & & & & & ImageReward $\uparrow$ & CLIP-IQA $\uparrow$ & NIQE $\downarrow$ & T2I-Comp. $\uparrow$ & GenEval $\uparrow$  \\
        \midrule
        FLUX (50 steps) & - & 25.1 & 1.00$\times$ & 2991.01 & \textbf{1.095} & 0.707 & 6.75 & \textbf{0.634} & \textbf{0.698}  \\
        RALU~\citep{ralu}   & \textcolor{red}{\textbf{S}}  & 15.85 & 1.58$\times$ & \underline{1749.94} & 1.028 & \underline{0.712} & \textbf{6.07} & 0.613 & 0.648 \\
        \textbf{Ours ($S=2$)}     & \textcolor{red}{\textbf{S}}       & \underline{15.16} & \underline{1.66}$\times$ & 1755.22 & \underline{1.049} & \textbf{0.719} & 6.43 & \underline{0.617} & \underline{0.654} \\
        \textbf{Ours ($S=3$)}     & \textcolor{red}{\textbf{S}}       & \textbf{14.54} & \textbf{1.73}$\times$ & \textbf{1672.04} & 1.015 & 0.711 & \underline{6.33} & 0.593 & 0.640 \\
        \cmidrule(lr){1-10}
        FLUX (10 steps) & \textcolor{blue}{\textbf{T}} & 5.18 & 4.84$\times$ & 610.02 & 0.981 & 0.679 & 6.93 & 0.618 & 0.647 \\
        $\Delta$-DiT~\citep{deltadit}      & \textcolor{blue}{\textbf{T}} & 7.42 & 3.38$\times$ & 772.10 & 0.102 & 0.487 & 9.60 & 0.306 & 0.397 \\
        ToCa~\citep{toca}                  & \textcolor{blue}{\textbf{T}} & 15.5 & 1.62$\times$ & 601.12 & -1.827 & 0.253 & 10.6 & 0.259 & 0.137 \\
        TeaCache~\citep{teacache}          & \textcolor{blue}{\textbf{T}} & 5.23 & 4.80$\times$ & 610.59 & 0.944 & 0.665 & 7.92 & 0.620 & 0.647 \\
        TaylorSeer~\citep{taylorseer}      & \textcolor{blue}{\textbf{T}} & 9.34 & 2.69$\times$ & 556.72 & 0.972 & 0.684 & 6.77 & 0.594 & 0.619 \\
        Bottleneck~\citep{bottleneck}      & \textcolor{red}{\textbf{S}}  & 5.37 & 4.67$\times$ & 571.23 & 0.889 & 0.661 & 9.16 & 0.620 & \textbf{0.687} \\
        RALU~\citep{ralu}                  & \textcolor{red}{\textbf{S}}  & 5.04 & 4.98$\times$ & 540.47 & 1.022 & \underline{0.700} & \textbf{6.43} & \textbf{0.626} & 0.652 \\
        \textbf{Ours ($S=2$)}     & \textcolor{red}{\textbf{S}}       & \underline{4.35} & \underline{5.77}$\times$ & \underline{500.34} & \textbf{1.059} & 0.696 & 6.69 & \underline{0.624} & \underline{0.655} \\
        \textbf{Ours ($S=3$)}     & \textcolor{red}{\textbf{S}}       & \textbf{4.12} & \textbf{6.09}$\times$ & \textbf{469.15} & \underline{1.042} & \textbf{0.701} & \underline{6.53} & 0.623 & 0.637 \\
        \cmidrule(lr){1-10}
        FLUX (7 steps) & \textcolor{blue}{\textbf{T}} & 3.79 & 6.62$\times$ & 431.45 & 0.920 & 0.660 & 8.25 & 0.594 & 0.583 \\
        TeaCache~\citep{teacache}          & \textcolor{blue}{\textbf{T}} & 4.21 & 5.96$\times$ & 431.83 & 0.733 & 0.623 & 13.7 & 0.599 & 0.594 \\
        TaylorSeer~\citep{taylorseer}      & \textcolor{blue}{\textbf{T}} & 7.00 & 3.59$\times$ & 431.74 & 0.660 & 0.646 & 9.43 & 0.514 & 0.446 \\
        Bottleneck~\citep{bottleneck}      & \textcolor{red}{\textbf{S}}  & 3.78 & 6.64$\times$ & 431.52 & 0.792 & 0.631 & 8.71 & 0.605 & \underline{0.672} \\
        RALU~\citep{ralu}                  & \textcolor{red}{\textbf{S}}  & 3.75 & 6.69$\times$ & \underline{426.01} & 0.999 & 0.681 & 6.87 & \textbf{0.633} & \textbf{0.682} \\
        \textbf{Ours ($S=2$)}     & \textcolor{red}{\textbf{S}}       & \underline{3.70} & \underline{6.78}$\times$ & 427.03 & \textbf{1.039} & \underline{0.689} & \underline{6.78} & 0.620 & 0.667 \\
        \textbf{Ours ($S=3$)}     & \textcolor{red}{\textbf{S}}       & \textbf{3.54} & \textbf{7.09}$\times$ & \textbf{406.24} & \underline{1.015} & \textbf{0.694} & \textbf{5.99} & \underline{0.627} & 0.637 \\
        \bottomrule
    \end{tabular}
    \end{scriptsize}}
\end{table*}

\subsection{Spectral-transformation-based fine-tuning for latent- and pixel-space image generation}
\label{subsec:image_pixel}
\paragraph{Additional latent-space fine-tuning results.} Since FLUX.1-dev is a guidance-distilled model and its base model is not open-sourced, we additionally conduct latent-space image generation fine-tuning experiments on Z-Image~\citep{zimage}, which employs the same VAE architecture as FLUX.1-dev and hence the same $\delta$-optimal resolution schedules. We adopt Low-Rank Adaptation (LoRA) of rank 32 and train for 2,000 steps following the spectral-transformation-based fine-tuning procedure in Sec.~\ref{sec:method:lora}. For fair comparison with the full-resolution and training-free acceleration baselines, we fine-tune on synthetic images generated by full-resolution Z-Image with 50 denoising steps using a set of 5,000 MS-COCO~\citep{lin2014microsoft} prompts that is disjoint from the 5,000 MS-COCO validation prompts used for evaluation. Training is conducted on the Stanford Marlowe cluster~\citep{marlowe2025} using 8 NVIDIA H100 GPUs, each with 80~GB of VRAM, at an effective batch size of 8 and takes approximately 2 hours. We keep the optimizer and all other LoRA fine-tuning hyperparameters at the default settings of the \texttt{DiffSynth Studio} library, including a constant learning rate of 1e-4. Fine-tuning compute scales linearly with the number of sampled training images, training iterations, and batch size. In our experiments we fix the synthetic fine-tuning set to 5,000 prompts and train LoRA adapters for 2,000 iterations, so increasing the fine-tuning dataset size would only affect compute through additional sampled training examples and does not change the inference-time cost of our method.

All latent-image generation experiments are conducted at a native resolution of $1024^2$ pixels, which is the default configuration for both FLUX.1-dev and Z-Image. The progressive-resolution settings for the $S = 2$ and $S = 3$ cases are the same as in the FLUX.1-dev training-free acceleration experiments in Sec.~\ref{subsec:image_latent}.

Further model fine-tuning allows a more aggressive error tolerance $\delta$, leading to an even higher speedup. For fine-tuning experiments, we additionally evaluate $\delta=0.05$, which is larger than the default $\delta=0.01$. As in Sec.~\ref{subsec:image_latent}, we organize the comparison into three wall-clock speedup tiers, all calculated relative to the 50-step Z-Image baseline on a single H100 GPU (we adopt H100 wall-clock measurements for Z-Image and all subsequent pixel-space image generation and video generation experiments):

\begin{itemize}
    \item \textbf{Small speedup}: the 50-step native full-resolution Z-Image baseline, compared against our training-free acceleration and LoRA fine-tuned variants at $S=2$ and $S=3$ ($\delta=0.01$, 50 steps), as well as our LoRA fine-tuned variant with the configuration ($\delta=0.05$, $S=2$, 50 steps), yielding $1.65\times$ to $2.01\times$ speedup;
    \item \textbf{Medium speedup}: the 10-step native full-resolution Z-Image baseline, used as an iso-compute reference for our training-free acceleration variants at $S=2$ and $S=3$ and our LoRA fine-tuned variant at $S=2$ with reduced denoising steps, yielding approximately $5\times$ speedup; and
    \item \textbf{Large speedup}: the 7-step native full-resolution Z-Image baseline, used as an iso-compute reference for our training-free acceleration variants at $S=2$ and $S=3$ and our LoRA fine-tuned variant at $S=2$ with further reduced denoising steps, yielding approximately $7\times$ to $8\times$ speedup.
\end{itemize}
For quantitative evaluation, we use the same metric evaluation settings as Sec.~\ref{subsec:image_latent}, following the RALU evaluation protocol, and report the numbers in Table~\ref{tab:zimage-supp}. We show extensive qualitative comparisons in Figures~\ref{fig:zimage-1}--\ref{fig:zimage-3}. Across all three speedup tiers, our training-free acceleration variants outperform the reduced-step Z-Image baselines, and our spectral-transformation-based LoRA fine-tuning further closes the gap to the 50-step full-resolution baseline while attaining up to $7.81\times$ wall-clock speedup. In particular, our LoRA fine-tuned model at $S=3$ outperforms the training-free acceleration variant at $S=2$ while improving efficiency, and our LoRA fine-tuned model at $S=2$ either with reduced denoising steps or at a larger $\delta=0.05$ attains image quality comparable to the training-free acceleration $S=2$, 50-step setting, providing additional efficiency at no quality cost.

We further evaluate RALU~\citep{ralu} on Z-Image at matched wall-clock speedups. RALU does not support end-to-end fine-tuning, as it contains non-differentiable operations, including Canny edge detection, so we evaluate it only in the training-free setting.

\begin{table*}[t!]
    \centering
    \caption{\textbf{Quantitative comparisons on latent-space image generation with fine-tuning.} We evaluate Z-Image~\citep{zimage} at $1024^2$ resolution and group methods into wall-clock speedup tiers. Speedup is calculated relative to the 50-step Z-Image baseline wall-clock runtime on a single H100 GPU. \textcolor{blue}{\textbf{T}} and \textcolor{red}{\textbf{S}} denote temporal and spatial acceleration methods, respectively. All metrics are evaluated using the same protocol as Table~\ref{tab:flux-main}.}
    \label{tab:zimage-supp}
    \resizebox{\textwidth}{!}{
    \begin{scriptsize}
    \begin{tabular}{cccccccccc}
         \toprule
        \multirow{2}{*}{Method} & \multirow{2}{*}{Accel.} & \multirow{2}{*}{Latency (s) $\downarrow$} & \multirow{2}{*}{Speedup ($\times$) $\uparrow$} & \multirow{2}{*}{TFLOPs $\downarrow$} & Overall & \multicolumn{2}{c}{Image quality} & \multicolumn{2}{c}{Prompt alignment} \\
        \cmidrule(lr){6-10}
        & & & & & ImageReward $\uparrow$ & CLIP-IQA $\uparrow$ & NIQE $\downarrow$ & T2I-Comp. $\uparrow$ & GenEval $\uparrow$  \\
        \midrule
        Z-Image (50 steps) & - & 21.25 & 1.00$\times$ & 4941.23 & 0.965 & 0.700 & 5.41 & 0.731 & 0.745  \\
        \cmidrule(lr){1-10}
        Z-Image (32 steps) & \textcolor{blue}{\textbf{T}} & 13.59 & 1.56$\times$ & 3166.62 & \underline{0.957} & \underline{0.697} & \textbf{5.44} & 0.686 & 0.725  \\
        RALU~\citep{ralu} & \textcolor{red}{\textbf{S}} & 13.36 & 1.59$\times$ & 3102.62 & 0.892 & 0.680 & 6.09 & 0.648 & 0.720 \\
        \textbf{Ours (TF, $\delta=0.01$, $S=2$)}     & \textcolor{red}{\textbf{S}}       & 12.90 & 1.65$\times$ & 3132.03 & 0.904 & 0.688 & 5.87 & 0.658 & \underline{0.730} \\
        \textbf{Ours (TF, $\delta=0.01$, $S=3$)}     & \textcolor{red}{\textbf{S}}       & \underline{12.19} & \underline{1.74$\times$} & \underline{2871.09} & 0.875 & 0.690 & \underline{5.59} & 0.650 & 0.682 \\
        \textbf{Ours (LoRA, $\delta=0.01$, $S=2$)}     & \textcolor{red}{\textbf{S}}       & 12.87 & 1.65$\times$ & 3132.03 & \textbf{0.982} & \textbf{0.699} & 5.72 & \textbf{0.725} & \textbf{0.731} \\
        \textbf{Ours (LoRA, $\delta=0.01$, $S=3$)}     & \textcolor{red}{\textbf{S}}       & 12.23 & 1.74$\times$ & \underline{2871.09} & 0.954 & \underline{0.697} & 5.75 & \underline{0.717} & 0.728 \\
        \textbf{Ours (LoRA, $\delta=0.05$, $S=2$)}     & \textcolor{red}{\textbf{S}}       & \textbf{10.59} & \textbf{2.01$\times$} & \textbf{2436.19} & 0.919 & 0.684 & 6.12 & 0.661 & 0.718 \\
        \cmidrule(lr){1-10}
        Z-Image (10 steps) & \textcolor{blue}{\textbf{T}} & 4.26 & 4.99$\times$ & 997.68 & 0.851 & 0.659 & 5.95 & \underline{0.678} & \underline{0.705} \\
        \textbf{Ours (TF, $\delta=0.01$, $S=2$)}     & \textcolor{red}{\textbf{S}}       & \underline{4.22} & \underline{5.04$\times$} & \underline{962.95} & \underline{0.860} & \underline{0.668} & 6.17 & 0.677 & 0.693 \\
        \textbf{Ours (TF, $\delta=0.01$, $S=3$)}     & \textcolor{red}{\textbf{S}}       & \textbf{4.01} & \textbf{5.30$\times$} & \textbf{875.97} & 0.827 & 0.662 & \underline{5.78} & 0.676 & 0.655 \\
        \textbf{Ours (LoRA, $\delta=0.01$, $S=2$)}     & \textcolor{red}{\textbf{S}}       &  4.24 & 5.01$\times$ & \underline{962.95} & \textbf{0.923} & \textbf{0.683} & \textbf{5.58} & \textbf{0.706} & \textbf{0.738} \\
        \cmidrule(lr){1-10}
        Z-Image (7 steps) & \textcolor{blue}{\textbf{T}} & 3.56 & 5.97$\times$ & 701.92 & 0.763 & 0.631 & \textbf{6.14} & 0.648 & 0.667 \\
        RALU~\citep{ralu} & \textcolor{red}{\textbf{S}} & 2.75 & 7.73$\times$ & \underline{673.99} & 0.572 & 0.595 & 6.74 & 0.636 & 0.596 \\
        \textbf{Ours (TF, $\delta=0.01$, $S=2$)}     & \textcolor{red}{\textbf{S}}       & 2.75 & 7.73$\times$ & 678.77 & \underline{0.804} & \underline{0.641} & 6.46 & \underline{0.665} & \underline{0.680} \\
        \textbf{Ours (TF, $\delta=0.01$, $S=3$)}     & \textcolor{red}{\textbf{S}}       & \textbf{2.63} & \textbf{8.08$\times$} & \textbf{609.18} & 0.759 & 0.630 & \underline{6.16} & 0.658 & 0.645 \\
        \textbf{Ours (LoRA, $\delta=0.01$, $S=2$)}     & \textcolor{red}{\textbf{S}}       & \underline{2.72} & \underline{7.81$\times$} & 678.77 & \textbf{0.918} & \textbf{0.658} & 6.57 & \textbf{0.683} & \textbf{0.744} \\
        \bottomrule
    \end{tabular}
    \end{scriptsize}
    }
\end{table*}


\paragraph{Additional pixel-space fine-tuning results.} For pixel-space image generation, we evaluate our method on PixelGen-XXL/16 T2I~\citep{pixelgen} at its native resolution of $512^2$ pixels. We use the same fine-tuning data construction and LoRA training procedure as the Z-Image latent-space experiments above, with synthetic fine-tuning images generated by full-resolution PixelGen with 25 denoising steps from a set of 5,000 MS-COCO prompts disjoint from the 5,000 MS-COCO validation prompts used for evaluation. Since PixelGen is pretrained with $x_0$ prediction rather than flow matching, we convert the spectral-transform velocity targets from the fine-tuning procedure in Sec.~\ref{sec:method:lora} to $x_0$ targets and supervise the LoRA adapter with the converted $x_0$.

We use $S=2$ with $\delta=0.01$ and, following Sec.~\ref{subsec:image_latent}, organize the comparison into two wall-clock speedup tiers, all calculated relative to the 25-step PixelGen baseline on a single H100 GPU:

\begin{itemize}
    \item \textbf{Small speedup}: the 25-step native full-resolution PixelGen baseline, used as the reference for the reduced 13-step PixelGen variant and our training-free acceleration and LoRA fine-tuned variants at $S=2$ (25 steps), yielding approximately $1.5\times$ to $1.9\times$ speedup;
    \item \textbf{Medium speedup}: the 8-step native full-resolution PixelGen variant, used as an iso-compute reference for our training-free acceleration and LoRA fine-tuned variants at $S=2$ with reduced denoising steps, yielding approximately $2\times$ to $3\times$ speedup.
\end{itemize}

For quantitative evaluation, we follow the same protocol used in latent-space image generation and report the numbers in Table~\ref{tab:pixelgen-main}. We show extensive qualitative comparisons in Figures~\ref{fig:pixelgen-1}--\ref{fig:pixelgen-3}. Because PixelGen does not natively support a high-fidelity $256^2$ resolution generation, our training-free acceleration variant falls short of the reduced-step PixelGen baselines on overall image quality across both speedup tiers. Our spectral-transformation-based LoRA fine-tuning closes this gap: at the same FLOPs as the training-free acceleration variant, the LoRA fine-tuned model matches or surpasses the reduced-step PixelGen baseline on most image-quality and prompt-alignment metrics, including ImageReward, CLIP-IQA, and T2I-CompBench. The qualitative comparisons in Figures~\ref{fig:pixelgen-1}--\ref{fig:pixelgen-3} corroborate this trend, showing that our LoRA fine-tuned model's outputs match the visual quality of the 25-step native-resolution PixelGen baseline at substantially higher throughput.

We likewise evaluate RALU~\citep{ralu} on PixelGen in the training-free setting at a matched wall-clock speedup.

\begin{table*}[t!]
    \centering
    \caption{\textbf{Quantitative comparisons on pixel-space image generation.} The number in parentheses next to PixelGen~\citep{pixelgen} indicates the total number of inference steps. $\uparrow / \downarrow$ denotes that a higher / lower metric is favorable. Speedup is calculated relative to the 25-step PixelGen baseline wall-clock runtime on a single H100 GPU. \textcolor{blue}{\textbf{T}} and \textcolor{red}{\textbf{S}} denote temporal and spatial acceleration, respectively.}
    \label{tab:pixelgen-main}
    \resizebox{\textwidth}{!}{
    \begin{scriptsize}
    \begin{tabular}{cccccccccc}
         \toprule
        \multirow{2}{*}{Method} & \multirow{2}{*}{Accel.} & \multirow{2}{*}{Latency (s) $\downarrow$} & \multirow{2}{*}{Speedup ($\times$) $\uparrow$} & \multirow{2}{*}{TFLOPs $\downarrow$} & Overall & \multicolumn{2}{c}{Image quality} & \multicolumn{2}{c}{Prompt alignment} \\
        \cmidrule(lr){6-10}
        & & & & & ImageReward $\uparrow$ & CLIP-IQA $\uparrow$ & NIQE $\downarrow$ & T2I-Comp. $\uparrow$ & GenEval $\uparrow$  \\
        \midrule
        PixelGen (25 steps) & - & 0.48 & 1.00$\times$ & 65.36 & 0.921 & 0.734 & 5.95 & 0.574 & 0.794 \\
        \cmidrule(lr){1-10}
        PixelGen (13 steps) & \textcolor{blue}{\textbf{T}} & \textbf{0.25} & \textbf{1.92$\times$} & \underline{34.16} & \underline{0.886} & \underline{0.726} & \underline{5.89} & \underline{0.571} & \underline{0.781} \\
        RALU~\citep{ralu} & \textcolor{red}{\textbf{S}} & \underline{0.30} & \underline{1.61$\times$} & 34.46 & 0.801 & 0.713 & 6.83 & 0.551 & 0.770 \\
        \textbf{Ours (TF, $\delta=0.01$, $S = 2$)}     & \textcolor{red}{\textbf{S}}       & \underline{0.30} & 1.60$\times$ & \textbf{33.72} & 0.799 & 0.718 & 6.10 & 0.568 & \textbf{0.782} \\
        \textbf{Ours (LoRA, $\delta=0.01$, $S = 2$)}   & \textcolor{red}{\textbf{S}}       & 0.31 & 1.55$\times$ & \textbf{33.72} & \textbf{0.913} & \textbf{0.728} & \textbf{5.87} & \textbf{0.580} & 0.776 \\
        \cmidrule(lr){1-10}
        PixelGen (8 steps) & \textcolor{blue}{\textbf{T}} & \textbf{0.16} & \textbf{3.00$\times$} & \underline{21.16} & \underline{0.858} & \underline{0.715} & \textbf{5.79} & \underline{0.572} & 0.756 \\
        \textbf{Ours (TF, $\delta=0.01$, $S = 2$)}     & \textcolor{red}{\textbf{S}}       & \underline{0.21} & \underline{2.29$\times$} & \textbf{20.75} & 0.779 & 0.713 & 6.02 & 0.567 & \underline{0.765} \\
        \textbf{Ours (LoRA, $\delta=0.01$, $S = 2$)}   & \textcolor{red}{\textbf{S}}       & 0.22 & 2.18$\times$ & \textbf{20.75} & \textbf{0.908} & \textbf{0.724} & \underline{5.85} & \textbf{0.577} & \textbf{0.770} \\
        \bottomrule
    \end{tabular}
    \end{scriptsize}
    }
\end{table*}

\clearpage
\begin{figure}
    \centering
    \includegraphics[width=0.9\linewidth]{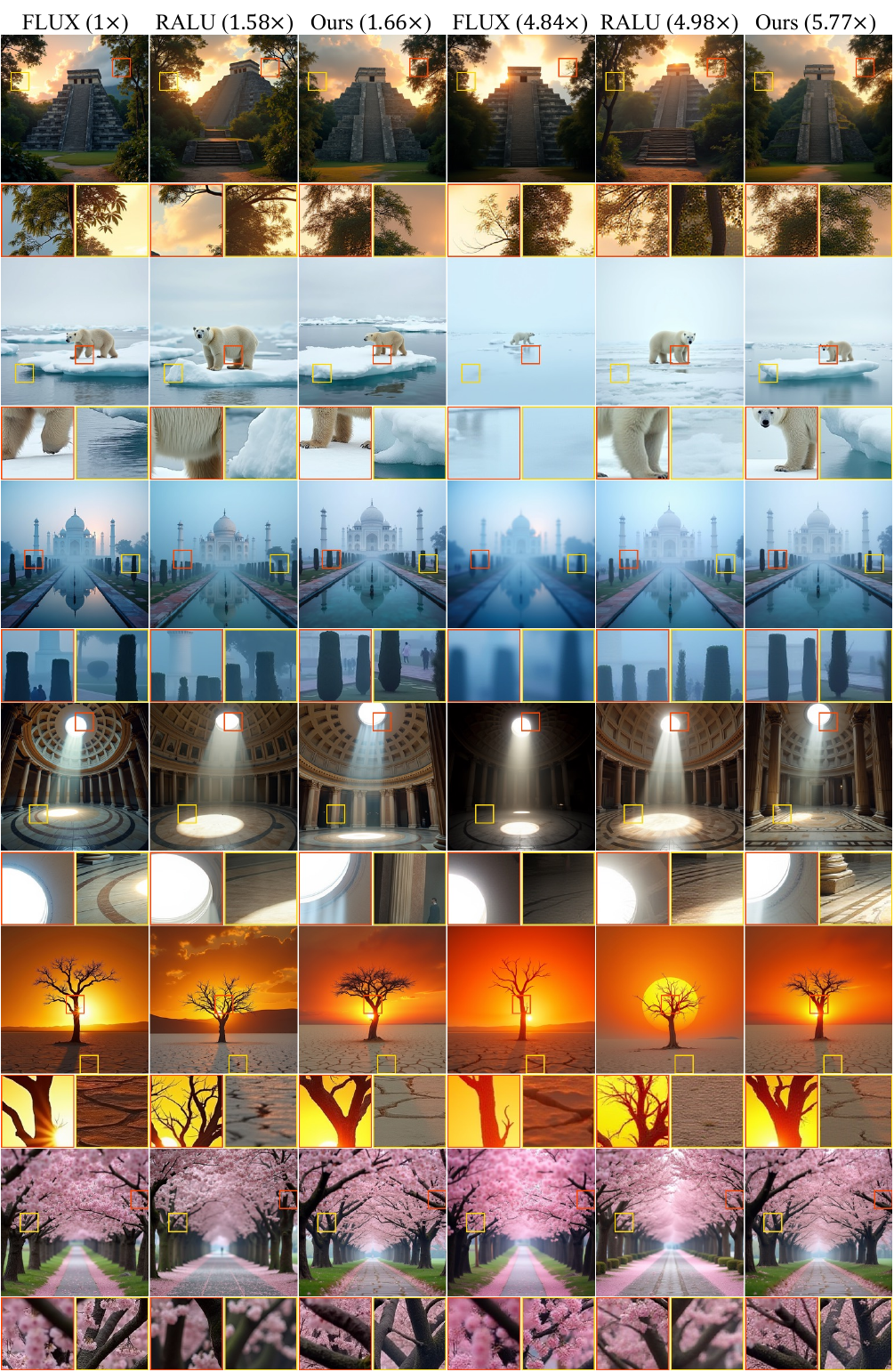}
    \caption{\textbf{Qualitative comparisons on latent-space image generation.} We compare our method against default-step generation, reduced-step native-resolution generation on FLUX.1-dev \citep{black2024flux}, and RALU~\citep{ralu}, a state-of-the-art acceleration baseline matched to similar speedups. Our method outperforms both baselines. FLUX.1-dev with reduced steps degrades image quality and exhibits over-saturation artifacts, while RALU introduces noticeable noise artifacts.}
    \label{fig:flux-1}
\end{figure}

\clearpage
\begin{figure}
    \centering
    \includegraphics[width=0.9\linewidth]{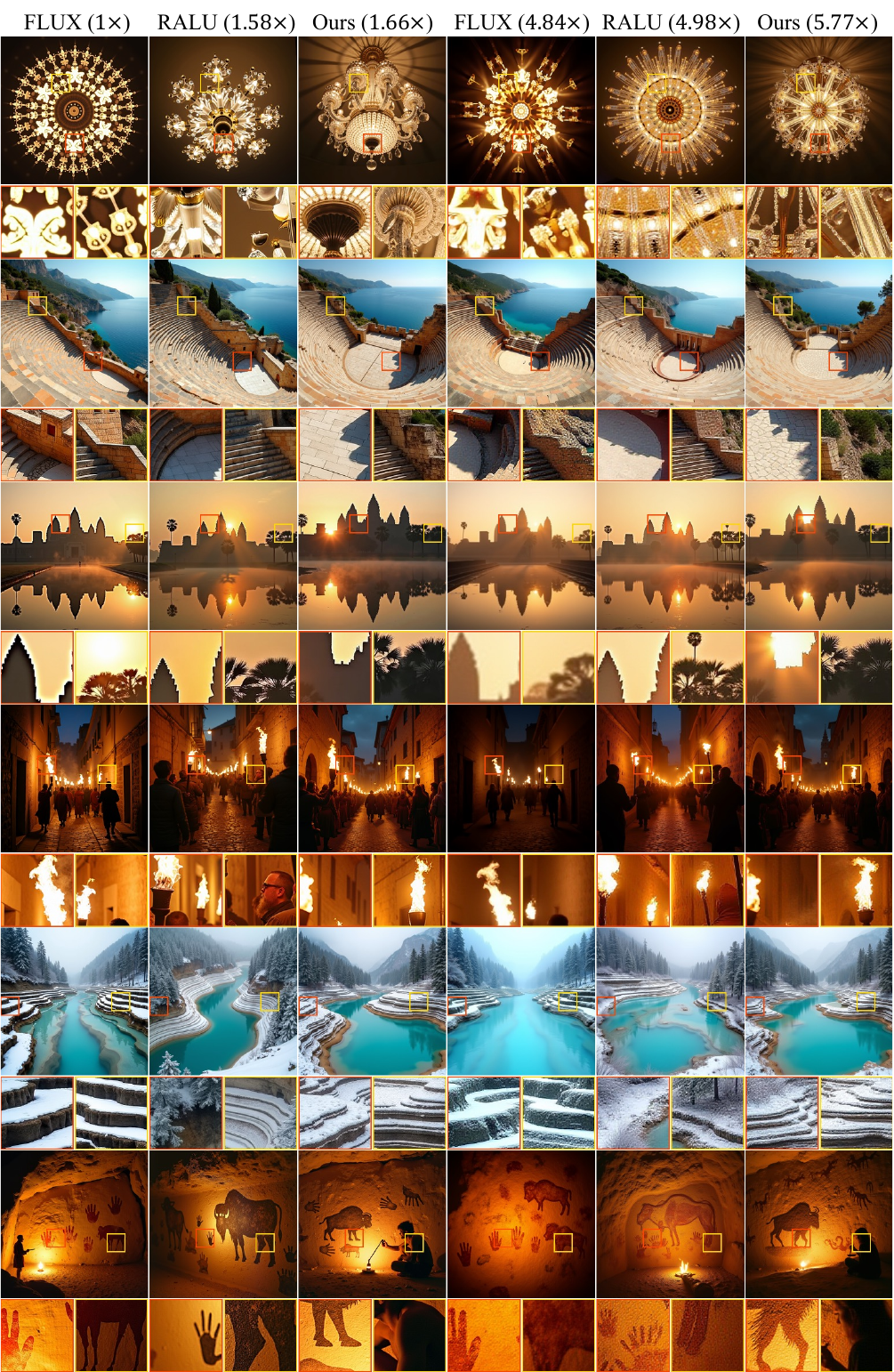}
    \caption{\textbf{Qualitative comparisons on latent-space image generation.} We compare our method against default-step generation, reduced-step native-resolution generation on FLUX.1-dev \citep{black2024flux}, and RALU~\citep{ralu}, a state-of-the-art acceleration baseline matched to similar speedups. Our method outperforms both baselines. FLUX.1-dev with reduced steps degrades image quality and exhibits over-saturation artifacts, while RALU introduces noticeable noise artifacts.}
    \label{fig:flux-2}
\end{figure}

\clearpage
\begin{figure}
    \centering
    \includegraphics[width=0.9\linewidth]{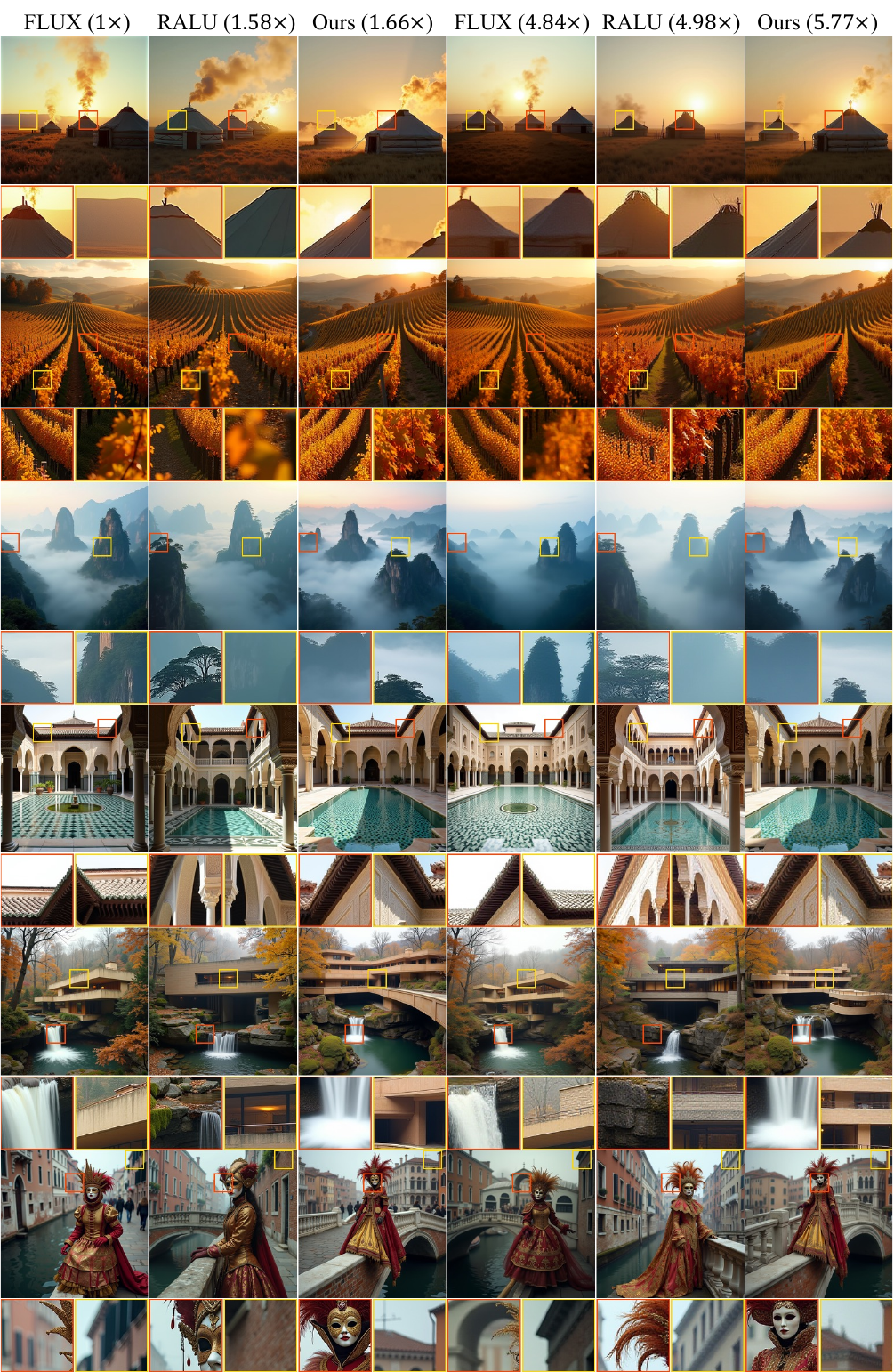}
    \caption{\textbf{Qualitative comparisons on latent-space image generation.} We compare our method against default-step generation, reduced-step native-resolution generation on FLUX.1-dev \citep{black2024flux}, and RALU~\citep{ralu}, a state-of-the-art acceleration baseline matched to similar speedups. Our method outperforms both baselines. FLUX.1-dev with reduced steps degrades image quality and exhibits over-saturation artifacts, while RALU introduces noticeable noise artifacts.}
    \label{fig:flux-3}
\end{figure}

\clearpage
\begin{figure}
     \centering
     \includegraphics[width=0.9\linewidth]{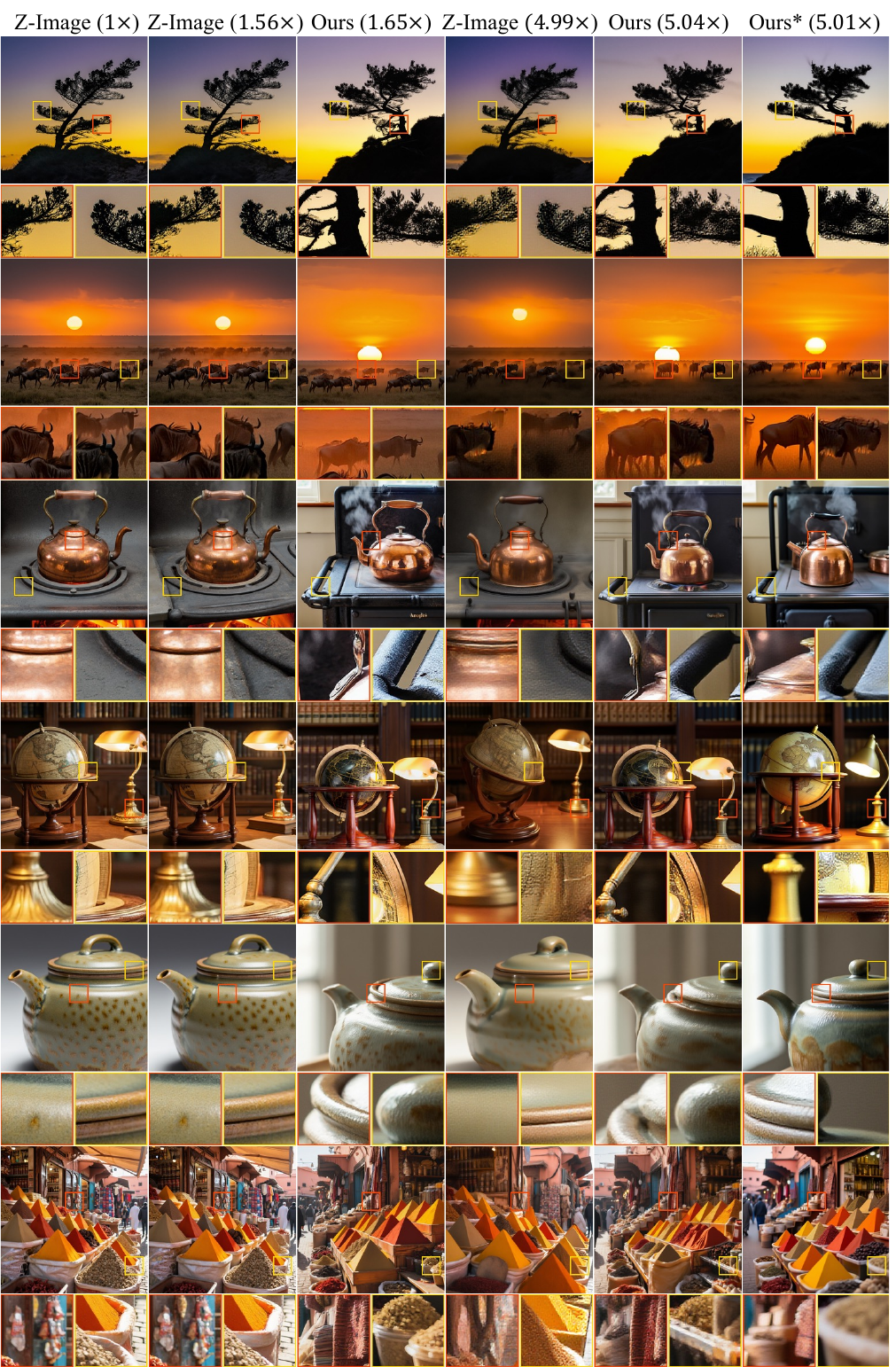}
     \caption{\textbf{Qualitative comparisons on latent-space image generation (fine-tuned).} We compare our method against default-step generation and reduced-step native-resolution generation on Z-Image~\citep{zimage}, matched to similar speedups. Our fine-tuned model (Ours$^*$) achieves even higher image quality compared to our training-free acceleration variant and outperforms the reduced-step baseline.}
     \label{fig:zimage-1}
\end{figure}

\clearpage
\begin{figure}
     \centering
     \includegraphics[width=0.9\linewidth]{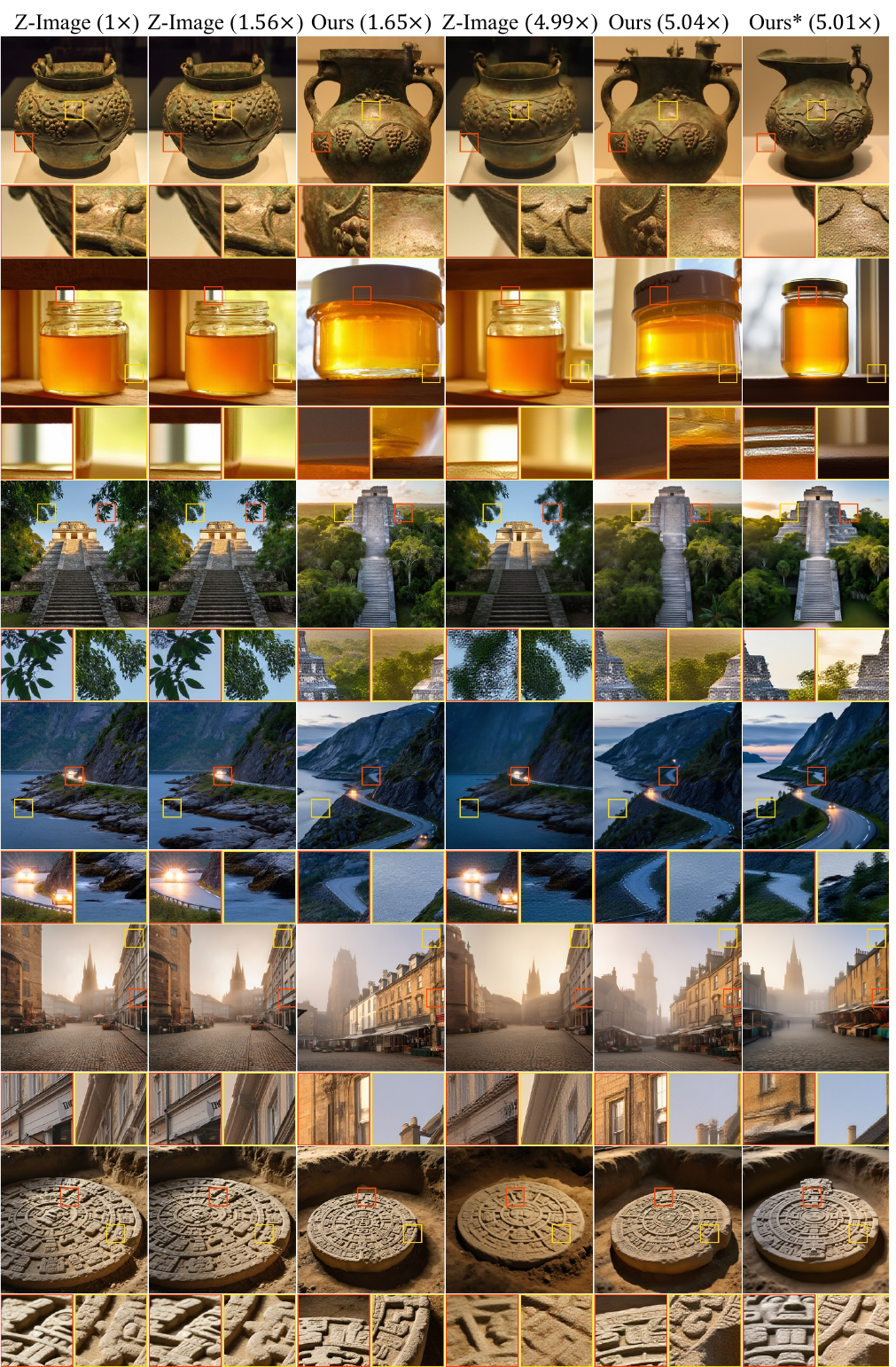}
     \caption{\textbf{Qualitative comparisons on latent-space image generation (fine-tuned).} We compare our method against default-step generation and reduced-step native-resolution generation on Z-Image~\citep{zimage}, matched to similar speedups. Our fine-tuned model (Ours$^*$) achieves even higher image quality compared to our training-free acceleration variant and outperforms the reduced-step baseline.}
     \label{fig:zimage-2}
\end{figure}

\clearpage
\begin{figure}
     \centering
     \includegraphics[width=0.9\linewidth]{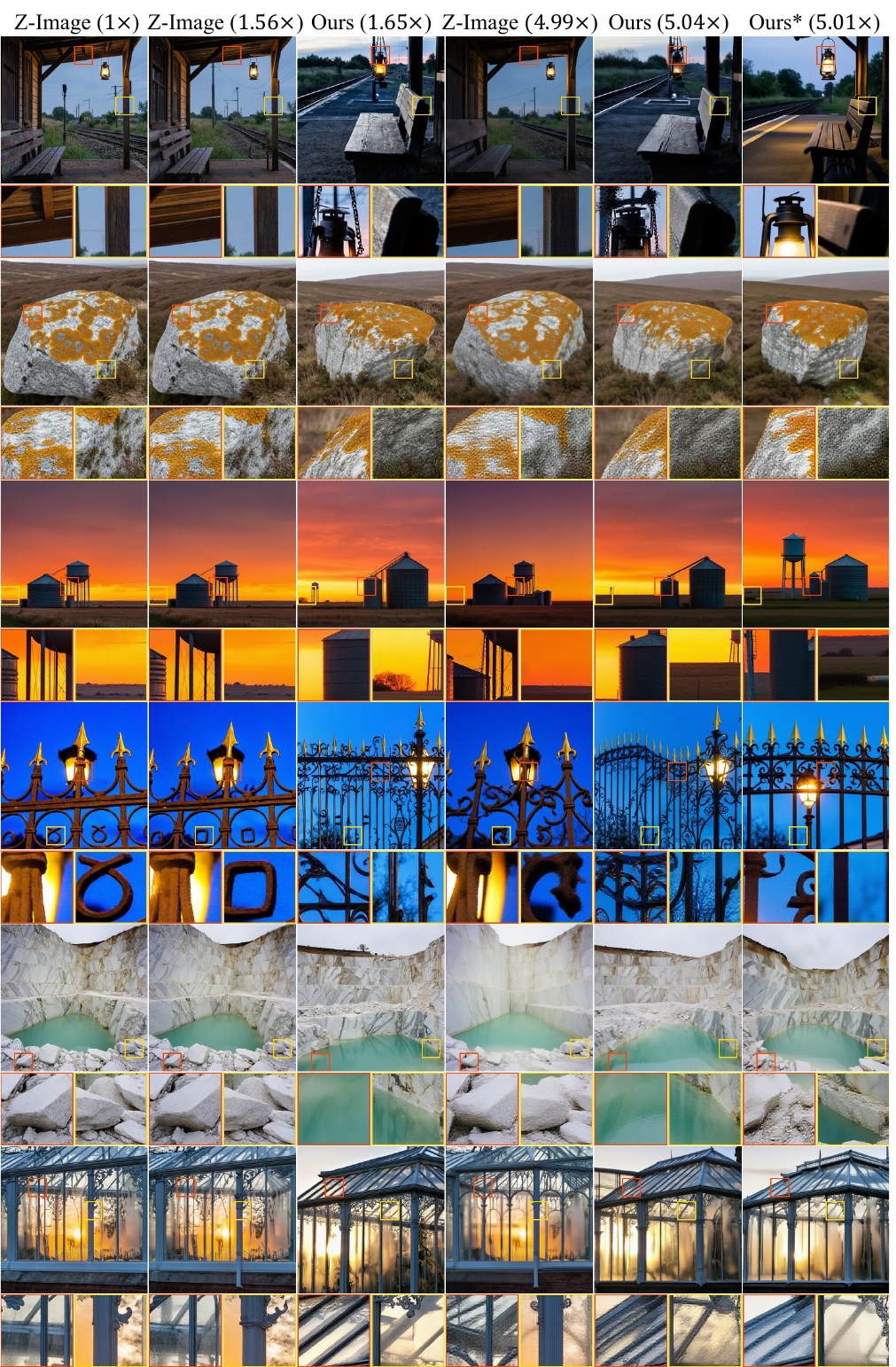}
     \caption{\textbf{Qualitative comparisons on latent-space image generation (fine-tuned).} We compare our method against default-step generation and reduced-step native-resolution generation on Z-Image~\citep{zimage}, matched to similar speedups. Our fine-tuned model (Ours$^*$) achieves even higher image quality compared to our training-free acceleration variant and outperforms the reduced-step baseline.}
     \label{fig:zimage-3}
\end{figure}

\clearpage
\begin{figure}
    \centering
    \includegraphics[width=0.6\linewidth]{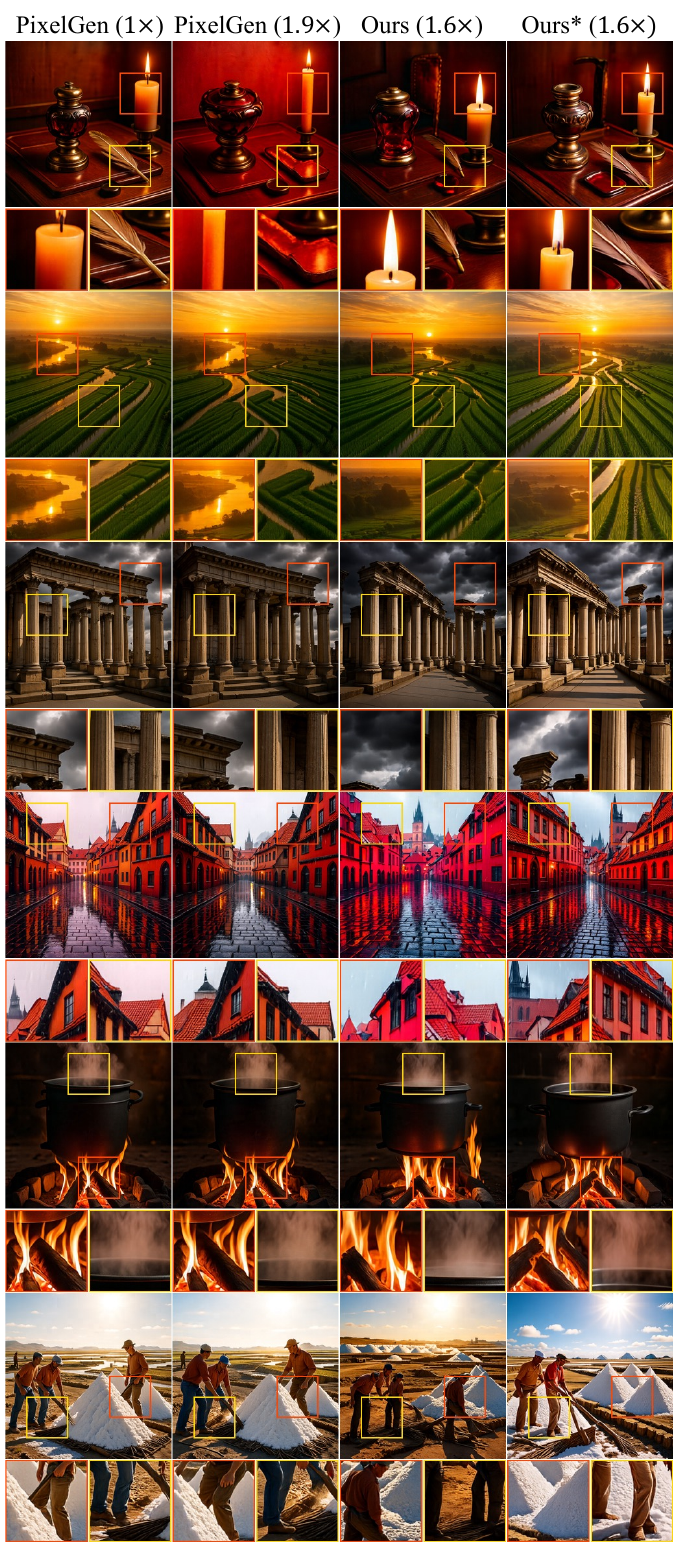}
    \caption{\textbf{Qualitative comparisons on pixel-space image generation.} We compare our method against default-step generation and reduced-step native-resolution generation on PixelGen~\citep{pixelgen}, matched to comparable speedups. An asterisk (Ours$^*$) marks the fine-tuned model. Our method achieves similar quality to full-resolution generation while attaining higher speedups.}
    \label{fig:pixelgen-1}
\end{figure}

\clearpage
\begin{figure}
    \centering
    \includegraphics[width=0.6\linewidth]{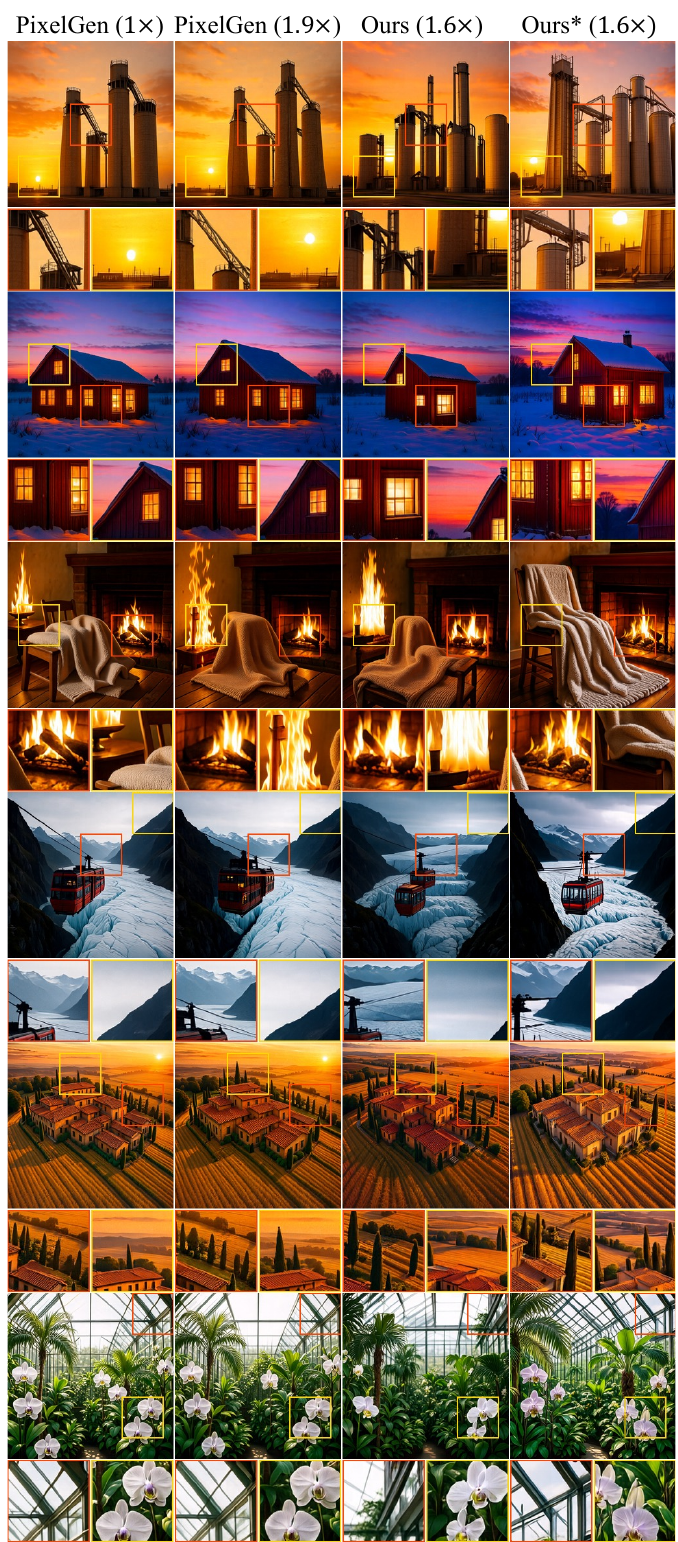}
    \caption{\textbf{Qualitative comparisons on pixel-space image generation.} We compare our method against default-step generation and reduced-step native-resolution generation on PixelGen~\citep{pixelgen}, matched to comparable speedups. An asterisk (Ours$^*$) marks the fine-tuned model. Our method achieves similar quality to full-resolution generation while attaining higher speedups.}
    \label{fig:pixelgen-2}
\end{figure}

\clearpage
\begin{figure}
    \centering
    \includegraphics[width=0.6\linewidth]{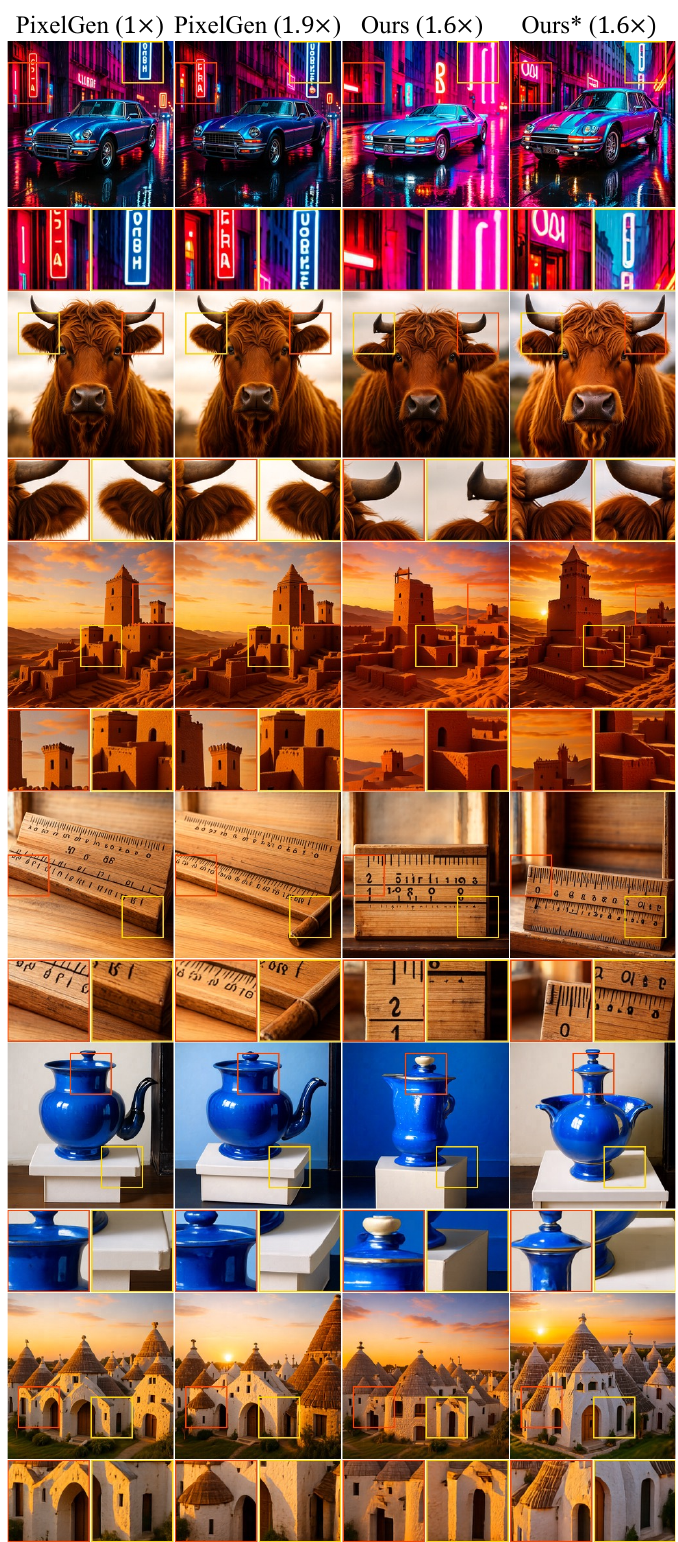}
    \caption{\textbf{Qualitative comparisons on pixel-space image generation.} We compare our method against default-step generation and reduced-step native-resolution generation on PixelGen~\citep{pixelgen}, matched to comparable speedups. An asterisk (Ours$^*$) marks the fine-tuned model. Our method achieves similar quality to full-resolution generation while attaining higher speedups.}
    \label{fig:pixelgen-3}
\end{figure}

\clearpage
\section{Additional video generation results}
\label{sec:video_generation}

\paragraph{Training-free acceleration of latent-space video generation.} For training-free inference acceleration of video generation, we use the Stanford Marlowe cluster~\citep{marlowe2025} with NVIDIA H100 GPUs, each with 80~GB of VRAM. For quantitative evaluation, we generate 200 videos using 200 test prompts and report the standard VBench~\citep{vbench} scores across six evaluation dimensions, including Subject Consistency, Background Consistency, Motion Smoothness, Dynamic Degree, Aesthetic Quality, and Image Quality, in Table~\ref{tab:wan_vbench_720p}. We do not fine-tune the video model due to its compute and memory cost. All final videos are in 720p resolution. Our approach uses two resolution stages ($S=2$) at 480p and 720p, or three resolution stages ($S=3$) at 240p, 480p, and 720p, while keeping the default number of 50 denoising steps of WAN 2.1~\citep{wan} in Table~\ref{tab:wan_vbench_720p}. For additional video results and qualitative comparisons, please see the project website at~\url{https://howardxiao.ca/speed/}.

\clearpage
\section{Compatibility with other acceleration methods}
\label{app:compatibility}

Our approach reduces spatial computation within each denoising step, whereas feature caching and step distillation reduce redundancy across denoising steps. Below, we show that our framework composes with both.

\paragraph{Feature caching.} We reproduce MagCache~\citep{ma2025magcache} on FLUX.1-dev using its released 28-step setting, magnitude-ratio curve, and E024K5R01 hyperparameters (threshold $0.24$, $K=5$, retention $0.1$). We apply our framework with $S=2$ and $\delta=0.01$, keeping the same FLUX.1-dev power spectrum and resolution schedule used in Table~\ref{tab:flux-main}. The only integration change occurs at the resolution transition as the token count changes, and no recalibration of MagCache's magnitude-ratio curve is required. As shown in Table~\ref{tab:magcache}, across 5,000 MS-COCO validation prompts, the combination increases the end-to-end speedup from $3.49\times$ to $4.20\times$ while improving every quality metric relative to MagCache alone. We hypothesize that this quality improvement arises from reduced cache approximation error in the lower-resolution stage. At each skipped step, MagCache reuses a previously computed residual, which becomes less accurate as fine-grained content evolves between steps. Our progressive-resolution schedule keeps most of the skipped steps at low resolution before high-frequency details emerge, making MagCache more stable across timesteps.

\begin{table}[h]
  \centering
  \small
  \caption{\textbf{Combining our method with MagCache~\citep{ma2025magcache} on FLUX.1-dev.} Results are averaged over 5,000 MS-COCO validation prompts.}
  \label{tab:magcache}
  \begin{tabular}{lcccc}
    \toprule
    Method & Speedup ($\times$) $\uparrow$ & ImageReward $\uparrow$ & CLIP-IQA $\uparrow$ & NIQE $\downarrow$ \\
    \midrule
    FLUX.1-dev (28 steps) & 1.00$\times$ & \textbf{1.068} & \textbf{0.642} & \underline{6.78} \\
    + MagCache (E024K5R01) & \underline{3.49}$\times$ & 0.993 & 0.554 & 7.88 \\
    + MagCache + \textbf{Ours} ($S=2$) & \textbf{4.20}$\times$ & \underline{1.010} & \underline{0.632} & \textbf{6.73} \\
    \bottomrule
  \end{tabular}
\end{table}

\paragraph{Step-distilled models.} Step distillation reduces the number of denoising steps, while our framework reduces the computational cost per step. We apply our method to the four-step-distilled FLUX.1-schnell~\citep{flux1schnell} with $\delta=0.01$, reusing the FLUX.1-dev latent power spectrum estimate and the resulting $S=2$ resolution schedule, as both models share the same VAE backbone. As shown in Table~\ref{tab:schnell}, we obtain an additional $1.20\times$ wall-clock speedup over the distilled model while improving four of the five quality metrics and matching on the fifth (GenEval, 0.675 against 0.676). We hypothesize that the distilled model generates high-frequency components at early timesteps even though those components remain noise-dominated, and that deferring their computation until later in denoising contributes to the observed quality gains. ImageReward, CLIP-IQA, and NIQE below are also averaged across 5,000 MS-COCO validation prompts. 

\begin{table}[h]
  \centering
  \footnotesize
  \setlength{\tabcolsep}{4pt}
  \caption{\textbf{Applying our method to step-distilled FLUX.1-schnell~\citep{flux1schnell}.}}
  \label{tab:schnell}
  \begin{tabular}{lcccccc}
    \toprule
    Method & Speedup $\uparrow$ & ImageReward $\uparrow$ & CLIP-IQA $\uparrow$ & NIQE $\downarrow$ & T2I-Comp. $\uparrow$ & GenEval $\uparrow$ \\
    \midrule
    FLUX.1-schnell (4 steps) & 1.00$\times$ & 1.022 & 0.696 & 6.56 & 0.628 & \textbf{0.676} \\
    \textbf{Ours} ($S=2$) & \textbf{1.20}$\times$ & \textbf{1.121} & \textbf{0.707} & \textbf{6.16} & \textbf{0.651} & 0.675 \\
    \bottomrule
  \end{tabular}
\end{table}

\clearpage
\section{Extended ablation studies}
\label{sec:ablation}

We perform extended ablation studies on latent-space image generation. Specifically, we vary the error tolerance $\delta$, the number of resolution stages $S$, and the spectral transformation $T_\Phi$, and report wall-clock latency, speedup, and ImageReward in Table~\ref{tab:ablation-quan}, with qualitative comparisons in Figures~\ref{fig:ablation-T} to \ref{fig:ablation-s}.

\begin{table*}[t!]
    \centering
    \caption{\textbf{Ablation studies on the spectral transformation $T_\Phi$, error tolerance $\delta$, and number of resolution stages $S$.} Speedup is calculated relative to the 7-step FLUX.1-dev baseline wall-clock runtime on a single H100 GPU.}
    \label{tab:ablation-quan}
    \centering
    \begin{tiny}
    \begin{tabular}{cccccc}
        \toprule
        $T_\Phi$ & $\delta$ & $S$ & Latency (s) $\downarrow$ & Speedup ($\times$) $\uparrow$ & ImageReward $\uparrow$ \\
        \midrule
        FLUX (7 steps) & - & - & 2.08 & 1.00$\times$ & 0.920 \\
        \midrule
        DCT & 0.001 & 2 & \textbf{2.08} & \textbf{1.00$\times$} & 1.032 \\
        DWT & 0.001 & 2 & \underline{2.08} & \textbf{1.00$\times$} & \textbf{1.038} \\
        FFT & 0.001 & 2 & 2.08 & \textbf{1.00$\times$} & \underline{1.036} \\
        \midrule
        DCT & 0.001 & 3 & \textbf{2.05} & \textbf{1.01$\times$} & \underline{1.034} \\
        DWT & 0.001 & 3 & \underline{2.06} & \textbf{1.01$\times$} & \underline{1.034} \\
        FFT & 0.001 & 3 & 2.06 & \textbf{1.01$\times$} & \textbf{1.040} \\
        \midrule
        DCT & 0.001 & 4 & \textbf{2.05} & \textbf{1.02$\times$} & \textbf{1.008} \\
        DWT & 0.001 & 4 & \underline{2.05} & \underline{1.01$\times$} & \underline{1.002} \\
        FFT & 0.001 & 4 & 2.06 & \underline{1.01$\times$} & 0.992 \\
        \midrule
        DCT & 0.002 & 2 & \textbf{2.07} & \textbf{1.01$\times$} & \underline{1.032} \\
        DWT & 0.002 & 2 & \underline{2.08} & \underline{1.00$\times$} & 1.031 \\
        FFT & 0.002 & 2 & 2.08 & \underline{1.00$\times$} & \textbf{1.036} \\
        \midrule
        DCT & 0.002 & 3 & \textbf{2.05} & \textbf{1.02$\times$} & \underline{1.034} \\
        DWT & 0.002 & 3 & \underline{2.06} & \underline{1.01$\times$} & \underline{1.034} \\
        FFT & 0.002 & 3 & 2.06 & \underline{1.01$\times$} & \textbf{1.040} \\
        \midrule
        DCT & 0.002 & 4 & \textbf{2.05} & \textbf{1.02$\times$} & \textbf{1.008} \\
        DWT & 0.002 & 4 & \underline{2.05} & \underline{1.01$\times$} & \underline{1.002} \\
        FFT & 0.002 & 4 & 2.06 & \underline{1.01$\times$} & 0.992 \\
        \midrule
        DCT & 0.005 & 2 & \textbf{1.91} & \textbf{1.09$\times$} & \underline{1.019} \\
        DWT & 0.005 & 2 & \underline{1.92} & \textbf{1.09$\times$} & \textbf{1.026} \\
        FFT & 0.005 & 2 & 1.92 & \underline{1.08$\times$} & 1.014 \\
        \midrule
        DCT & 0.005 & 3 & \textbf{1.87} & \textbf{1.12$\times$} & \underline{1.009} \\
        DWT & 0.005 & 3 & \underline{1.88} & \underline{1.11$\times$} & \underline{1.009} \\
        FFT & 0.005 & 3 & 1.88 & \underline{1.11$\times$} & \textbf{1.028} \\
        \midrule
        DCT & 0.005 & 4 & \textbf{1.87} & \textbf{1.12$\times$} & \underline{1.005} \\
        DWT & 0.005 & 4 & \underline{1.87} & \underline{1.11$\times$} & \textbf{1.008} \\
        FFT & 0.005 & 4 & 1.87 & \underline{1.11$\times$} & 0.994 \\
        \midrule
        DCT & 0.01 & 2 & \textbf{1.75} & \textbf{1.19$\times$} & \textbf{1.039} \\
        DWT & 0.01 & 2 & \underline{1.76} & \underline{1.18$\times$} & \underline{1.025} \\
        FFT & 0.01 & 2 & 1.76 & \underline{1.18$\times$} & 1.016 \\
        \midrule
        DCT & 0.01 & 3 & \textbf{1.71} & \textbf{1.22$\times$} & 1.015 \\
        DWT & 0.01 & 3 & \underline{1.72} & \underline{1.21$\times$} & \underline{1.018} \\
        FFT & 0.01 & 3 & 1.72 & \underline{1.21$\times$} & \textbf{1.030} \\
        \midrule
        DCT & 0.01 & 4 & \textbf{1.71} & \textbf{1.22$\times$} & \textbf{1.014} \\
        DWT & 0.01 & 4 & \underline{1.71} & \textbf{1.22$\times$} & \underline{1.006} \\
        FFT & 0.01 & 4 & 1.72 & \underline{1.21$\times$} & 0.996 \\
        \midrule
        DCT & 0.02 & 2 & \textbf{1.60} & \textbf{1.30$\times$} & \underline{1.009} \\
        DWT & 0.02 & 2 & \underline{1.61} & \textbf{1.30$\times$} & \textbf{1.010} \\
        FFT & 0.02 & 2 & \underline{1.61} & \textbf{1.30$\times$} & 1.002 \\
        \midrule
        DCT & 0.02 & 3 & \textbf{1.53} & \textbf{1.36$\times$} & \underline{0.967} \\
        DWT & 0.02 & 3 & \underline{1.53} & \textbf{1.36$\times$} & 0.966 \\
        FFT & 0.02 & 3 & 1.53 & \textbf{1.36$\times$} & \textbf{0.980} \\
        \midrule
        DCT & 0.02 & 4 & \textbf{1.53} & \textbf{1.36$\times$} & \textbf{0.964} \\
        DWT & 0.02 & 4 & \underline{1.53} & \textbf{1.36$\times$} & \underline{0.955} \\
        FFT & 0.02 & 4 & 1.53 & \textbf{1.36$\times$} & 0.949 \\
        \midrule
        DCT & 0.05 & 2 & \textbf{1.44} & \textbf{1.44$\times$} & \underline{0.989} \\
        DWT & 0.05 & 2 & \underline{1.45} & \textbf{1.44$\times$} & \textbf{0.991} \\
        FFT & 0.05 & 2 & 1.45 & \underline{1.43$\times$} & 0.972 \\
        \midrule
        DCT & 0.05 & 3 & \textbf{1.34} & \textbf{1.55$\times$} & \underline{0.936} \\
        DWT & 0.05 & 3 & \underline{1.34} & \textbf{1.55$\times$} & \textbf{0.937} \\
        FFT & 0.05 & 3 & 1.35 & \underline{1.54$\times$} & 0.933 \\
        \midrule
        DCT & 0.05 & 4 & \textbf{1.33} & \textbf{1.56$\times$} & \textbf{0.918} \\
        DWT & 0.05 & 4 & \underline{1.34} & \textbf{1.56$\times$} & \underline{0.902} \\
        FFT & 0.05 & 4 & 1.34 & \underline{1.55$\times$} & 0.895 \\
        \midrule
        DCT & 0.1 & 2 & \textbf{1.29} & \textbf{1.62$\times$} & \underline{0.966} \\
        DWT & 0.1 & 2 & \underline{1.29} & \textbf{1.62$\times$} & \textbf{0.977} \\
        FFT & 0.1 & 2 & 1.30 & \underline{1.61$\times$} & 0.944 \\
        \midrule
        DCT & 0.1 & 3 & \textbf{1.15} & \textbf{1.81$\times$} & \underline{0.880} \\
        DWT & 0.1 & 3 & 1.17 & 1.78$\times$ & \textbf{0.892} \\
        FFT & 0.1 & 3 & \underline{1.17} & \underline{1.79$\times$} & 0.877 \\
        \midrule
        DCT & 0.1 & 4 & \textbf{1.15} & \textbf{1.82$\times$} & \underline{0.864} \\
        DWT & 0.1 & 4 & \underline{1.15} & \underline{1.81$\times$} & \textbf{0.867} \\
        FFT & 0.1 & 4 & 1.16 & 1.80$\times$ & 0.841 \\
        \bottomrule
    \end{tabular}
    \end{tiny}
\end{table*}

\paragraph{Computational efficiency and scaling.} The inference cost of Spectral Progressive Diffusion is determined by the number of denoising steps and the token count processed at each resolution stage. For a fixed pretrained model and resolution schedule, the per-sample cost is independent of the dataset size, so generating or evaluating more samples scales linearly with the number of samples. The spectral transformation and noise-expansion operations add only lightweight per-transition overhead compared to the DiT forward passes. 

\paragraph{Spectral transformation $T_\Phi$.} In Fig. \ref{fig:ablation-T}, we observe that DCT, DWT, and FFT achieve similar high-fidelity performance. We default to DCT for its real-valued computation and native support of non-power-of-two resolution scaling. We exclude the cepstrum transform~\citep{bogert1963quefrency} because its logarithmic mapping is nonlinear and does not preserve Gaussian white noise, which our derivation assumes.

\paragraph{Error tolerance $\delta$.} In Fig. \ref{fig:ablation-delta}, we observe that increasing $\delta$ improves efficiency, but results in ghosting and halo artifacts near detailed edges.

\paragraph{Resolution stages $S$.} In Fig. \ref{fig:ablation-s}, we find that increasing $S$ leads to marginal speedup improvements and little image quality degradation. Since the primary speedup derives from bypassing native high-resolution steps, most of the speedup comes from $S=2$ while higher $S$ provides diminishing efficiency gains by further optimizing the already-inexpensive low-resolution trajectory.

\paragraph{Spectral windowing at resolution transitions.} Our spectral noise expansion applies a hard frequency cutoff at each resolution transition, which could in principle introduce ringing artifacts in the spatial domain. We test whether smoothing this transition improves quality on FLUX.1-dev at $1024^2$ with 10 denoising steps, $S=2$ ($64 \to 128$ latent resolution), and DCT, sweeping $\delta \in \{0.01, 0.05, 0.1\}$. Windowing does not change the transition step, so the numbers of function evaluations and tokens are identical to the unwindowed baseline at each $\delta$. We evaluate on 5,000 MS-COCO validation prompts, fixing the initial noise and injected high-band noise across variants, and report per-prompt paired differences against the hard cutoff.

We compare two variants. (i) \emph{Cosine tapering}: we multiply the retained coarse-band coefficients by a radial window
\begin{equation}
W_\rho(\omega) =
\begin{cases}
1, & |\omega| \le (1-\rho)\,\omega_c, \\
\frac{1}{2}\left[1 + \cos\left(\pi\,\frac{|\omega| - (1-\rho)\,\omega_c}{\rho\,\omega_c}\right)\right], & (1-\rho)\,\omega_c < |\omega| \le \omega_c,
\end{cases}
\end{equation}
where $\omega_c$ is the band limit of the coarse band, and add compensation noise of variance $(1 - W_\rho^2)\,t_i^2$, so that the noise floor is preserved and only the signal band edge is smoothed. Here $\rho \in \{0.25, 0.5, 1.0\}$ is the fraction of the band occupied by the roll-off. (ii) \emph{Spectral continuation}: rather than tapering the coarse band, we fill the newly expanded band with colored noise of variance $(1-t_i)^2 P_\omega + t_i^2$ instead of flat noise of variance $t_i^2$, raising the expanded band level rather than discarding coarse-band signal.

As shown in Table~\ref{tab:windowing}, cosine tapering monotonically lowers CLIP-IQA with increasing $\rho$ at all thresholds and lowers ImageReward for most $\rho$ values. Spectral continuation only slightly improves over the baseline for selected settings. These results indicate that spectral windowing does not necessarily improve image quality.

\begin{table*}[t!]
    \centering
    \caption{\textbf{Spectral windowing at resolution transitions on FLUX.1-dev.} We report per-prompt paired differences in ImageReward (IR) and CLIP-IQA against the hard-cutoff baseline over 5,000 MS-COCO validation prompts, using 10 denoising steps, $S=2$, and DCT. The number of function evaluations and token counts are identical across variants at each $\delta$.}
    \label{tab:windowing}
    \footnotesize
    \setlength{\tabcolsep}{4pt}
    \begin{tabular}{lcccccc}
        \toprule
        & \multicolumn{2}{c}{$\delta=0.01$} & \multicolumn{2}{c}{$\delta=0.05$} & \multicolumn{2}{c}{$\delta=0.1$} \\
        \cmidrule(lr){2-3} \cmidrule(lr){4-5} \cmidrule(lr){6-7}
        Variant & $\Delta$IR $\uparrow$ & $\Delta$CLIP-IQA $\uparrow$ & $\Delta$IR $\uparrow$ & $\Delta$CLIP-IQA $\uparrow$ & $\Delta$IR $\uparrow$ & $\Delta$CLIP-IQA $\uparrow$ \\
        \midrule
        Baseline (hard cutoff) & 0 & 0 & 0 & 0 & 0 & 0 \\
        Cosine tapering, $\rho=0.25$ & $-0.0044$ & $-0.0010$ & $+0.0036$ & $-0.0038$ & $-0.0148$ & $-0.0106$ \\
        Cosine tapering, $\rho=0.5$ & $-0.0031$ & $-0.0015$ & $-0.0023$ & $-0.0070$ & $-0.0315$ & $-0.0215$ \\
        Cosine tapering, $\rho=1.0$ & $-0.0055$ & $-0.0041$ & $-0.0324$ & $-0.0208$ & $-0.1158$ & $-0.0580$ \\
        Spectral continuation & $-0.0006$ & $+0.0001$ & $+0.0001$ & $+0.0005$ & $+0.0052$ & $+0.0021$ \\
        \bottomrule
    \end{tabular}
\end{table*}

\newpage
\begin{figure}
    \centering
    \includegraphics[width=\linewidth]{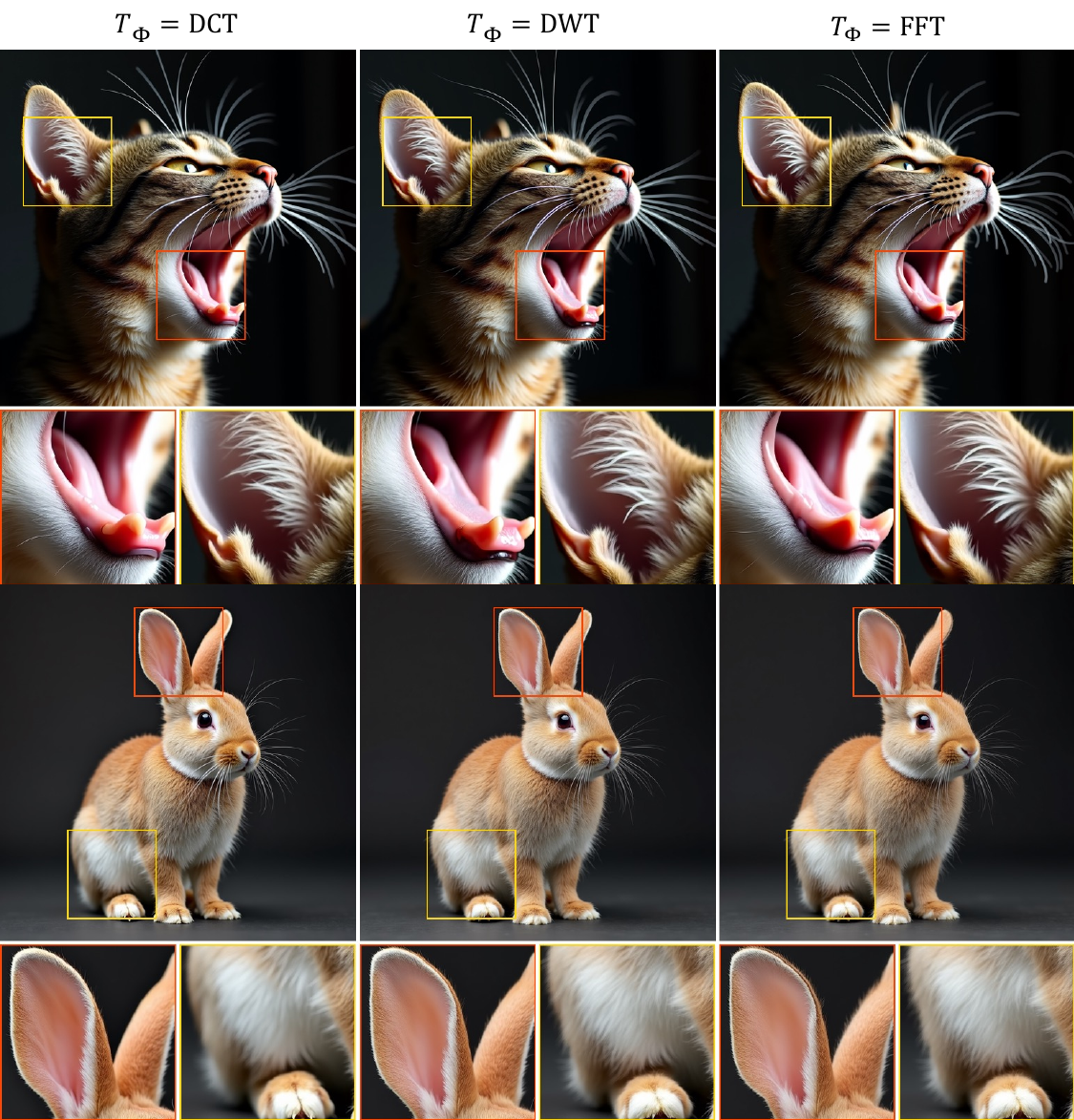}
    \caption{\textbf{Qualitative ablation on $T_\Phi$.} We see that all spectral transformations attain similar image quality. }
    \label{fig:ablation-T}
\end{figure}

\newpage
\begin{figure}
    \centering
    \includegraphics[width=\linewidth]{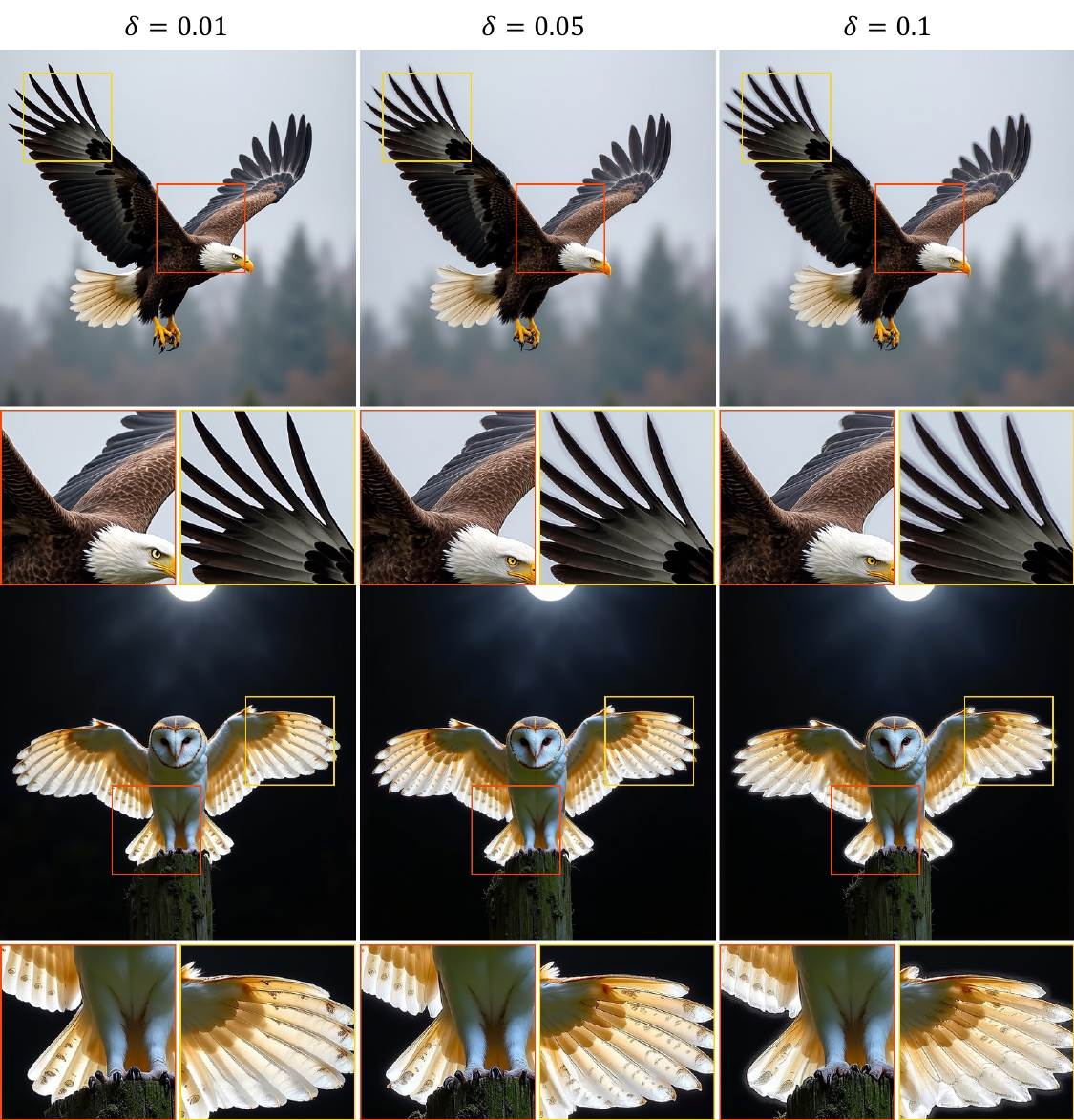}
    \caption{\textbf{Qualitative ablation on $\delta$.} We observe that increasing $\delta$ improves efficiency, but results in ghosting and halo artifacts near detailed edges. }
    \label{fig:ablation-delta}
\end{figure}

\newpage
\begin{figure}
    \centering
    \includegraphics[width=\linewidth]{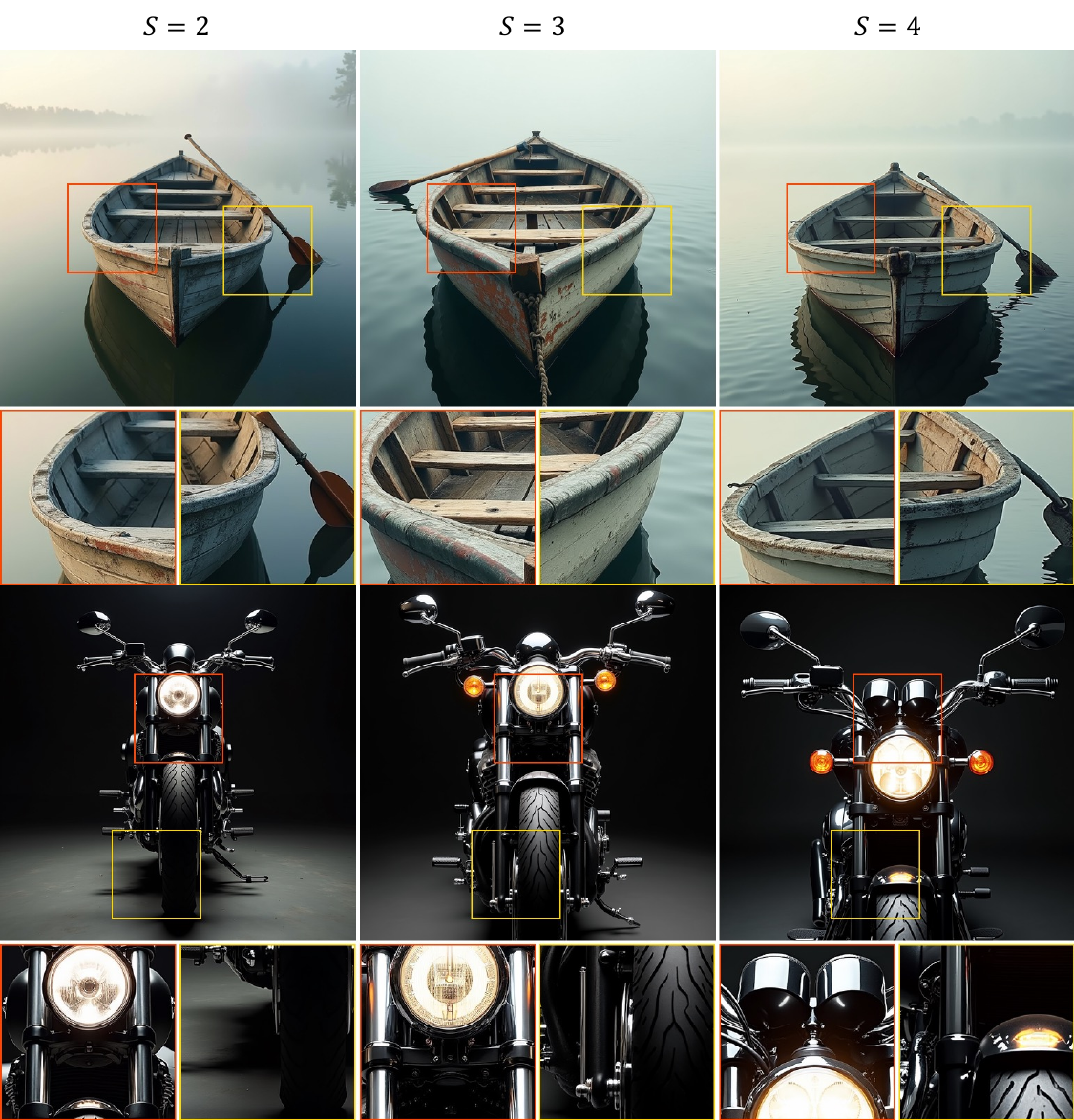}
    \caption{\textbf{Qualitative ablation on $S$.} We find that increasing $S$ leads to marginal speedup improvements and little image quality degradation.}
    \label{fig:ablation-s}
\end{figure}

\clearpage
\section{Additional frequency-based image editing details}
\label{sec:image_editing}

\paragraph{Detailed procedure of frequency-based image editing.} In our frequency-based editing pipeline, we perform the following operations given an input image $\mathbf{x}_\text{in}$, resolution scales $s_{1:S}$, and transition times $t_{1:S-1}$. We assume that the image editing process starts from transition time $t_k$, with $k \in \{1, 2, \ldots, S-1\}$.

\begin{enumerate}
    \item Compute the spectrum of the input image $\boldsymbol{\xi}_\text{in} = T_\Phi(\mathbf{x}_\text{in})$.
    \item Extract the low-frequency part of $\boldsymbol{\xi}_\text{in}$ corresponding to the representable spectrum $\Omega_{s_k}$ at scale $s_k$, and add noise at level $t_k$ to obtain $(1-t_k)\,\boldsymbol{\xi}_\text{in} + t_k\epsilon$ on $\Omega_{s_k}$, where $\epsilon \sim \mathcal{N}(0, 1)$.
    \item Fill the spectrum $\Omega_{s_{k+1}} \setminus \Omega_{s_k}$ with $t_k\epsilon$, where $\epsilon \sim \mathcal{N}(0, 1)$.
    \item Convert the expanded spectrum $\Omega_{s_{k+1}}$ back to the spatial domain $\mathbf{x}_{t_k}^{s_{k+1}}$ via $T^{-1}_\Phi$, apply the timestep alignment of Eqs.~\eqref{eq:scale-align}--\eqref{eq:time-align} to obtain $\tilde{\mathbf{x}}^{s_{k+1}}$ at the aligned timestep $\tilde{t}_k$, and continue denoising from $\tilde{t}_k$ with the editing prompt.
\end{enumerate}

The total time of denoising is thus $\tilde{t}_k$.

We conduct image editing experiments using Z-Image \citep{zimage} as the backbone, evaluating our frequency-domain image editing pipeline against the SDEdit baseline \citep{meng2022sdedit}. We set $S=2$ and  $\delta=0.01$.

\paragraph{Extended qualitative results.} We show extended image editing results on two tasks: \textit{texture editing} and \textit{artistic stylization}. We configure our method with $\delta=0.01$ and $S=2$, which corresponds to a resolution transition at step 26 of the default 50-step denoising trajectory. We set $T_\Phi=\text{DCT}$, and evaluate the effect of different spectral transforms later in this section. For the SDEdit baseline, we initialize the process by corrupting the clean source image to the equivalent noise level at step 26 and subsequently proceed with the denoising schedule. This setting accurately reflects the editing method's capability to inject textural or style information, which are considered high-frequency details, in later denoising steps. 

As shown in Figures \ref{fig:texture_edit_1} to \ref{fig:art_stylization}, our frequency-based editing method achieves superior prompt alignment and geometric consistency compared to SDEdit, successfully changing the texture and artistic style of the input image while retaining the topological structure of the input. While SDEdit-style editing at earlier timesteps (higher noise levels at step 17) can facilitate stylistic changes, it introduces significant structural drift and geometric divergence from the source image. Furthermore, starting from early timesteps incurs substantial computational overhead and requires a brittle renoising to a specific, difficult-to-tune initialization timestep to balance editability with structural fidelity.

\paragraph{Effect of $T_\Phi$ on image editing.} As shown in Fig. \ref{fig:T_comparison_editing}, FFT-, DCT- and DWT-based editing achieve similar editing quality.

\begin{figure}
    \centering
    \includegraphics[width=\linewidth]{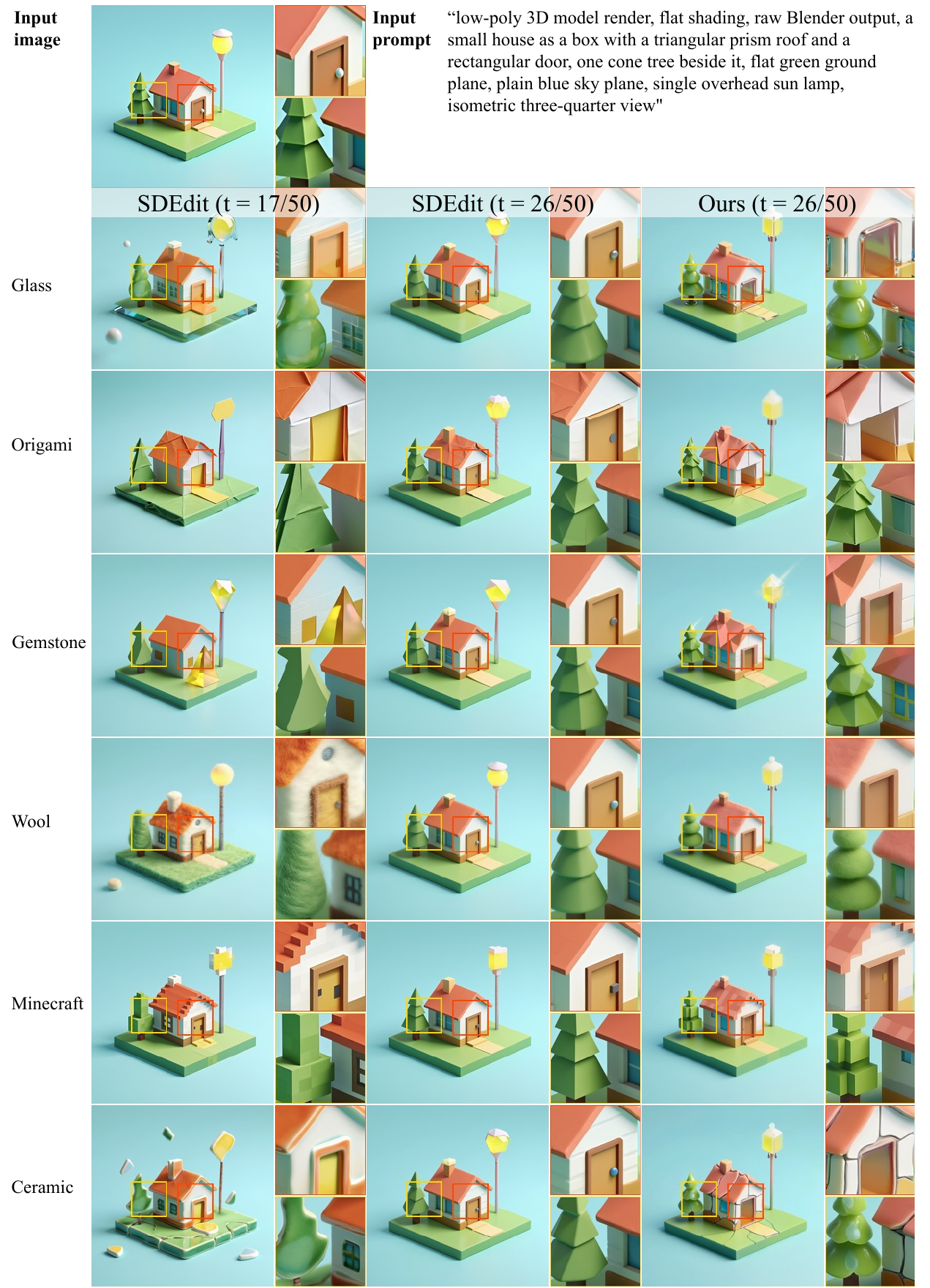}
    \caption{\textbf{Texture editing results.} Our frequency-based editing framework outperforms SDEdit, enabling high-fidelity texture transfer while preserving the geometric structure of the input image.}
    \label{fig:texture_edit_1}
\end{figure}

\begin{figure}
    \centering
    \includegraphics[width=\linewidth]{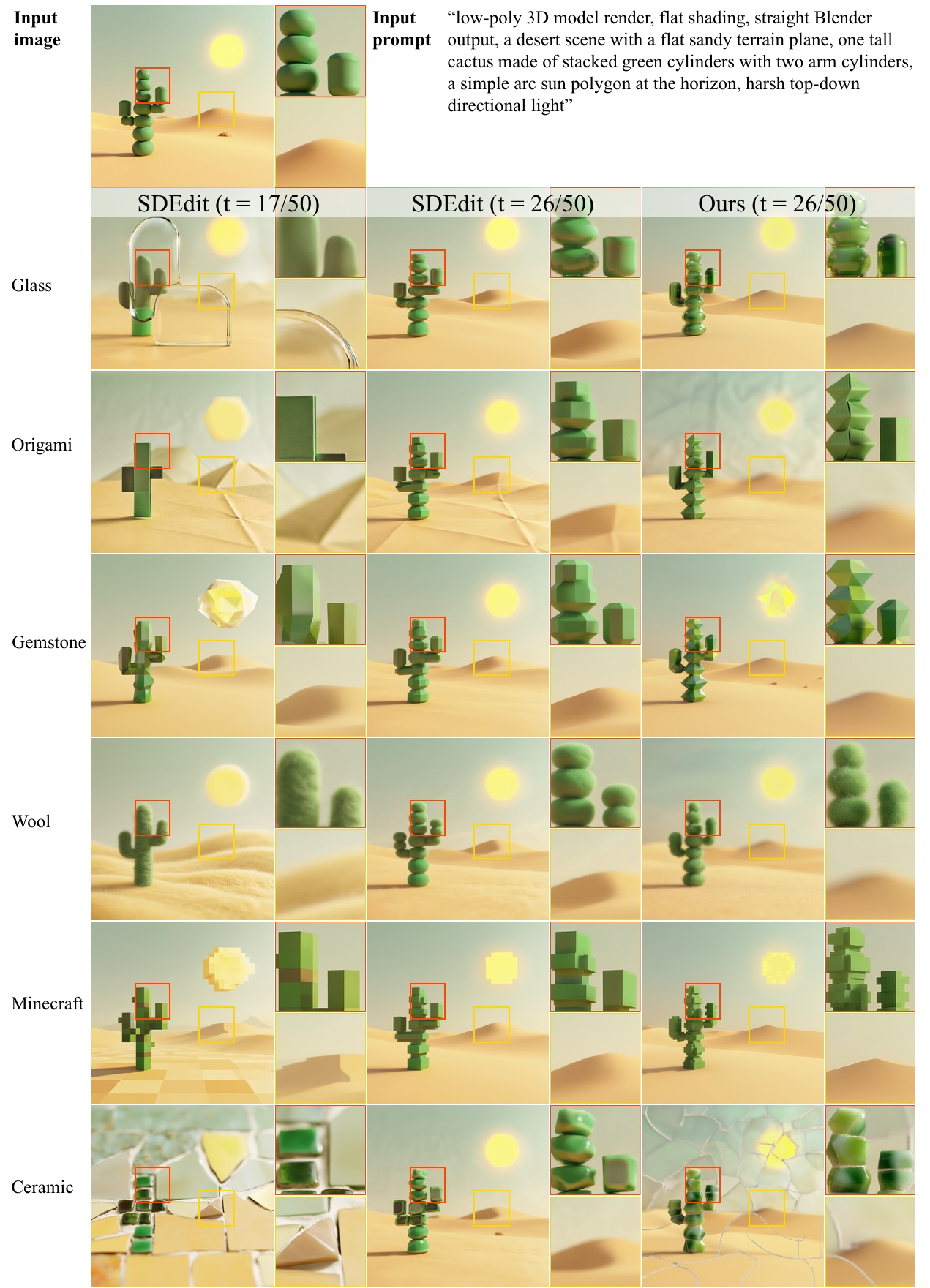}
    \caption{\textbf{Texture editing results.} Our frequency-based editing framework outperforms SDEdit, enabling high-fidelity texture transfer while preserving the geometric structure of the input image.}
    \label{fig:texture_edit_2}
\end{figure}

\begin{figure}
    \centering
    \includegraphics[width=\linewidth]{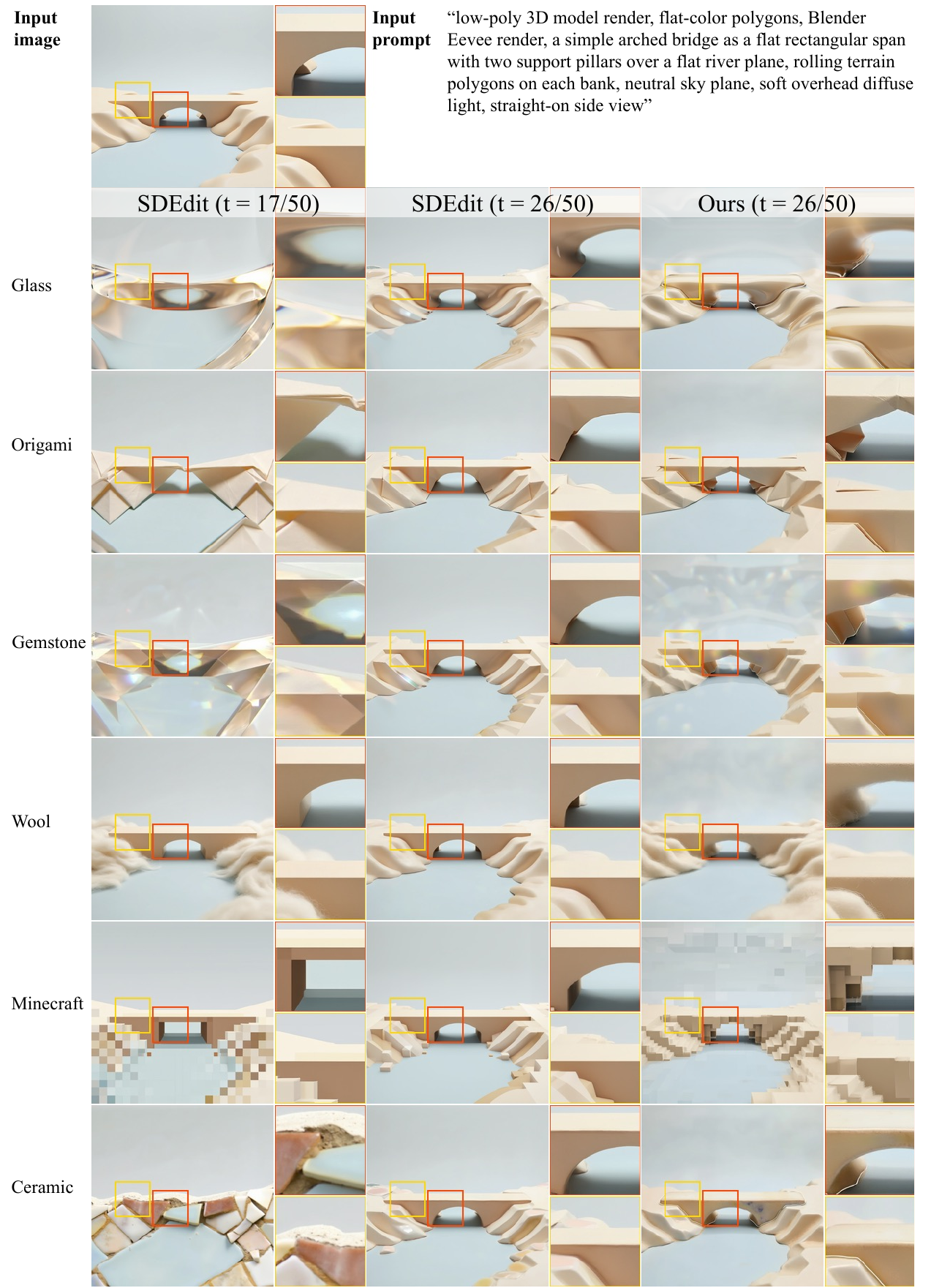}
    \caption{\textbf{Texture editing results.} Our frequency-based editing framework outperforms SDEdit, enabling high-fidelity texture transfer while preserving the geometric structure of the input image.}
    \label{fig:texture_edit_3}
\end{figure}

\begin{figure}
    \centering
    \includegraphics[width=\linewidth]{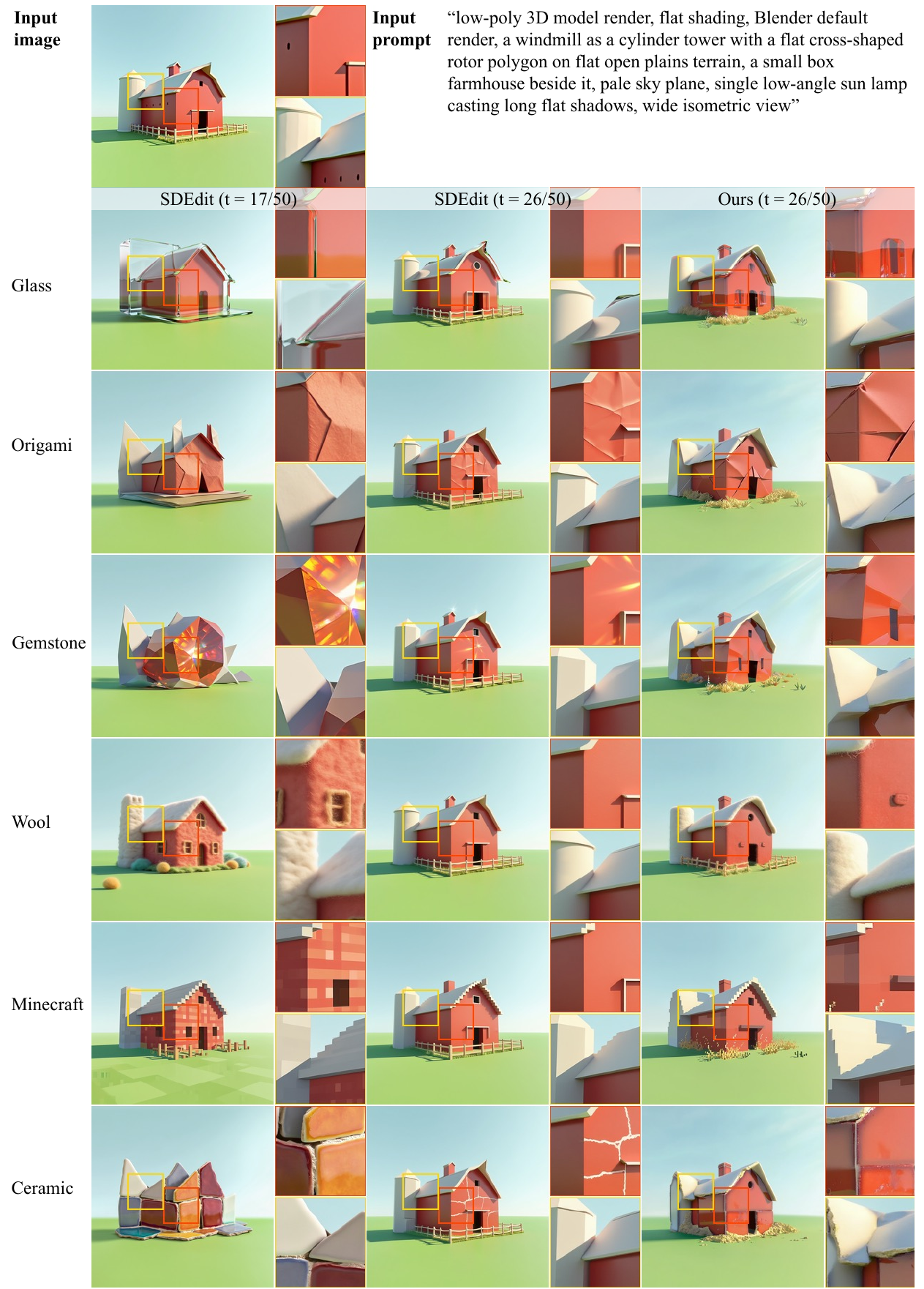}
    \caption{\textbf{Texture editing results.} Our frequency-based editing framework outperforms SDEdit, enabling high-fidelity texture transfer while preserving the geometric structure of the input image.}
    \label{fig:texture_edit_4}
\end{figure}

\begin{figure}
    \centering
    \includegraphics[width=\linewidth]{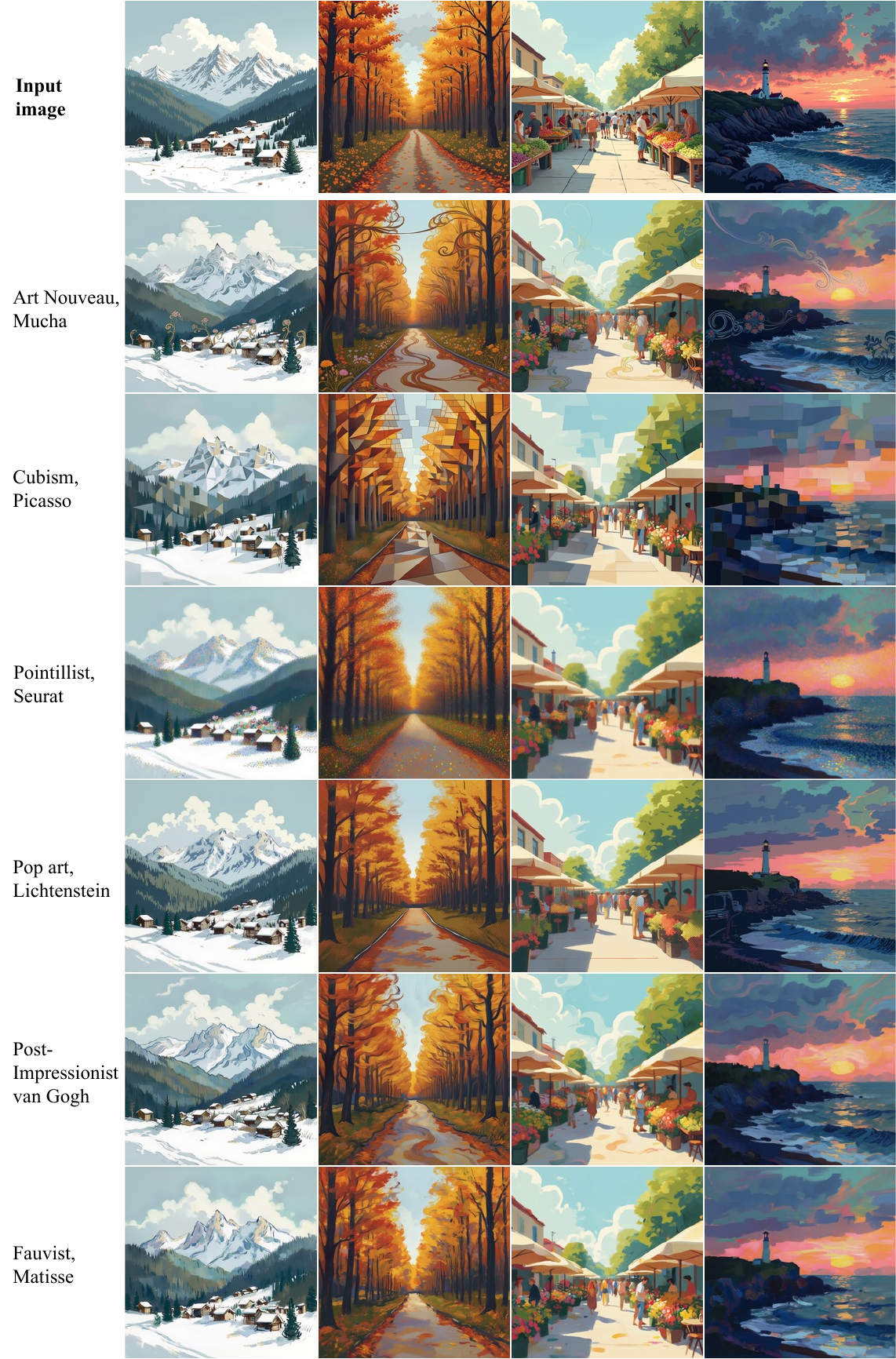}
    \caption{\textbf{Artistic stylization results.} Aside from texture editing, our frequency-based editing approach also supports artistic stylization given stylistic descriptions and a representative artist. }
    \label{fig:art_stylization}
\end{figure}

\begin{figure}
    \centering
    \includegraphics[width=\linewidth]{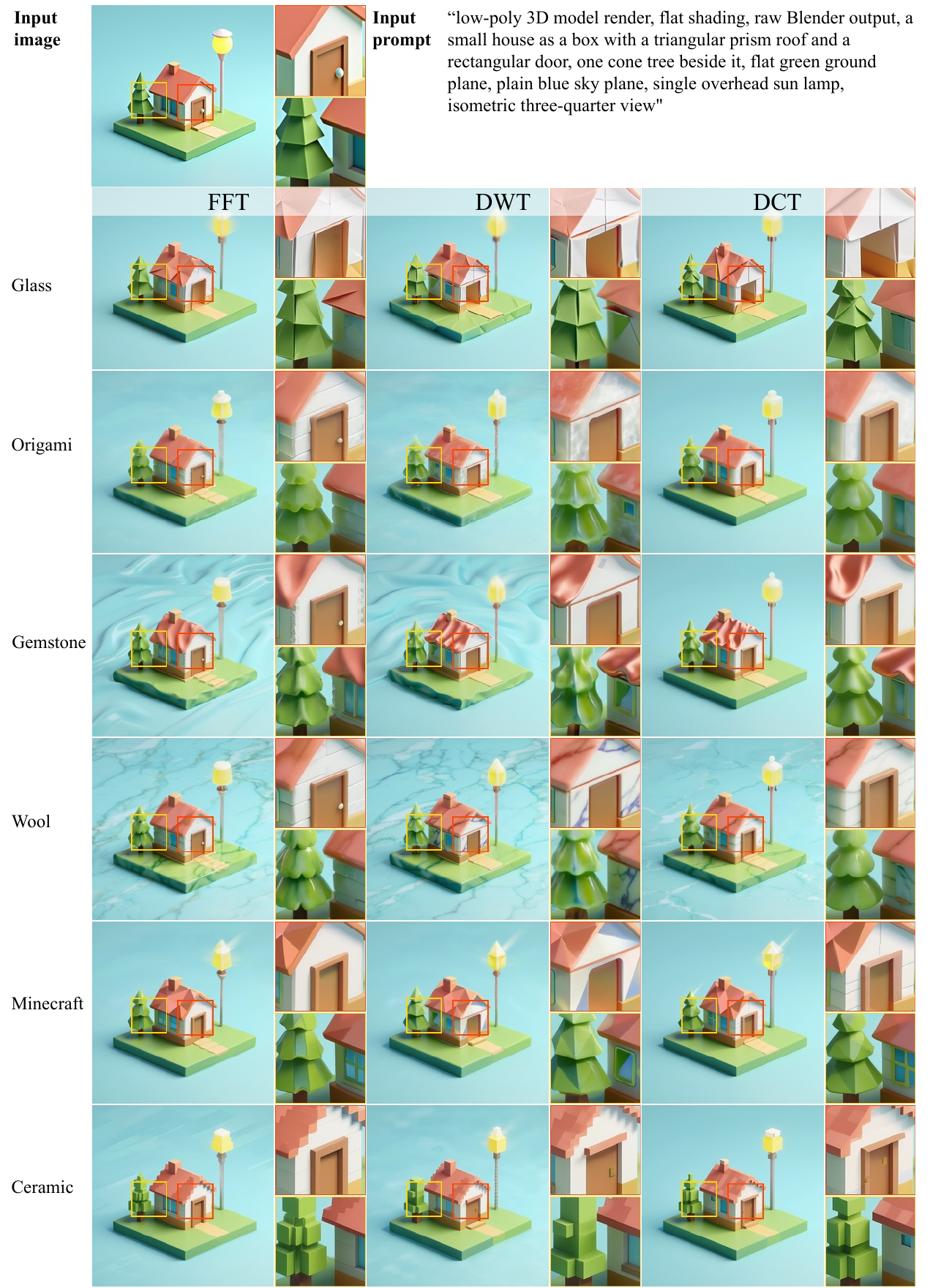}
    \caption{\textbf{Effect of $T_\Phi$ on image editing.} FFT-, DCT- and DWT-based editing achieve similar editing quality.}
    \label{fig:T_comparison_editing}
\end{figure}

\clearpage
\section{Broader impacts}
\label{app:broader-impacts}

Spectral Progressive Diffusion is a principled method for improving the efficiency of pretrained image and video generation models, rather than a system targeted at a particular deployment domain. Its main positive impact is to reduce the compute, latency, and energy cost of generation, which can make experimentation and creative workflows more accessible and lower the environmental footprint of repeated generation. At the same time, because the method can accelerate existing generative models, we acknowledge that it could also lower the cost of producing synthetic media for undesirable uses such as disinformation, impersonation, or spam, depending on the pretrained model and release setting. We do not release new datasets or model checkpoints. Our public code and demo run existing pretrained models and remain subject to their access terms, safeguards, and usage restrictions.

\section{Existing assets and licenses}
\label{app:assets-licenses}

Our experiments build on publicly available pretrained models, datasets, evaluation benchmarks, and software packages, including FLUX.1-dev~\citep{black2024flux}, FLUX.1-schnell~\citep{flux1schnell}, Z-Image~\citep{zimage}, PixelGen~\citep{pixelgen}, WAN~2.1~\citep{wan}, MS-COCO~\citep{lin2014microsoft}, Aesthetic-Train-V2~\citep{zhang2025diffusion4k}, VChitect-T2V-Dataverse~\citep{fan2025vchitect}, GenEval~\citep{geneval}, T2I-CompBench~\citep{t2icompbench}, VBench~\citep{vbench}, DiffSynth Studio (\url{https://github.com/modelscope/DiffSynth-Studio}), and the evaluation metrics used in Sec.~\ref{sec:exp}, including ImageReward~\citep{imagereward}, CLIP-IQA~\citep{wang2023clipiqa}, and NIQE~\citep{mittal2013niqe} via PyIQA (\url{https://github.com/chaofengc/IQA-PyTorch}). We cite the original sources for these assets and use them under their respective licenses and access terms.
The named licenses and access terms are: FLUX.1-dev is released under the FLUX.1 [dev] Non-Commercial License; FLUX.1-schnell, Z-Image, PixelGen, WAN~2.1, VBench, ImageReward, VChitect-T2V-Dataverse, and DiffSynth Studio are under the Apache License 2.0; the PyIQA implementation used for NIQE is under the PolyForm Noncommercial License 1.0.0 and the NTU S-Lab License; GenEval and T2I-CompBench are under the MIT License; and CLIP-IQA is under the NTU S-Lab License 1.0.
MS-COCO annotations are under Creative Commons Attribution 4.0 International (CC BY 4.0), while its images follow their original Flickr licenses; the public page for Aesthetic-Train-V2 does not list an explicit license, so we use it only under its stated access terms and do not redistribute the dataset.

\newpage

\end{document}